\documentclass{article}

\usepackage[preprint]{neurips_2022}

\usepackage[utf8]{inputenc} 
\usepackage[T1]{fontenc}    
\usepackage{graphicx}%
\usepackage{multirow}%
\usepackage{amsmath,amssymb,amsfonts}%
\usepackage{amsthm}%
\usepackage{mathrsfs}%
\usepackage[title]{appendix}%
\usepackage{xcolor}%
\usepackage{textcomp}%
\usepackage{manyfoot}%
\usepackage{booktabs}%
\usepackage{algorithm}%
\usepackage{algorithmicx}%
\usepackage{algpseudocode}%
\usepackage{listings}%

\usepackage{amssymb}

\usepackage{microtype}
\usepackage{graphicx}
\usepackage{subfigure}
\usepackage{booktabs} 
\usepackage{hyperref}

\usepackage{xcolor}         

\newcommand{\Real}[0]{\mathbb{R}}
\usepackage{amsmath}
\usepackage{amssymb}
\usepackage{graphicx}

\usepackage{algorithm}
\usepackage{algpseudocode}

\algnewcommand\algorithmicinput{\textbf{Input:}}
\algnewcommand\algorithmicoutput{\textbf{Output:}}
\algnewcommand\Input{\item[\algorithmicinput]}%
\algnewcommand\Output{\item[\algorithmicoutput]}%

\algdef{SE}[SUBALG]{Indent}{EndIndent}{}{\algorithmicend\ }%
\algtext*{Indent}
\algtext*{EndIndent}

\newtheorem{prop}{\protect\propositionname}
\usepackage{graphicx}

\usepackage{color, colortbl}
\definecolor{Gray}{gray}{0.8}

\providecommand{\assumptionname}{Assumption}
\providecommand{\lemmaname}{Lemma}
\providecommand{\propositionname}{Proposition}
\providecommand{\theoremname}{Theorem}

\newcommand{\Exp}{\mathbb{E}}

\newcommand{\bg}{\mathbf{g}}

\newtheorem{corollary}{Corollary}[section] 

\newcommand{\myscaleplottitle}{0.207}
\newcommand{\myscaleplotsubplot}{0.20}

\title{Personalized Federated Learning with Exact Stochastic Gradient Descent}

\author{
    Sotirios Nikoloutsopoulos$^{1}$\thanks{Corresponding author}
        \quad
    Iordanis Koutsopoulos$^{1}$
      \quad
      Michalis~K.~Titsias$^{2}$ \\
          $^{1}$Athens University of Economics and Business
          \\ $^{2}$Google DeepMind\\
        \texttt{\{snikolou,jordan\}@aueb.gr}\\
        \texttt{mtitsias@google.com}  
}

\begin{document}
\maketitle

\begin{abstract}
We propose 
a Stochastic Gradient Descent (SGD)-type algorithm for Personalized Federated Learning
which can be particularly attractive for mobile energy-limited regimes due to its low per-client computational cost.
The model to be trained includes a set of common weights for all clients, and a set of personalized weights that are specific to each client. At each optimization round, randomly selected clients perform multiple full gradient-descent updates over their client-specific weights towards optimizing the loss function on their own datasets,
without updating the common weights.
This procedure is energy-efficient 
since it has low computational cost per client.
At the final update of each round, each client computes the joint gradient over both the client-specific and the common weights and returns the gradient of common weights to the server, which allows to perform  
 an exact SGD step over the full set of weights in a distributed manner. 
 For the overall optimization scheme, we rigorously prove convergence, even in non-convex settings such as those encountered when training neural networks, 
 with a rate of $\mathcal{O} \left (\frac{1}{\sqrt{T}} \right )$  with respect to communication rounds $T$. In practice, PFLEGO exhibits substantially lower per-round wall-clock time, used as a proxy for energy. Our 
 theoretical guarantees translate to superior performance in practice against baselines such as  FedAvg and FedPer, as evaluated in several multi-class classification datasets, in particular, Omniglot, CIFAR-10, MNIST, Fashion-MNIST, and EMNIST.
\end{abstract}




\maketitle

\section{Introduction}\label{sec1}

{

\label{subsec:distributed_learning}

Federated Learning (FL) is a form of distributed learning \cite{dean2012,xing2016,marozzo2022} and has emerged due to increased focus on privacy. The first proposed FL algorithm was FedAvg \cite{mcmahan2017communicationefficient} followed by many others in the recent literature \cite{mcmahan2017communicationefficient,li2020federated,arivazhagan2019federated,singhal2021federated,collins2021exploiting,fedbabu,ditto}. The main concept in FL is to allow multiple clients to collaboratively train a shared model without sharing their raw data. Thus, unlike conventional distributed algorithms, FL algorithms are decentralized\footnote{Decentralized settings refer to systems where data and computations are distributed across multiple independent clients rather than centralized on a single server, enhancing privacy and efficiency by keeping data local.} distributed algorithms that preserve privacy constraints \cite{alekh2018eugeneraldataprotection}. 
 
 
When the tasks are drawn from different data distributions, relying on a single shared model in FL often leads to poor performance \cite{li2020federated}, as the structure of a single model cannot generalize to all tasks. To address this issue, we extend Federated Learning to Personalized FL, in which each client additionally trains a set of parameters specifically tailored to their own task. 

In non-IID\footnote{Non Independent, Identically Distributed} settings, clients’ gradient updates derived from their locally trained models can conflict with one another, leading to a tug-of-war dynamic \cite{hadsell2020embracing}, which can be viewed as a war of gradients. This occurs when the updates from one client, based on its specific data distribution, push the global model in one direction, while updates from another client, with a different or even opposing data distribution, pull the model in the opposite direction. As a result, the progress made by one client can be undone by another, reducing the overall performance and stability of the shared model. 
Personalized FL addresses this issue to ensure that models are flexible enough to adapt to individual client needs without being compromised by conflicting gradients from others.

Personalized FL has found applications in various domains. In healthcare \cite{Wuetal20,chen2021fedhealth} it enables collaborative training of models across hospitals without compromising patient data privacy. In finance, it allows institutions to build fraud detection systems without exposing sensitive transaction data \cite{aurna}. Mobile device personalization, such as keyboard suggestions \cite{hard2019federated} and voice recognition \cite{guliani}, also benefits from FL by training models directly on user devices. The advent of 5G technology further enhances these applications by providing high-speed and low-latency connections, which facilitate efficient communication and rapid model aggregation. This is particularly beneficial in scenarios requiring real-time updates and synchronization across multiple clients.


}



In this work, we propose a novel approach for personalized FL, which we call Personalized Federated Learning with Exact Gradient-based Optimization (PFLEGO) \footnote{Our source code is available at \href{https://github.com/sotirisnik/PFLEGO}{https://github.com/sotirisnik/PFLEGO}.}. {While previous works~\cite{arivazhagan2019federated, singhal2021federated} have also considered FL frameworks with both global (shared) and personalized (client-specific) weights, our approach differs significantly in the way the federated system is trained.}
Its main advantage is that the updates are engineered so that the overall algorithm achieves exact stochastic gradient descent (SGD) \cite{robbinsmonro51, bottou2010large} minimization of the training loss function, which is the \textit{exact equivalent} of training with all data concentrated in one place. At each optimization round, our algorithm performs the following steps: \textit{(a)} The server sends to a randomly selected subset of clients the updated weights corresponding to the common layers; \textit{(b)} Given the common weights, each client performs a number of local gradient descent updates of its client-specific weights on its own local dataset, towards optimizing its loss function, without updating the common weights; \textit{(c)} Contrary to the other methods that we compare against, in our approach each client at its \textit{final} optimization step for that round computes the joint gradient over both the client-specific and the common weights and sends to the server the \textit{gradient} of the common weights, rather than the raw weights themselves; \textit{(d)} Then, the  
server aggregates these gradients across clients and uses them to update the common weights. The process continues until convergence. The sequence of the last two steps (\textit{c}),(\textit{d}) is equivalent to an exact SGD update over the full set of parameters that is 
implemented by the clients and the server in a distributed manner; see Section 
\ref{sec:convergence} for details. 
The convergence of the algorithm
is ensured, as shown by our rigorous proof, which also holds in non-convex settings such as for training
 neural networks, where a $\mathcal{O} \left (\frac{1}{\sqrt{T}} \right )$ convergence rate is established with respect to the number of communication rounds $T$.


We experimentally validate the superiority of our method over state-of-the-art algorithms, in particular, FedAvg \citep{mcmahan2017communicationefficient}, FedProx\cite{li2020federated}, FedPer \citep{arivazhagan2019federated}, FedRecon \citep{singhal2021federated}, FedRep\cite{collins2021exploiting}, FedBabu\cite{fedbabu}, and Ditto\cite{ditto}, using several benchmark datasets in multi-class classification, specifically, MNIST, CIFAR-10, Fashion-MNIST, EMNIST, and Omniglot. The results show that our algorithm leads to lower training loss, and thus much more effective learning, especially in cases where personalization is needed most, i.e., when the data distributions across clients differ most.

Further we show that the computational complexity of our algorithm at each local client optimization step is $\mathcal{O}(1)$, while for other baselines such as the popular FedAvg the cost is $\mathcal{O}(\tau)$ where $\tau$ is the number of gradient descent updates at the client-side. A practical implication of this is that the energy consumed by each client will be much lower for our method compared to other methods such as FedAvg \citep{mcmahan2017communicationefficient}, FedProx\cite{li2020federated} and FedPer \citep{arivazhagan2019federated}. This makes our method attractive for mobile energy-limited regimes. 

{Our main contributions are as follows:
\begin{itemize}
    \item We propose \textbf{PFLEGO}, a novel personalized FL algorithm that combines exact stochastic gradient descent updates with a two-tiered (global and personalized) parameter structure.
    \item We provide a rigorous theoretical analysis, establishing an $\mathcal{O}(1/\sqrt{T})$ convergence rate for non-convex objectives under arbitrary client sampling.
    \item We demonstrate, through extensive experiments on standard FL benchmarks (MNIST, CIFAR-10, Fashion-MNIST, EMNIST, Omniglot), that PFLEGO consistently outperforms state-of-the-art baselines in terms of accuracy and communication efficiency, especially under heterogeneous data distributions.
    \item We theoretically analyze the computational cost of PFLEGO and empirically demonstrate, by measuring execution time per round, that our method achieves significant time savings compared to existing approaches.
    \item PFLEGO is inherently scalable: its per-round computation and communication costs depend only on the number of participating clients, not the total client population, and its convergence guarantees hold as long as the averaging coefficients used in the aggregation step for each round sum to one.
\end{itemize}}

%

{The rest of the paper is organized as follows. Section \ref{related_work} presents the related work and compares our approach against prior methods. Section \ref{sec:proposed_framework} details our proposed framework and provides formal convergence analysis. Section \ref{experiments} reports experimental results, and Section \ref{conclusion} concludes the paper.}

\section{Related Work}
\label{related_work}
{\paragraph{FedAvg and its variants}
Chronologically, the first proposed FL algorithm was FedAvg \cite{mcmahan2017communicationefficient}. In FedAvg, at each round, the server selects a subset of clients and sends to them the current model parameters or weights. The clients update the model parameters locally by performing a number $E$ of local gradient update steps towards optimizing their loss function, and then, they return the locally optimized parameters to the server. The server then averages model parameters, it sends them back to the clients,  and the process is repeated. For $E=1$, FedAvg is called Federated Stochastic Gradient Descent 
\cite{jiang2019improving}.

FedAvg performs notoriously poorly when there is data heterogeneity, i.e., when the data are non-iid \cite{li2020convergence,zhu2021federated}, which means that the clients' datasets are drawn from different distributions. Several extensions of FedAvg have been proposed, such as  \cite{ li2020convergence,li2020federated,li2020feddane,wang2020tackling, reddi2021adaptive,liu2019accelerating,yu2019linear}, which aim to improve the behaviour of the algorithm in heterogeneous data settings. While such methods can improve convergence \cite{wang2020tackling,li2020convergence},
they 
do not account for client personalization, since
they try to train a single global model across heterogeneous clients.  

{In a diverse setting, clients may 
possess 
distinct preferences or interests. 
In fact, the greater the number of clients is, the less likely it is that a substantial overlap of interests exists \cite{leiyang}. This leads to the concept of personalization, and personalized Federated Learning emerges as a response to this need, where the aim is to train multiple models, one for each client, so as to tailor model parameters according to client preferences, while enjoying the benefits of Federated Learning.} Some examples of personalized Federated Learning applications introduced in the literature are next word prediction \cite{hard2019federated}, emoji prediction \cite{ramaswamy2019federated,lee2021opportunistic}, health monitoring \cite{Wuetal20}, and personalized healthcare via wearable devices \cite{chen2021fedhealth}.

One important class of personalized FL methods \cite{agarwal2018cpsgd,collins2021exploiting,singhal2021federated} dictates clients to share some common weights, since clients cannot produce a good model on their own due to insufficient amounts of data, but also utilize some client-specific (personalized) weights for each client as well that are trained exclusively on that client's dataset \cite{konecny2016federated,arivazhagan2019federated}.
}
\paragraph{Personalized FL by fine-tuning a global model}
One approach to deal with personalization in FL is to allow clients to fine-tune a shared global model using either local adaptation \cite{wang2019federated,yu2021salvaging} or techniques inspired by meta-learning
\cite{jiang2019improving,chen2019federatedmeta,dinh2022personalized,fallah2020personalized}, without resorting to client-specific model parts. 
These methods differ significantly from ours since they require the communication of the full set of parameters of the shared global model between the server and the clients. Note also that fine-tuning can be computationally expensive since it requires the individual clients to adapt the full parameter vector of a deep neural network.

In some works that follow this fine-tuning principle \cite{li2020federated,li2020feddane,ditto}, the local objective of the clients incorporates a regularization term to keep the local model close to the 
global model. 
More precisely, FedDane \cite{li2020feddane} constructs this regularization by including an additional computational step that requires the server to communicate with two different subsets of clients. First, a subset of clients is selected to compute and send gradients based on the current global model. These gradients are aggregated by the server to form a global gradient. Then, a different subset of clients performs a local optimization that integrates both this global gradient and their local gradients while applying a quadratic regularization. In the related Ditto method \cite{ditto} each client maintains two separate models with parameter vectors of the same dimension as the global model.
In the first model, the client copies the global parameters from the server and updates them by performing a number of optimization steps. Then the client performs a number of optimization steps on the second model parameters by including a regularization term to keep their values close to the first model optimized weights.

\paragraph{Personalized FL by feature transfer} Our approach falls into the class of methods that achieve personalization in FL by using a feature transfer model, similarly to traditional non-distributed multi-task architectures \citep{caruana97,Ruder2017AnOO}.
%
Such methods aim to train a NN whose parameters consist of a common part (also called global or shared parameters) for all clients, and a client-specific part (also called personalized parameters). 
This approach was followed by FedPer \cite{arivazhagan2019federated}. 
However, the optimization algorithm in \cite{arivazhagan2019federated} differs from our method, as it is based on the standard FedAvg scheme. Specifically, the clients update both the global and the personalized parameters by executing joint gradient descent steps. Then, each client sends back to the server the locally updated global parameters and the server updates the global parameters through averaging. Their method reduces to standard FedAvg when there are no client-specific parameters. 
{In contrast, our algorithm updates only the personalized weights multiple times within each round, which is energy efficient, and 
performs only one full backpropagation 
computation
for the global parameters in order to perform the final exact SGD step. This design preserves theoretical convergence while simultaneously it reduces per-round energy costs compared to prior personalized FL methods.}

In another related work, FedRecon \cite{singhal2021federated}, similarly to FedPer \cite{arivazhagan2019federated}, a set of global and client-specific parameters are learned, but at each round the optimization algorithm does not update simultaneously the global and personalized parameters, and therefore that method can be considered as performing block coordinate optimization. 
In contrast, our distributed optimization algorithm incorporates \textit{exact} SGD steps where the full set of model parameters are simultaneously updated. This latter property means that our method enjoys theoretical convergence guarantees similar to SGD schemes \cite{robbinsmonro51,bottou2010large}.

In yet another thread of works FedBaBu \cite{fedbabu} and FedRep \cite{collins2021exploiting}, similarly to FedPer, decouple the model parameters to a set of global and a set of client-specific parameters. However with FedBabu, the client-specific parameters have an orthogonal weight initialization\footnote{Orthogonal weight initialization involves setting the weights in the head of a neural network as an orthogonal matrix.} at the server side, and the client-specific parameters are never trained. FedBabu differs from our method, since the client-specific parameters are initialized at the server-side, and they remain fixed across all rounds. The orthogonal initialization is essential to achieve desirable personalized performance.
With FedRep each client performs a number of updates to their client-specific parameters, and then the clients proceed to perform a number of updates to their global parameters. While our method uses a weighted unbiased gradient scheme on the global parameters of the clients, FedRep employs a weighted averaging scheme on the global parameters. Another difference between our method and FedRep is that the final gradient update for client-specific parameters in our method is unbiased. Our distributed unbiased scheme produces estimates that are on average equal to the true value of the parameters being estimated, without over or underestimating the true value, therefore, we expect our method to outperform FedRep.

\paragraph{FL with gradient return}
Finally, there exist non-personalized FL algorithms optimizing a set of global parameters, where gradients are returned to the server \citep{yao2020federated, Renetal21}. 
In \cite{yao2020federated}, each client performs many training steps over local copies of the global parameters and returns to the server the final gradient (after the final iteration) of these parameters. The server then aggregates them and performs a gradient step. 
However, convergence is guaranteed only in the special case where each client performs a single iteration so that the algorithm reduces to  SGD optimization. The work by \cite{Renetal21} aims to accelerate training by optimizing batch size. Specifically, clients sample a batch over their private datasets and perform a single gradient step to update their local copies of global parameters, and then they return the gradient to the server. This work differs from our method, since it does not involve personalization. In our scheme, the clients may perform multiple gradient steps over the client-specific parameters before returning the gradient of global parameters to the server.
\section{Proposed Framework} \label{sec:proposed_framework}


{Before presenting the details of the PFLEGO algorithm, we clarify the distinction between local and global optimization steps in our framework.

In our framework, each participating client performs local updates using full gradient descent (GD) on its entire local dataset, rather than stochastic mini-batches. This means that, on the client side, the computed gradients are exact with respect to the client's local data. However, since only a random subset of clients participates in each communication round, the global aggregation step at the server implements stochastic gradient descent (SGD) on the overall objective: the aggregated gradient is an unbiased estimator of the true global gradient. Thus, the proposed PFLEGO algorithm (Algorithm~\ref{alg:PFLEGO}) combines local GD with global SGD, and our theoretical analysis and notation reflect this distinction throughout the manuscript.}

\subsection{Personalized FL Setting}

We consider a supervised FL setting in which there is a single server and $I$ clients. Each client has a locally stored dataset $\mathcal{D}_i = (\mathcal{X}_i,\mathcal{Y}_i)$, where $\mathcal{X}_i=\{x_{i,j}\}_{j=1}^{N_i}$ are the input data samples (e.g., images) and $\mathcal{Y}_i=\{y_{i,j}\}_{j=1}^{N_i}$ are the corresponding target outputs (e.g., class labels). 
The objective of FL is to optimize a shared or global model, together with personalized or client-specific parameters, by utilising all client datasets. As a shared backbone model, we assume a deep neural network that consists of a number of common layers with overall parameter vector $\theta$. The number of outputs of the common layers, i.e., 
the size of the feature vector, is $M$. More precisely, the shared model receives an input $x$ and constructs in its final output a representation or feature vector $\phi(x;\theta) \in \Real^M$. Each client has a copy of the same network architecture corresponding to this shared representation. Each client $i$ also has an additional set of personalized layers attached as a head to $\phi(x;\theta)$. For simplicity, we assume that personalized layers consist of a single linear output layer with weights $W_i$; see Figure \ref{fig:clientsevercommunication} for a pictorial description.     

\paragraph{Training Objective for classification}

The learning objective is to train the neural network model by adapting the global parameters $\theta$ and the personalized parameters $\{W_i\}_{i=1}^I$. The full set of parameters is denoted by $\psi = \{\theta, W_1,\ldots,W_I\}$. Learning requires the minimization of the following training loss function that aggregates all client datasets:
\begin{equation}
    \mathcal{L}(\psi) = \sum_{i=1}^{I} \alpha_i \ell_i (W_i,\theta),
    \label{eq:fullloss}
\end{equation}
where the scalar $\alpha_i=\frac{N_i}{\sum_{j=1} N_j}$ quantifies the data volume proportionality of different clients, and $\ell_i(W_i,\theta)$ is the client-specific loss:   
\begin{equation}
\ell_i(W_i,\theta) 
= \frac{1}{N_i}\sum_{j=1}^{N_i} \ell(y_{i,j}, x_{i,j};W_i, \theta).
\label{eq:clientloss}
\end{equation}
The form of each data-individual loss $\ell(y_{i,j}, x_{i,j};W_i, \theta)$  depends on the application, e.g., on whether the task is a regression or a classification one. The case of multi-class classification, that we consider in our experiments, is detailed below while other cases can be dealt with similarly. 
\paragraph{Multi-class classification}
In multi-class classification, the input $x$ is assigned a class label $y \in \{1,\ldots,K_i\}$. In our personalized setting, each client $i$ can tackle a separate classification problem with $K_i$ classes, where classes across clients can be mutually exclusive or partially overlap. The set of personalized weights $W_i$ becomes a $K_i \times M$ matrix that allows to compute the logits in the standard cross entropy loss, 
$$
\ell(y_{i,j}, x_{i,j};W_i, \theta) = 
 - \log Pr(y_{i,j} | x_{i,j};W_i, \theta). 
$$
The class probability is modeled by the softmax function, i.e.,
$$
\text{Pr}(y_{i,j}|x_{i,j};W_i, \theta) 
= \frac{e^{z_{y_{i,j}}}}
{\sum_{k=1}^{K_i} e^{z_k}}
$$
where the $K_i$-dimensional vector of logits  is 
$z = W_i \times \phi(x_{i,j};\theta)$ and  $\phi(x_{i,j};\theta)$ is a $M \times 1$ vector. 
\begin{figure*}[t]
    \begin{center}
     \includegraphics[scale=0.5]{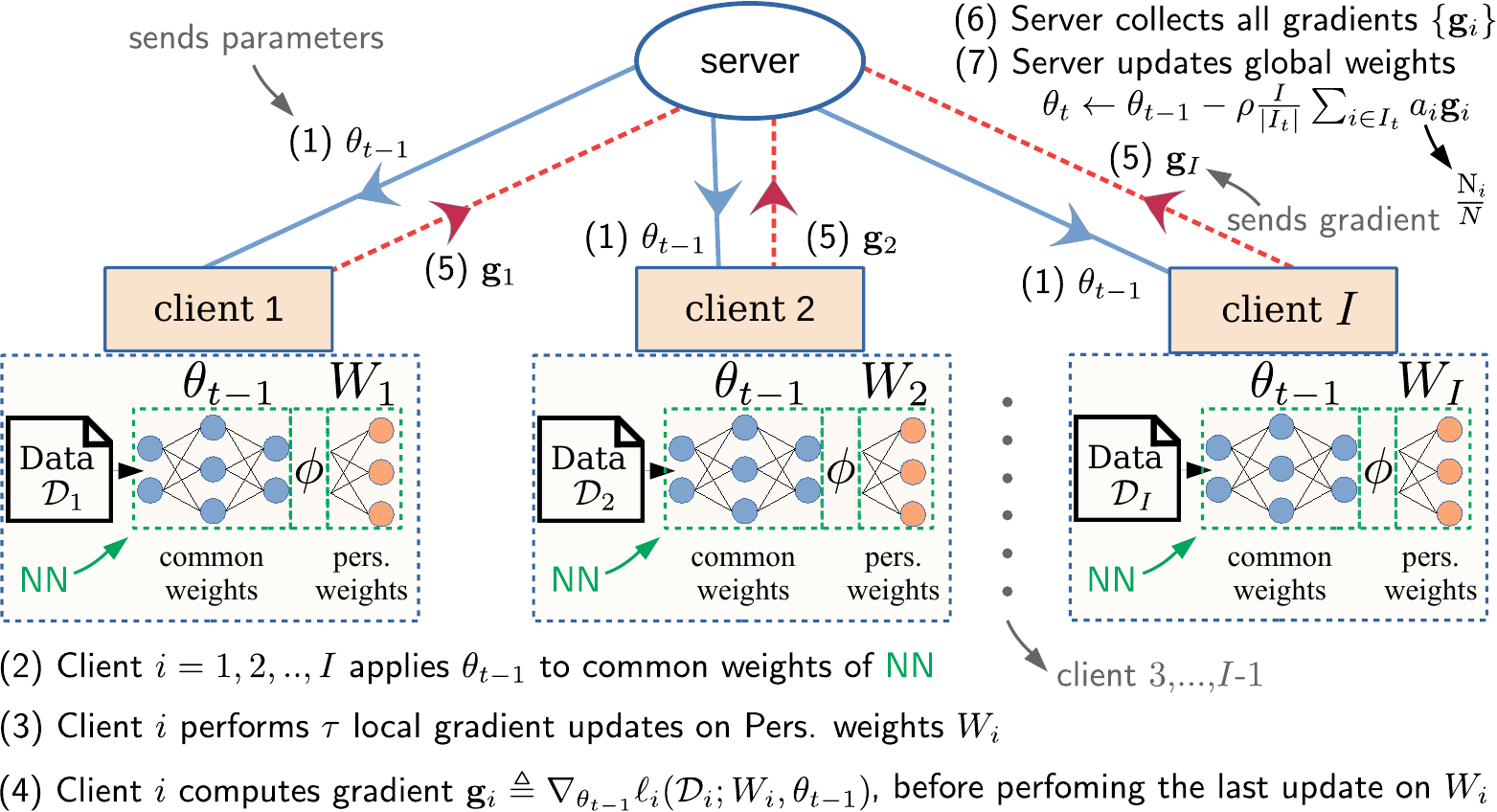}
    \caption{Training process of the PFLEGO algorithm with $I$ clients and a single server for one communication round. The numbering (1), $\dots$, (7) display the sequence of the execution steps within a single round. }
    \label{fig:clientsevercommunication}
    \end{center}
\end{figure*}

\subsection{The Proposed Algorithm}

The objective of Personalized FL is to minimize the global training loss in \eqref{eq:fullloss} over the full set of parameters $\psi$. To this end, we propose a distributed optimization algorithm that incorporates exact Stochastic Gradient Descent (SGD) steps over $\psi$. The stochasticity arises due to a random  client participation or selection process defined in the sequel, in Section \ref{sec:client-selection}. Without stochasticity, i.e., if all clients participate at every round, these previous steps become exact gradient descent (GD) steps. This means that the proposed algorithm (detailed in Section \ref{sec:serverclient}) converges similarly to standard GD or SGD methods respectively; see Sections \ref{sec:convergence}, and \ref{convergence_rate}. 
 
%
%

\subsubsection{Client Participation Process
\label{sec:client-selection}
}

We assume that the optimization of the global loss in \eqref{eq:fullloss} is performed in different rounds, where each round involves a communication of the server with some of the clients. Specifically, at the beginning of each optimization round $t$, a subset of clients $\mathcal{I}_t \subset \{1,\ldots,I\}$ is selected uniformly at random to participate. For instance, two sensible options  are: (\textit{a}) the number of clients $r_t \triangleq |\mathcal{I}_t|$ follows a Binomial distribution $\text{B}(I,\pi)$, i.e., each client participates independently with probability $\pi$, or (\textit{b}) a fixed number $0< r \leq  I$ of clients are always selected, i.e., $|\mathcal{I}_t|=r$ for any $t$. For both cases an arbitrary client $i \in \{1,\ldots, I\}$ participates in each round with probability 
$
\text{Pr}(i \in \mathcal{I}_t) = 
\frac{r}{I},
$
where for case (\textit{a}) $r = I \pi$ can be a real number, while for (\textit{b}), $r$ is strictly an integer.  

{PFLEGO is robust to stragglers\footnote{In FL, stragglers refer to client devices that are unable to complete their local training or communication tasks within the allotted time for a given round, often due to limited computational resources, poor connectivity, or interruptions.} \cite{li2020federated} and intermittent connectivity: in each round, the server aggregates gradients only from participating clients, and missing clients’ gradients can be set to zero. The theoretical convergence proof remains valid under this partial participation model, as long as the aggregation weights for the participating clients add up to one.}

\subsubsection{Client and Server Updates
\label{sec:serverclient}  
}

\paragraph{Client side} Having selected the subset of clients $\mathcal{I}_t$ to participate in round $t$, the server sends them the global model parameters $\theta_{t-1}$. Then, each client $i \in \mathcal{I}_t$ updates locally the current values $W_i \triangleq W_{i,t-1} $   of the client-specific parameters by performing a total number of $\tau$ gradient descent steps.   For the first $\tau-1$ steps, the global parameter $\theta_{t-1}$ is ``ignored'' (i.e., it remains fixed to the value sent by the server), and only the gradient $\nabla_{ W_{i}} \ell_i(W_i,\theta_{t-1})$ over the client-specific parameters $W_i$ is computed.

The gradient steps of these $\tau-1$ steps have the form 
$$
W_i \leftarrow W_i - \beta  \nabla_{ W_{i}} \ell_i(W_i,\theta_{t-1}),
$$
where $\beta$ is the learning rate. In fact, these $\tau-1$ updates could be replaced by any other optimization procedure, including one with momentum, e.g., Adam \cite{kingma2017adam}, as long as the final loss value $\ell_i(W_i,\theta)$, i.e., after these $\tau-1$ steps, is smaller or equal to the corresponding initial value.

In contrast, for the final ($\tau$-th) iteration, the client simultaneously computes the joint gradient $(\nabla_{W_i} \ell_i, \nabla_{\theta} \ell_i)$ of both $W_i$ and the shared parameters $\theta$, and it performs the final ($\tau$-th) gradient step for $W_i$ using the rule 
\begin{equation}
W_i \leftarrow W_i - \rho \frac{I}{r} \nabla_{ W_{i}} \ell_i(W_i,\theta_{t-1}),
\label{eq:clientupdate}
\end{equation}
where $\rho$ is the learning rate at round $t$, and the multiplicative scalar $\frac{1}{\text{Pr}(i \in \mathcal{I}_t) } = \frac{I}{r}$ ensures unbiasedness of the full gradient update over all parameters $\psi$; for more details about that important property of our proposed algorithm, see Section \ref{sec:convergence}.  

\paragraph{Server side} The server gathers all gradients from the participating clients and then performs a gradient update to the parameters $\theta$ by taking into account also the data proportionality weight of each client. Specifically, the update it performs takes the form:
\begin{equation}
\theta_{t} \leftarrow \theta_{t-1} - \rho \frac{I}{r} \sum_{i \in \mathcal{I}_t} \alpha_i \nabla_{\theta_{t-1}} 
 \ell_i(W_i,\theta_{t-1}),
 \label{eq:serverupdate}
\end{equation}
where again the term $\frac{I}{r}$ is included to ensure unbiasedness as detailed next. The whole optimization procedure across rounds is described by Algorithm \ref{alg:PFLEGO}.

\begin{algorithm}[t]
\small
\caption{PFLEGO}
\label{alg:PFLEGO}
\begin{algorithmic}
\Input $T$ rounds, $\tau$ local gradient updates, $I$ clients, $0 \!< r \!\leq I$ sampled clients per round, $N_i$ data samples at the $i$-th client
\State \textbf{Server:}
\Indent
\State \text{Initialize} global parameters $\theta_0$
\For{round $t\, = 1,2,...,T$}
    \State $\mathcal{I}_t \gets$ (Select a random subset of  clients)
    \State \text{Receive}
                $\bg_i$ from each \texttt{client $i$} $\in$ $\mathcal{I}_t$
    \State \text{Aggregate:}
                    $ \theta_{t} \leftarrow \theta_{t-1} - \rho \frac{I}{r}\sum_{i \in \mathcal{I}_t} \alpha_i \bg_i$
\EndFor
\EndIndent
\\
\State \textbf{ClientI\_Update($\theta_{t-1}$):}\#runs on client $i$
            \Indent
                        \State \text{Initialize} $W_i$, the first time client $i$ is visited

                        \For{local gradient
                    update $=1,2,...,\tau-1$}
          \State $W_{i}$ $\gets$ $W_{i} - \beta$ $\nabla_{ W_{i} }$$\ell_i(W_i,\theta_{t-1})$
          
                    \EndFor
                           \State \text{Compute} $( \nabla_{W_i} \ell_i( W_i, \! \theta_{t-1}), \nabla_{\theta_{t-1}} \ell_i( W_i,\theta_{t-1}))$
       
\State $W_i \leftarrow W_i - \rho \frac{I}{r} \nabla_{ W_{i}} \ell_i(W_i,\theta_{t-1})$
        \State \text{Return} $\bg_i \triangleq \nabla_{\theta_{t-1}} \ell_i(W_i, \theta_{t-1})$ to server 
            \EndIndent
\end{algorithmic}
\end{algorithm}

\subsection{Exact Stochastic Gradient Descent Optimization 
\label{sec:convergence}
}

An important property of our proposed algorithm is that the final iteration $\tau$ over $W_i$'s at the selected set of clients $\mathcal{I}_t$, combined with the update over the global parameter $\theta_{t-1}$ at the server results in an unbiased SGD step over all parameters $\psi_{t-1} = \{\theta_{t-1}, W_{1,t-1},\ldots, W_{I,t-1}\}$. To prove this rigorously, we introduce the stochastic 
gradient vector $\nabla^s_{\psi_{t-1}} \mathcal{L}(\psi_{t-1}) = \{ \nabla_{\theta_{t-1}}^s \mathcal{L}(\psi_{t-1}), \nabla_{W_1}^s \mathcal{L}(\psi_{t-1}),\ldots, \nabla_{W_I}^s \mathcal{L}(\psi_{t-1})\}$ where we use the symbol $s$ in $\nabla^s$ to indicate that these gradient vectors are stochastic. For any client  $i=1,\ldots,I$ the vector $\nabla_{W_{i,t-1}}^s \mathcal{L}(\psi_{t-1})$ is defined as
\begin{align}
\nabla_{W_{i,t-1}}^s \! \mathcal{L}(\psi_{t-1})
& \! = \!  {\bf 1}_{i \in \mathcal{I}_t} 
\frac{I}{r} \nabla_{W_{i,t-1}} \mathcal{L}(\psi_{t-1}) \nonumber \\
& \! = \! {\bf 1}_{i \in \mathcal{I}_t} 
\frac{I}{r} \alpha_i \nabla_{W_{i,t-1}} 
\! \ell_i(W_{i,t-1},\theta_{t-1}). 
\end{align}
Here, ${\bf 1}_{i \in \mathcal{I}_t}$ denotes the indicator function that equals one if $i \in \mathcal{I}_t$, i.e., if client $i$ was selected in round $t$, and zero otherwise. We also used the fact that $\nabla_{W_{i,t-1}} \mathcal{L}(\psi_{t-1}) = \alpha_i \nabla_{W_{i,t-1}} \ell_i(W_{i,t-1},\theta_{t-1})$, which follows from Equation (\ref{eq:fullloss}). Note that for selected clients in the set $\mathcal{I}_t$, the corresponding vector, $\nabla_{W_{i,t-1}}^s \! \mathcal{L}(\psi_{t-1}) = \frac{I}{r} \alpha_i \nabla_{W_{i,t-1}} \ell_i(W_{i,t-1},\theta_{t-1})$ is precisely the gradient used in the client update in \eqref{eq:clientupdate}, while for the remaining clients $i \notin \mathcal{I}_t$, $\nabla_{W_{i,t-1}}^s \mathcal{L}(\psi_{t-1}) = 0$. The stochastic gradient $\nabla_{\theta_{t-1}}^s \mathcal{L}$ is defined as
\begin{equation}
\nabla_{\theta_{t-1}}^s \! \mathcal{L}(\psi_{t-1}) \! = \! \frac{I}{r}
\! \sum_{i=1}^I {\bf 1}_{i \in \mathcal{I}_t} 
\alpha_i \nabla_{\theta_{t-1}} 
\! \ell_i(W_{i,t-1},\theta_{t-1}). 
\end{equation}
We can see that in Algorithm \ref{alg:PFLEGO}, the final client update together with the server update can be compactly written as the following gradient update over all parameters $\psi$:
$$
\psi_{t} \leftarrow \psi_{t-1} - \rho
\nabla^s_{\psi_{t-1}} \mathcal{L}(\psi_{t-1}).
$$
This is now a proper SGD step as long as the stochastic gradient  $\nabla^s_{\psi_{t-1}} \mathcal{L}(\psi_{t-1})$ is unbiased, as we state next.

\begin{prop}
The stochastic gradient $\nabla^s_{\psi_{t-1}} \mathcal{L}(\psi_{t-1})$ is unbiased, i.e.,
$\Exp [\nabla^s_{\psi_{t-1}} \mathcal{L}(\psi_{t-1})]
\! = \! \nabla_{\psi_{t-1}} \mathcal{L}(\psi_{t-1})$ where 
$\nabla_{\psi_{t-1}} \mathcal{L}(\psi_{t-1})$ denotes the exact gradient and the expectation is taken under the client participation process (either case a or b) defined in Section \ref{sec:client-selection}.
\label{prop:1}
\end{prop}
\begin{proof}
By taking the expectation 
$\Exp [\nabla^s_{W_{i,t-1}} \mathcal{L}(\psi_{t-1})]$ for any $i$, and the expectation $\Exp[ \nabla^s_{\theta_{t-1}} \mathcal{L}(\psi_{t-1})]$, the indicator function ${\bf 1}_{i \in \mathcal{I}_t}$ is replaced by its expected value  $\text{Pr}(i \in \mathcal{I}_t) = \frac{r}{I}$ (this value is the same for cases a and b), which gives the exact gradient.
\end{proof}



\subsection{Convergence Rate}
\label{convergence_rate}

In this section we
simplify notation to write $\nabla_{\theta} \mathcal{L} \triangleq \nabla_{\theta} \mathcal{L}(\psi)$, which is the gradient of the full loss over the shared parameters
$\theta$. Similarly, we also denote the corresponding gradient for a client specific loss by  
$\nabla_{\theta} \ell_i \triangleq \nabla_{\theta} \ell_i(W_i, \theta)$.

In Proposition \ref{prop:2} we state our convergence result, which holds also in non-convex settings \cite{haoyu}, including how to set the values for $\rho$, and $\beta$ at the server-side and at the client-side, respectively, to ensure convergence. We follow the same convention as in \cite{haoyu,ghadimi,lian2017,alistarh} where the average expected squared gradient norm is used to characterize the convergence rate.
    
\begin{prop}
    \label{prop:2}

    If the following assumptions hold:
    \begin{itemize}
        \item Lipschitz continuity of $\nabla_\theta \mathcal{L}$ w.r.t. the $\ell_2$ norm, with a Lipschitz constant L: \begin{itemize}
            \item[] $\| \nabla_\theta \mathcal{L} - \nabla_{\theta'} \mathcal{L} \| \leq L \| \theta - \theta'\|.$
        \end{itemize}
        \item The $\ell_2$ norm of gradients is bounded by a constant $G$: \begin{itemize}
            \item[] $\|\nabla_\theta \ell_i\| \leq G.$
        \end{itemize}
        \item The learning rates $\beta$ and $\rho$ satisfy the bounds: \begin{itemize}
            \item[] $0 < \beta < \frac{2}{L}$ and $0 < \rho < \frac{2}{L}\frac{r}{\mathcal{I}}$,
        \end{itemize}
    \end{itemize} then we have that PFLEGO's parameters are guaranteed to converge to a stationary point. For the convergence, we also use Proposition \ref{prop:1}, in which, the computed gradients are unbiased gradients, i.e., $\mathbb{E} [\nabla_\theta \ell_i ] = \nabla_\theta \mathcal{L}$.
    
    
    
     

\end{prop}

\begin{corollary}
    \label{corol:1}
        We establish two convergence rates from the main result of Proposition \ref{prop:2}.
        \begin{itemize}
            \item If we set $\rho = \frac{1}{L}\frac{\sqrt{r}}{\sqrt{T}\sqrt{I}}$, then
        PFLEGO achieves a convergence rate of $\mathcal{O} \left (\sqrt{\frac{I}{T}} \right )$.
            \item If we further set $r=I$, then
        PFLEGO achieves a convergence rate of $\mathcal{O} \left (\frac{1}{\sqrt{T}} \right )$.
        \end{itemize}
         Both convergence rates indicate that the performance progressively improves and PFLEGO gets closer to the optimal solution as the number of rounds $T$ increases.
\end{corollary}

\begin{proof}
For the proofs of Proposition \ref{prop:2}, and of Corollary \ref{corol:1}, see Appendix \ref{appendix:proof}.
\end{proof}

\subsection{Computational Complexity}

In this section we discuss the computational advantage of PFLEGO against other methods such as FedAvg and FedPer, even when the number of gradient update steps $\tau > 1$ at the client-side. Notably, PFLEGO roughly runs $\frac{\tau}{2}$ times faster than the previously mentioned methods, which becomes significant, when the minimization of energy consumption is important, i.e., the size of the model parameter vector is large, and the NN has a large number of layers.

{In PFLEGO, each client performs two full forward passes and one full backward pass through the neural network per communication round, regardless of the value of $\tau$.  This reduction in computation is possible since,} during the first $\tau-1$ client updates the global parameters $\theta$ are fixed. 
Indeed, at the beginning of each round, we pass once the data from the NN, store all feature vectors, and then carry out $\tau-1$ GD steps to update only the client-specific parameters. For the final $\tau$-th iteration, we need to pass the data for a second time from the NN  to compute the joint gradient. 
Formally, given that the complexity per round is dominated by the NN evaluations, then PFLEGO is $O(1)$ while others are $O(\tau)$. {By contrast, baseline methods such as FedAvg \cite{mcmahan2017communicationefficient} and FedPer \cite{arivazhagan2019federated} require a full forward and backward pass through the entire network for each of the $\tau$ local update steps, resulting in $\tau$ such passes per round.} Consequently in our algorithm the clients consume much less energy per round, which can be useful for applications with energy constraints.

{On the server side, only a subset of $r \leq I$ clients participates in each round. The server aggregates updates from these 
$r$ clients and performs a single gradient update of the global parameters. Aggregation involves summing $r$ vectors of size 
$|\theta|$, which has complexity $O(r\cdot|\theta|)$. The subsequent parameter update adds an additional 
$O(|\theta|)$, resulting in a total server-side complexity per round of $O(r \cdot |\theta|)$. Thus, while PFLEGO significantly reduces per-round computation on clients, the server-side cost scales linearly with the number of participating users and the number of global parameters, similar to FedAvg.}

\section{Experiments} \label{experiments}

In order to evaluate the performance benefits of the proposed PFLEGO algorithm, we compare it against state-of-the-art algorithms, in particular, we compare against the proposed personalized algorithms: \begin{itemize}
    \item (\textit{a}) FedRep and FedBabu\cite{fedbabu} which have the FedPer\cite{arivazhagan2019federated} NN architecture,
    \item (\textit{b}) personalized variants of FedAvg \cite{mcmahan2017communicationefficient}: one that uses a regularization term such as FedProx\cite{li2020federated}, and Ditto\cite{ditto} which trains 2 NN per client, and
    \item (\textit{c}) FedRecon \cite{singhal2021federated}, which follows a stochastic block coordinate descent scheme, as opposed to our SGD approach.
 \end{itemize}
We experiment with the Omniglot, CIFAR-10, MNIST, Fashion-MNIST and EMNIST datasets.


{At the client-side we use the GD optimizer, and at the server side we use the Adam optimizer \cite{kingma2017adam} as it offers stability compared to the GD optimizer.} {In our case, it is natural to use Adam \cite{kingma2017adam} at the server since PFLEGO performs a stochastic gradient descent (SGD) update, whereas other FL algorithms such as FedAvg \cite{mcmahan2017communicationefficient} and FedPer \cite{arivazhagan2019federated} simply compute an average of weights which does not require an optimizer.}

We report training loss (given by \eqref{eq:fullloss}) and test accuracy for the compared methods. Plotting the training loss shows how fast each FL algorithm optimizes the model, i.e., how many rounds are required, while test accuracy quantifies predictive classification performance. Results are averages over all clients.



\subsection{Dataset and NN Architecture Description}
\label{appendix:datasetdesciption}

{We provide the description of the datasets used in our experiments, along with their corresponding
neural network architecture. For further details about the neural network architectures see Table \ref{tab:architectures_altogether} and Figure \ref{fig:nn_architectures} in the Appendix \ref{appendix:neural_network_architectures}.}

    \textit{Omniglot}. This dataset was introduced in
    \cite{Lake2015HumanlevelCL} and consists of 1623 $105\times105$ handwritten characters from 50 different alphabets, with 20 samples per handwritten character. Omniglot can be a natural choice for Personalized learning due to its small number of samples per character and large number of different handwritten characters per alphabet.
    Each alphabet can be considered as a classification problem with a certain number of classes, e.g., the English  alphabet has $26$ classes. We use $4$ convolutional layers, each layer followed by one max pooling layer; the architecture we use is the same as the one from \cite{finn2017modelagnostic}.
 
 \textit{CIFAR-10}. This dataset consists of $32 \times 32$ RGB images of 10 different classes of visual categories \citep{krizhevsky2009learning}. 
    We use the same architecture  as in \cite{yao2020federated}, which includes two convolutional layers of $64$ filters each, and a kernel of size $5$. Each layer is followed by a max pooling layer of size $3\times3$ with stride $2$. The output of the convolutional layers passes through two additional fully connected layers of size $384$ and $192$. The activation function for the convolutional and fully connected layers is ReLU.

\textit{MNIST}. This dataset consists of $28 \times 28$ handwritten grayscale images of single digits from classes $0$ to $9$. For MNIST, we use an MLP architecture that consists of one fully connected layer of $200$ units with a ReLU activation function.
    
    \textit{Fashion-MNIST}. This dataset is similar to MNIST but more challenging 
    and consists of $28 \times 28$ grayscale images of $10$ different classes of  clothing. We use the same MLP architecture as the one in MNIST.
   
   \textit{EMNIST}. This dataset extends the MNIST dataset with grayscale handwritten digits. In total, there are $62$ different classes of handwritten letters and digits. We use the same MLP architecture as the one in MNIST.


\subsection{Experimental Setup
\label{appendix:experimental_setup}
}

For Omniglot, we assume that a single alphabet is stored in each client. Due to the fact that each handwritten system of each alphabet is unique, there is no class label set overlap among the clients, which makes Omniglot the hardest and most personalized FL problem in our experiments. At each alphabet, data of each class are split into $75\%$ used for training and $25\%$ used for testing. Also, standard data augmentation is used by including rotated (by multiples of $90^o$) image samples \citep{finn2017modelagnostic}. The setup we follow is based on \cite{shamsian2021personalized} that uses $I=50$ clients, $\tau=50$ inner steps per client, and $T=5,000$ communication rounds, and each client has a unique alphabet. For the FedAvg algorithm, the final layer of the common weights is set to 55 outputs, which is equal to the maximum number of classes among all $50$ alphabets.

For MNIST, Fashion-MNIST, EMNIST, and CIFAR-10, we use $I= 100$ clients, $\tau=50$ inner steps per client, and $T=200$ communication rounds. MNIST, Fashion-MNIST, and CIFAR-10 are well-balanced datasets and contain 10 classes. EMNIST is more challenging since it has 62 classes and varying number of examples per class. We simulate several FL scenarios by varying the amount of task-personalization among clients. This involves varying the degree of class label set overlap among clients, so that different clients can have different classes in their private datasets.

\subsection{Different Degrees of Personalization}
\label{appendix:degreeofpersonalization}
For all datasets, except Omniglot, which is personalized by design, i.e., for MNIST, CIFAR-10, Fashion-Mnist and EMNIST, we artificially simulate personalized FL problems by varying the degree of personalization. 
This is quantified by the size $K$ of classes randomly assigned to each client from the total set of classes, so that
the smaller $K$ is, the higher the degree of personalization is, since the probability  of class overlap among clients reduces with smaller $K$ values. 

We consider three degrees  
of personalization. 
\begin{itemize}
    \item (\textit{i}) High-Pers where each client has $K=2$ randomly chosen classes from the total set of $C$ classes ($C=10$ for  MNIST, CIFAR-10, Fashion-MNIST and $C=62$ for EMNIST),
    \item (\textit{ii}) Medium-Pers where $C/2$ classes are randomly assigned to each client, and
    \item (\textit{iii}) No-Pers where all clients  have data points from all $C$ classes, i.e., all clients solve the same task.
\end{itemize} 



\begin{table*}[ht]
\caption{Test accuracy for MNIST, CIFAR-10, EMNIST, and the Fashion-MNIST datasets for different degrees of personalization.}
\label{table:all_datasets_test_acc}
    \begin{center}
    \begin{small}
  \resizebox{\textwidth}{!}{%
        \begin{tabular}{llllllll}
        \toprule
         &  & \multicolumn{ 3}{c}{MNIST} & \multicolumn{ 3}{c}{CIFAR-10} \\ 
         \midrule
        Method / Deg of Pers.  &  & \multicolumn{1}{c}{High-Pers} & \multicolumn{1}{c}{Medium-Pers} & \multicolumn{1}{c}{No-Pers} & \multicolumn{1}{c}{High-Pers} & \multicolumn{1}{c}{Medium-Pers} & \multicolumn{1}{c}{No-Pers} \\ 
        
        FedPer &  & 97.88 $\pm$ 0.25 & 92.83 $\pm$ 0.46 & 87.23 $\pm$ 0.33 & 85.15 $\pm$ 1.08 & 61.01 $\pm$ 0.84 & 37.88 $\pm$ 0.59 \\ 
        FedAvg &  & 97.54 $\pm$ 0.25 & 92.83 $\pm$ 0.52 & 93.12 $\pm$ 0.24 & 85.18 $\pm$ 0.96 & 64.53 $\pm$ 0.75 & 61.33 $\pm$ 0.52 \\

        FedProx &  & 97.50 $\pm$ 0.26 & 92.62 $\pm$ 0.54 & 93.02 $\pm$ 0.25 & 85.94 $\pm$ 0.87 & 64.07 $\pm$ 0.60 & 60.00 $\pm$ 0.62 \\

    Ditto &  & 98.25 $\pm$ 0.12 & 94.05 $\pm$ 0.47 & 92.98 $\pm$ 0.25 & 85.89 $\pm$ 0.98 & 64.18 $\pm$ 0.78 & 60.67 $\pm$ 0.49 \\

    FedRep &  & 98.03 $\pm$ 0.23 & 93.85 $\pm$ 0.47 & 89.28 $\pm$ 0.28 & 84.03 $\pm$ 0.83 & 66.64 $\pm$ 0.67 & 46.52 $\pm$ 0.96 \\

        FedBabu &  & 97.19 $\pm$ 0.33 & 92.46 $\pm$ 0.46 & \textbf{93.98 $\pm$ 0.24} & 87.88 $\pm$ 0.66  & 66.64 $\pm$ 0.68 & 56.48 $\pm$ 0.51 \\
        
        PFLEGO &  & \textbf{98.70 $\pm$ 0.21} & \textbf{95.16 $\pm$ 0.32} & 91.37 $\pm$ 0.19 & \textbf{87.81 $\pm$ 0.94} & \textbf{74.83 $\pm$ 0.66} & \textbf{63.31 $\pm$ 0.61} \\ 
        \midrule
         &  & \multicolumn{ 3}{c}{EMNIST} & \multicolumn{ 3}{c}{Fashion-MNIST} \\ \midrule
        Method / Deg of Pers.  &  & \multicolumn{1}{c}{High-Pers} & \multicolumn{1}{c}{Medium-Pers} & \multicolumn{1}{c}{No-Pers} & \multicolumn{1}{c}{High-Pers} & \multicolumn{1}{c}{Medium-Pers} & \multicolumn{1}{c}{No-Pers} \\ 
        FedPer &  & 97.78 $\pm$ 0.51 & 74.19 $\pm$ 0.36 & 48.12 $\pm$ 0.23 & 96.14 $\pm$ 0.35 & 88.22 $\pm$ 0.64 & 77.44 $\pm$ 0.59 \\ 
        FedAvg &  & 97.29 $\pm$ 0.54  & 68.82 $\pm$ 0.29 & \textbf{69.40 $\pm$ 0.10} & 96.35 $\pm$ 0.47 & 87.51 $\pm$ 0.73 & 83.59 $\pm$ 0.35 \\ 

        FedProx &  & 97.24 $\pm$ 0.54 & 54.28 $\pm$ 0.58 & 69.20 $\pm$ 0.09 & 95.21 $\pm$ 0.48 & 87.44 $\pm$ 0.78 & 83.54 $\pm$ 0.38 \\

        Ditto &  & 98.11 $\pm$ 0.48 & 67.24 $\pm$ 0.28 & 69.11 $\pm$ 0.09 & 96.45 $\pm$ 0.47 & 89.30 $\pm$ 0.50 & 83.48 $\pm$ 0.34 \\

        FedRep &  & 98.20 $\pm$ 0.45 & 77.62 $\pm$ 0.34 & 65.23 $\pm$ 0.18 & 95.99 $\pm$ 0.38 & 89.10 $\pm$ 0.68 & 79.18 $\pm$ 0.33 \\

        FedBabu &  & 97.08 $\pm$ 0.58 & 60.58 $\pm$ 0.34 & 71.82 $\pm$ 0.16 & 95.68 $\pm$ 0.44  & 88.00 $\pm$ 0.72 & \textbf{84.12 $\pm$ 0.43} \\
        
        PFLEGO &  & \textbf{98.79 $\pm$ 0.36} & \textbf{77.75 $\pm$ 0.33} & 60.79 $\pm$ 0.15 & \textbf{96.34 $\pm$ 0.43} & \textbf{89.84 $\pm$ 0.52} & 81.49 $\pm$ 0.51 \\ 
        \bottomrule
        \end{tabular}
         }
    \end{small}
    \end{center}
\end{table*}

\begin{table*}[!htbp]
\caption{Test accuracy for the Omniglot dataset, which exhibits a high degree of personalization. The average is measured by averaging the last 10 global rounds of the first 1000 global rounds.}
\label{omniglot_single_table_results}
\begin{center}
\begin{small}
\resizebox{\textwidth}{!}{%
\begin{tabular}{cccccccccc}
\toprule
\multicolumn{8}{c}{Omniglot}
\\
\midrule
FedRep &
FedBabu &
FedProx & Ditto & FedPer & FedAvg & FedRecon & PFLEGO\\ 
69.65 $\pm$ 2.47 & 68.45 $\pm$ 1.53 & 48.41 $\pm$ 1.83 & 50.88 $\pm$ 1.72 & 68.02 $\pm$ 1.74 & 49.65 $\pm$ 1.74 & 74.06 $\pm$ 1.19 & \textbf{75.85 $\pm$ 1.21} \\
\bottomrule
\end{tabular}
}
\end{small}
\end{center}
\end{table*}

\begin{table*}[!htbp]
\caption{Average running time (in seconds per round) for each algorithm on Omniglot. Lower is better.}
\label{omniglot_time_results}
\begin{center}
\begin{tabular}{ccccccc}
\toprule
\multicolumn{7}{c}{Omniglot}
\\

\midrule
FedRep & FedBabu & FedProx & Ditto & FedPer & FedAvg & PFLEGO \\
10.571 & 16.814 & 18.632 & 26.105 & 14.436 & 16.553 & {\bf7.024} \\
\bottomrule
\end{tabular}
\end{center}
\end{table*}

\subsection{Results and Discussion}
\label{sec:results_and_discussion}
Table \ref{table:all_datasets_test_acc}
reports test accuracy scores for the aforementioned state-of-the-art algorithms and our approach PFLEGO 
for all datasets, except Omniglot, and three degrees of personalization (high/medium/no personalization). 
Table \ref{omniglot_single_table_results} reports the  test accuracy for Omniglot. Each reported accuracy value and confidence interval is the mean of the corresponding values at the final 10 global rounds. In all tables the best
performing method is indicated in bold font. We observe that our algorithm has significantly higher accuracy from the other algorithms for high degree of personalization; see High-Pers columns in Table \ref{table:all_datasets_test_acc} and the Omniglot results in Table \ref{omniglot_single_table_results}. In the case of Medium-Pers, our algorithm has better performance than the baselines in all datasets as well. Note the significant difference in performance in CIFAR-10, where our method has achieved $8\%$ higher accuracy than all other algorithms.   
{Finally, in the case of No-Pers, the orthogonal initialization of personalized weights in FedBabu outperforms even FedAvg, except on CIFAR-10, where PFLEGO obtains higher test accuracy than all methods. The superior performance of PFLEGO can be attributed to its use of unbiased exact SGD updates. Unlike methods such as FedAvg and FedPer that rely on averaging locally updated model weights, PFLEGO directly aggregates unbiased gradients, ensuring that each global update moves consistently toward the true optimum of the overall objective. We also see that PFLEGO exhibits substantially lower per-
round wall-clock time as shown in Table \ref{omniglot_time_results}}.

To visualize the learning speed, Figures \ref{fig:main_mnist_plot}, \ref{fig:main_fashion_mnist_plot}, \ref{fig:main_cifar10_plot}, \ref{fig:main_emnist_plot} show the training loss and test accuracy with respect to the number of rounds for MNIST, Fashion-MNIST, CIFAR-10, and EMNIST respectively for different degrees of personalization. Figure \ref{fig:main_omniglot_plot} plots the same quantities for the Omniglot dataset. These plots clearly indicate that PFLEGO achieves high classification accuracy faster in the high-personalized regime and requires fewer communication rounds to minimize the overall training loss.

\begin{figure*}[ht]
\begin{center}
    \begin{tabular}{ccc}
     \multicolumn{3}{c}{{\includegraphics[scale=\myscaleplottitle]
    {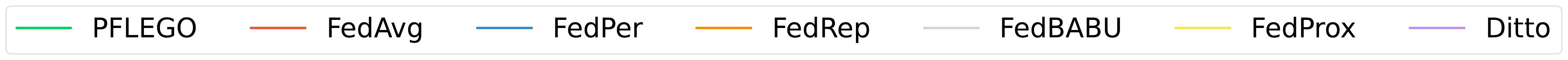}}}
    \end{tabular}\\
    \begin{tabular}{ccc}
    {\includegraphics[scale=\myscaleplotsubplot]
    {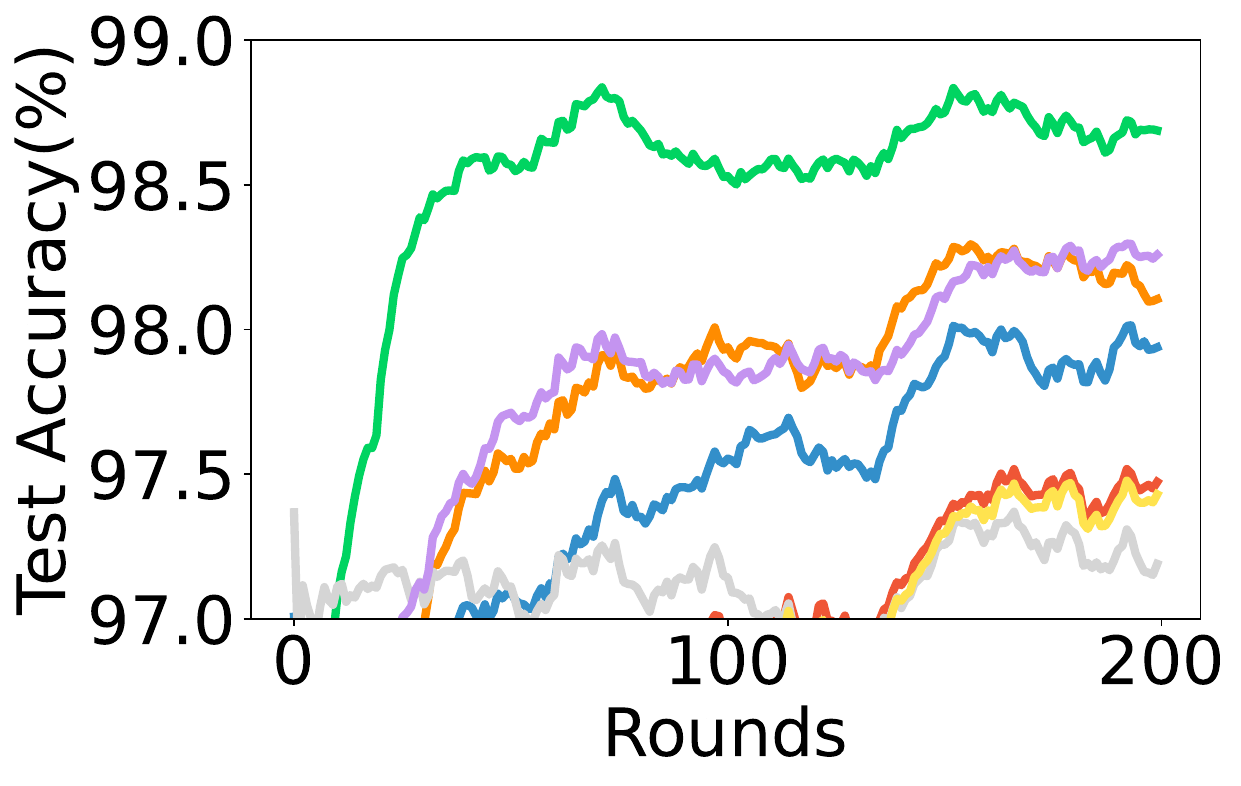}} &
    {\includegraphics[scale=\myscaleplotsubplot]
    {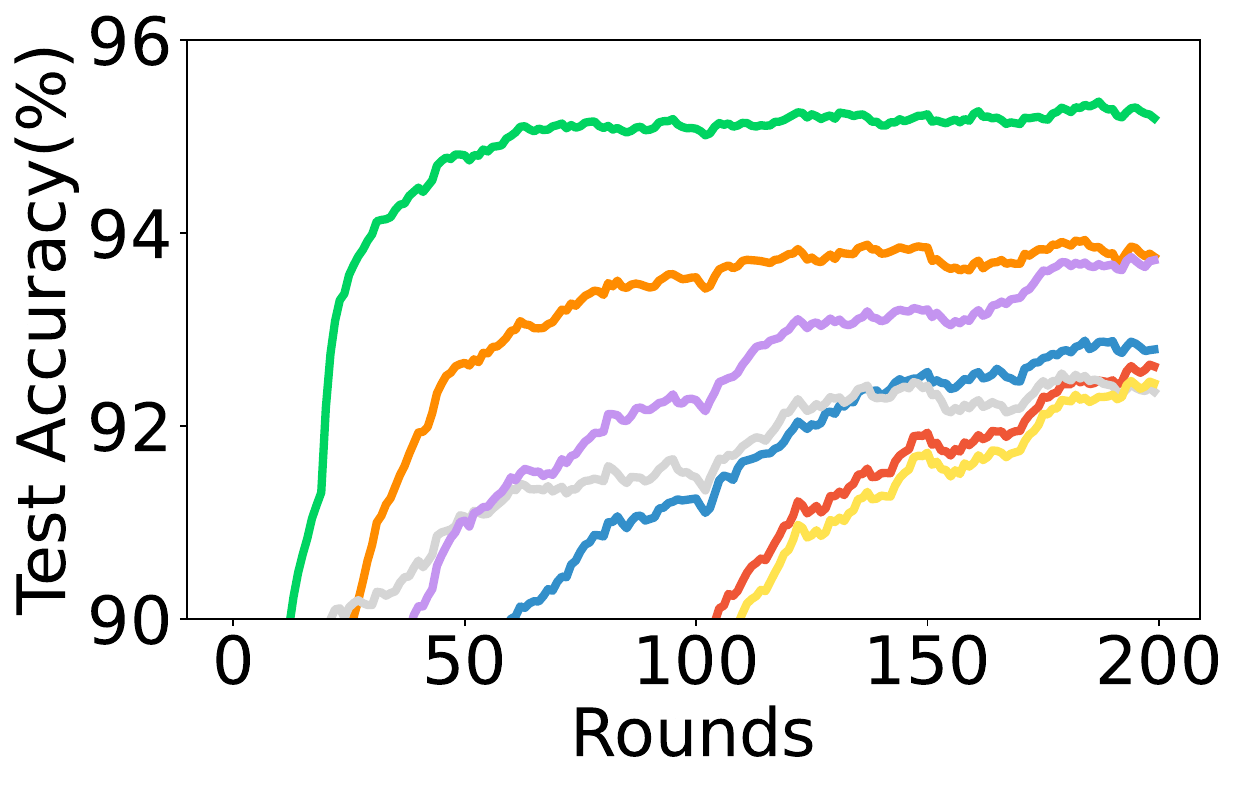}} &
    {\includegraphics[scale=\myscaleplotsubplot]
    {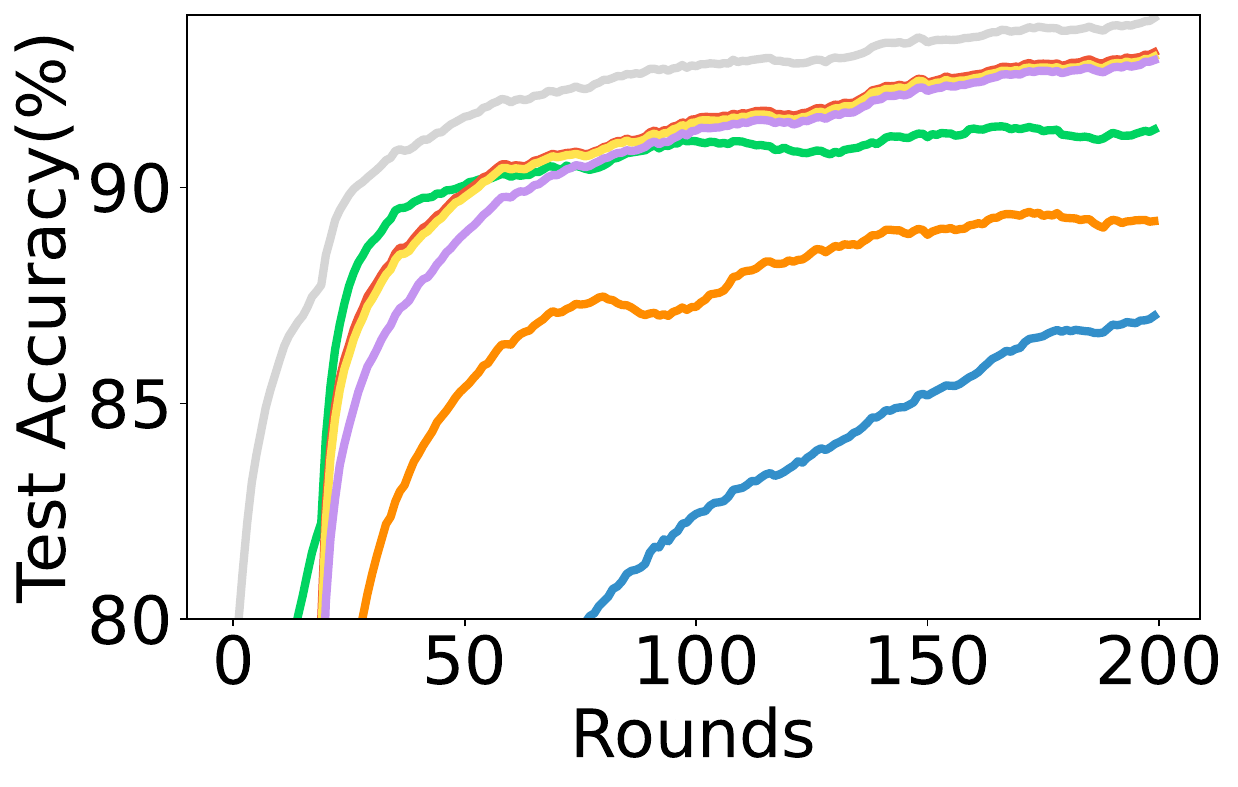}} \\
    {\includegraphics[scale=\myscaleplotsubplot]
    {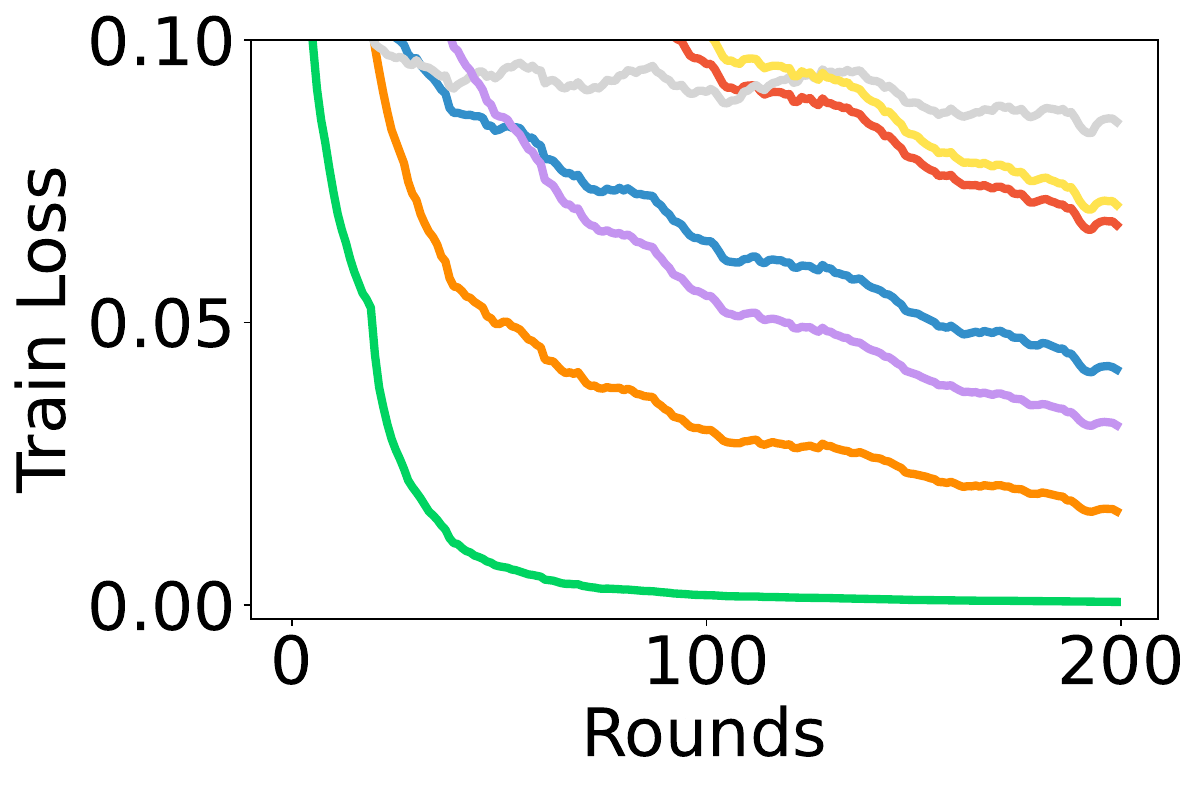}} &
    {\includegraphics[scale=\myscaleplotsubplot]
    {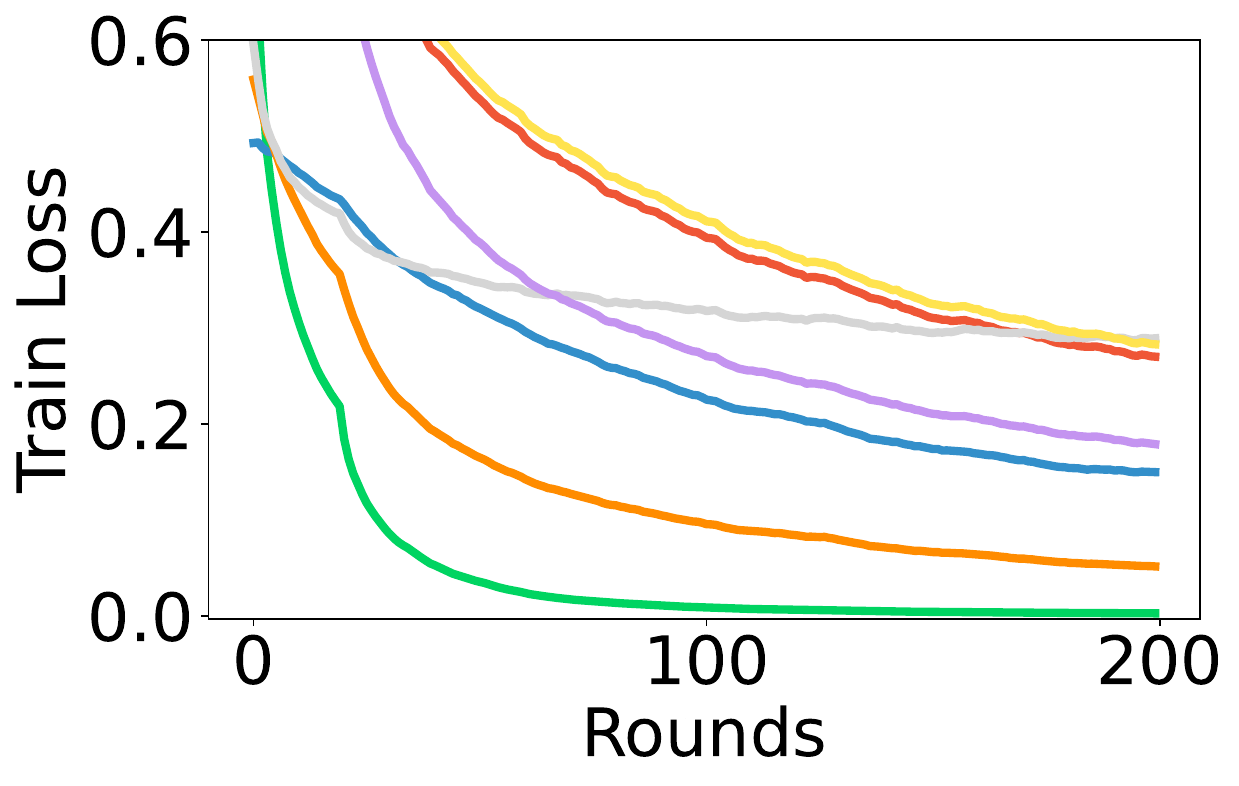}} &
    {\includegraphics[scale=\myscaleplotsubplot]
    {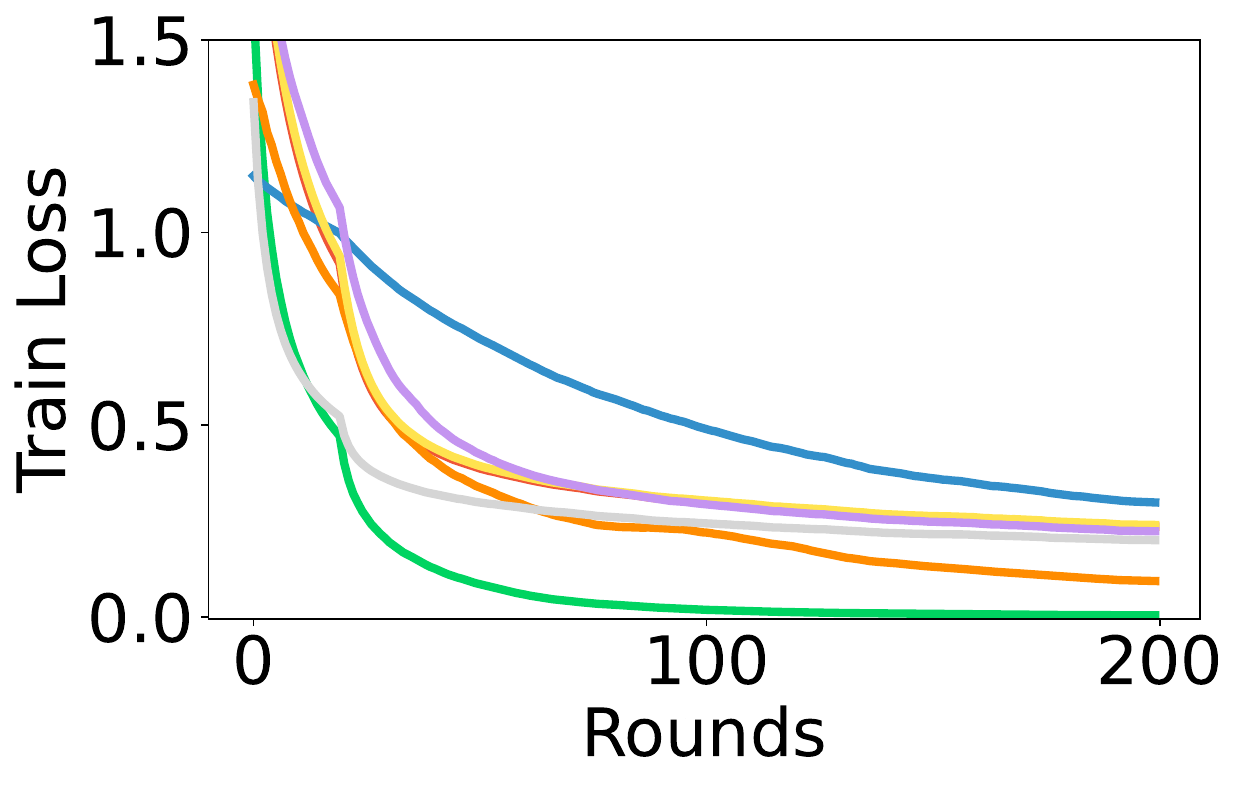}} \\
    High-pers  & Medium-pers & No-pers
    \end{tabular}
       \caption{Test accuracy (top row) and training loss (bottom row) for MNIST dataset for PFLEGO, FedAvg and FedPer over $T=200$ {communication} rounds, $I=100$ clients, and $r=20\%$ client participation per round. Each column corresponds to a degree of personalization  (high, medium, no personalization). 
        }
      \label{fig:main_mnist_plot}
  \end{center}
\end{figure*}

\begin{figure*}[ht]
\begin{center}
\begin{tabular}{ccc}
    \multicolumn{3}{c}{{\includegraphics[scale=\myscaleplottitle]
   {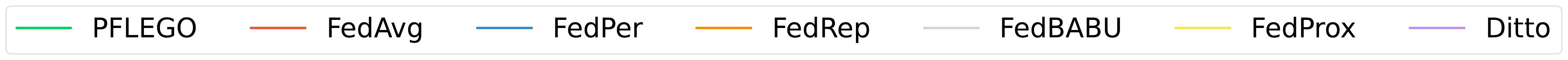}}}
   \end{tabular}\\
\begin{tabular}{ccc}
   {\includegraphics[scale=\myscaleplotsubplot]
   {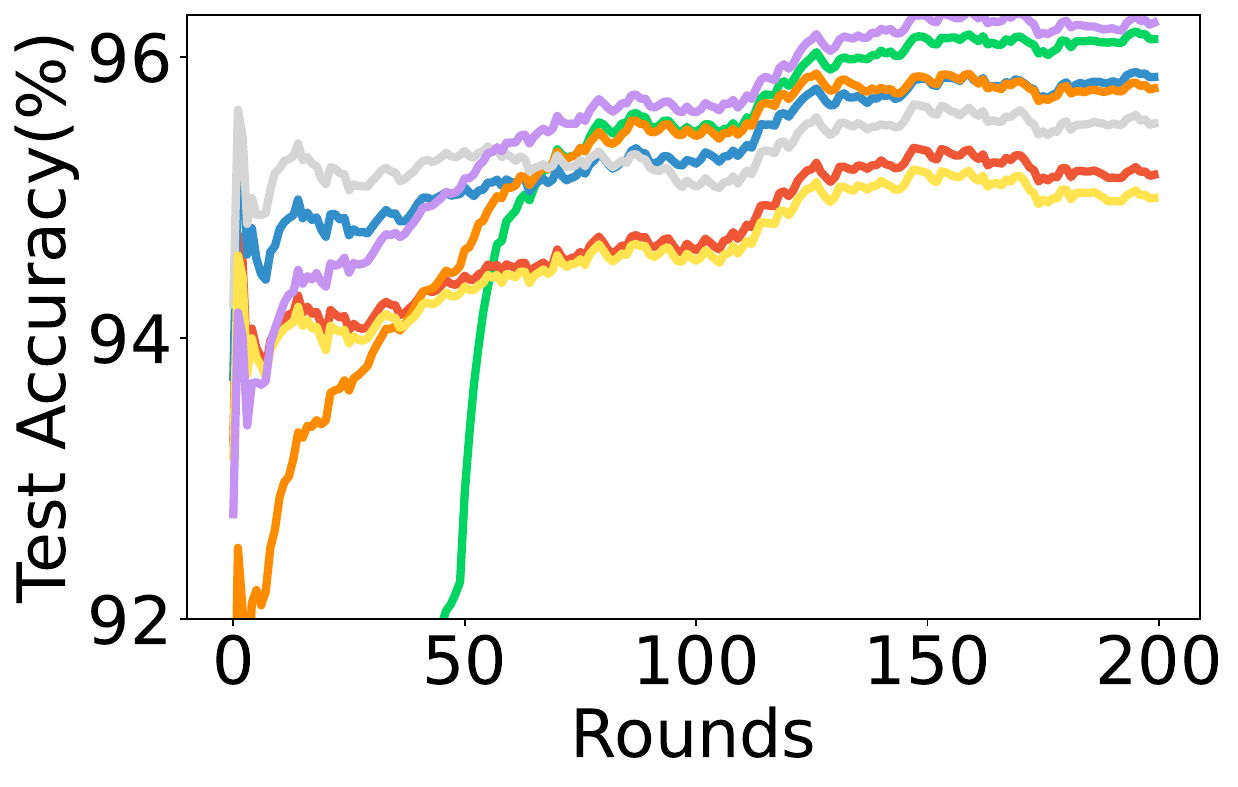}} &
   {\includegraphics[scale=\myscaleplotsubplot]
   {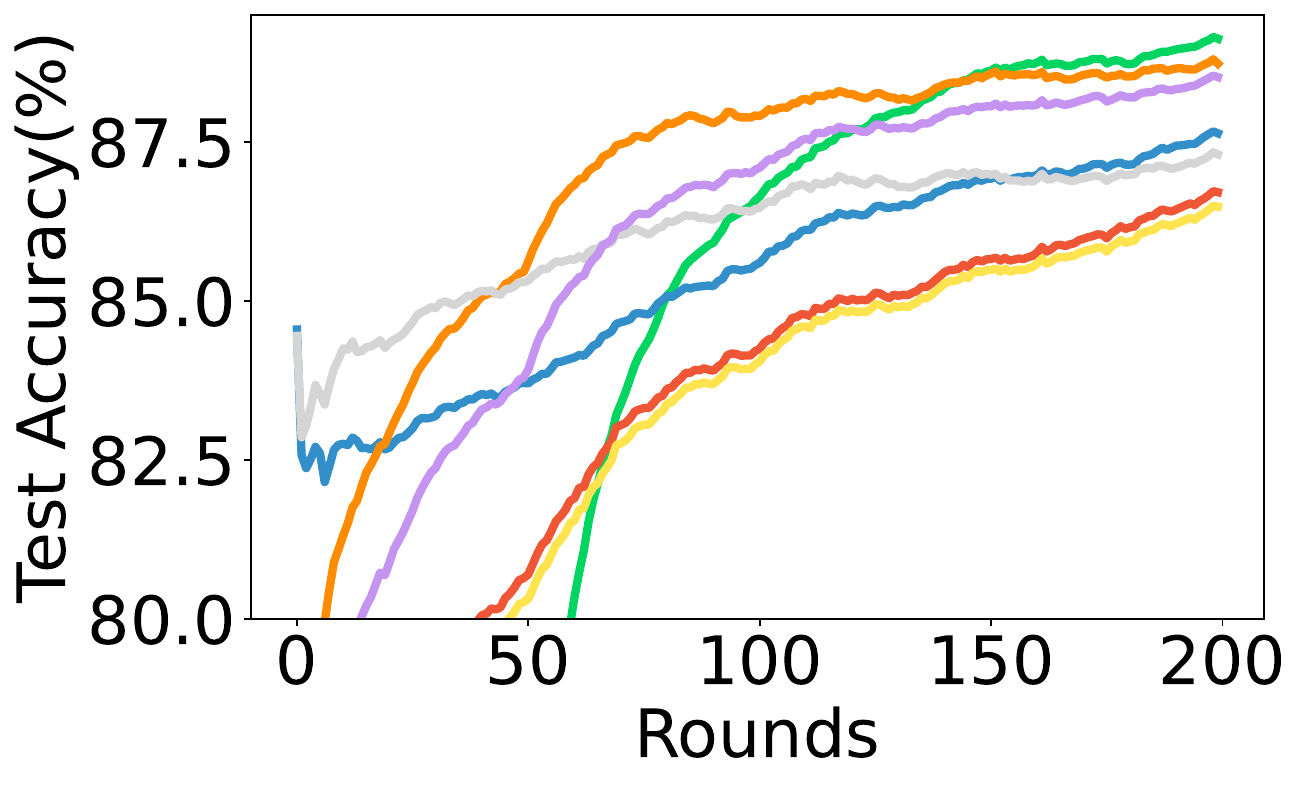}} &
   {\includegraphics[scale=\myscaleplotsubplot]
   {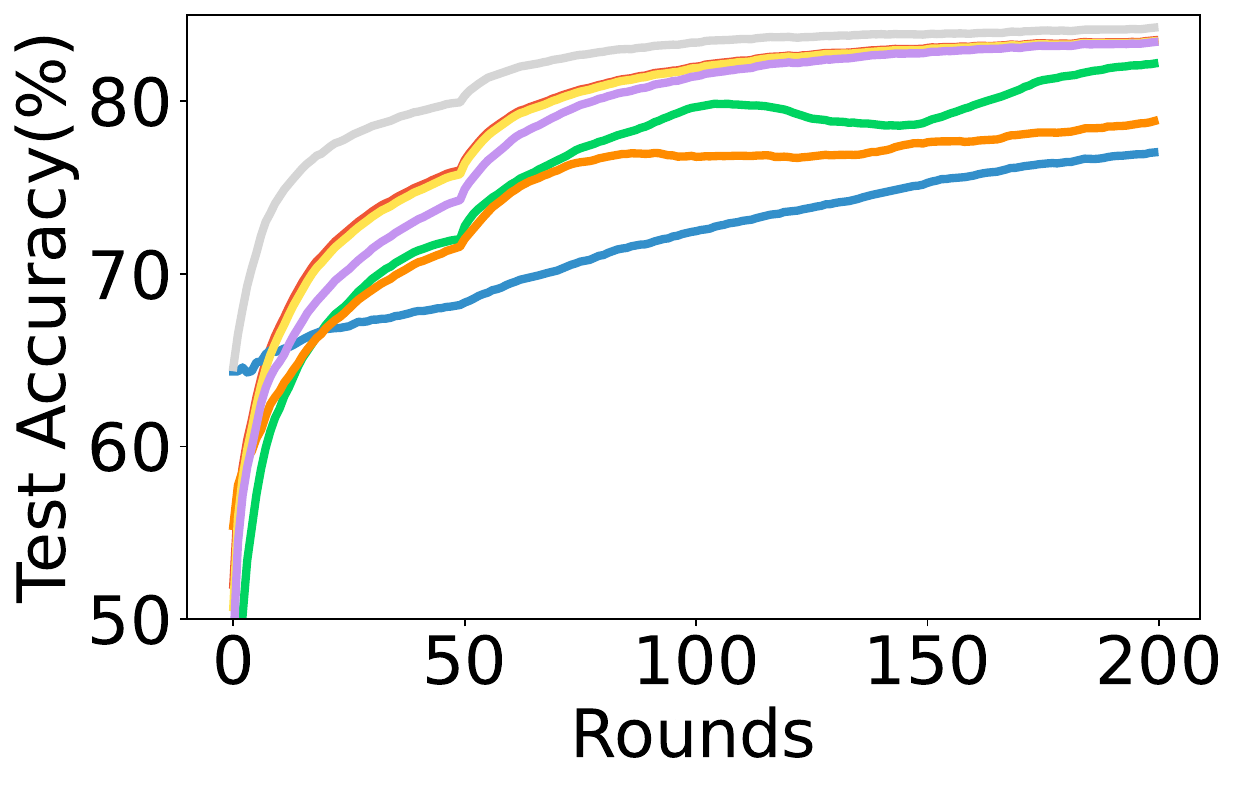}} \\
   {\includegraphics[scale=\myscaleplotsubplot]
   {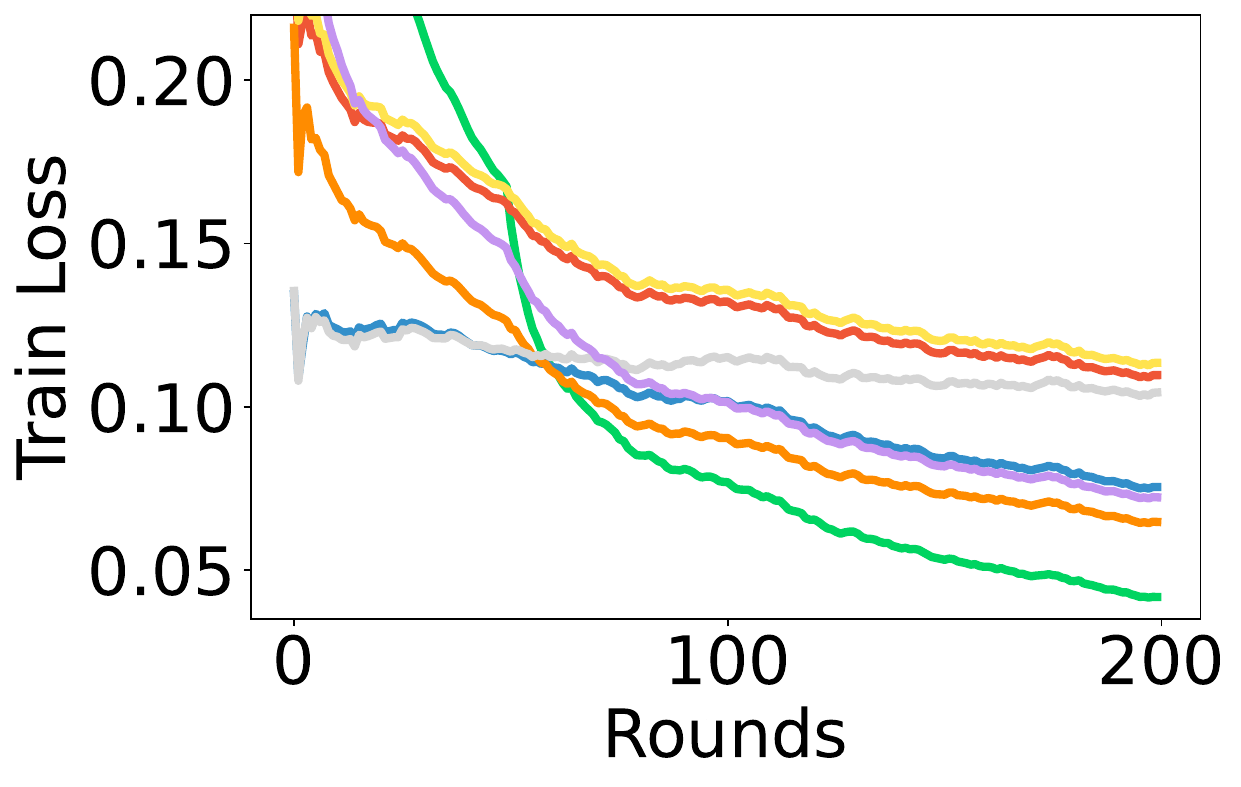}} &
   {\includegraphics[scale=\myscaleplotsubplot]
   {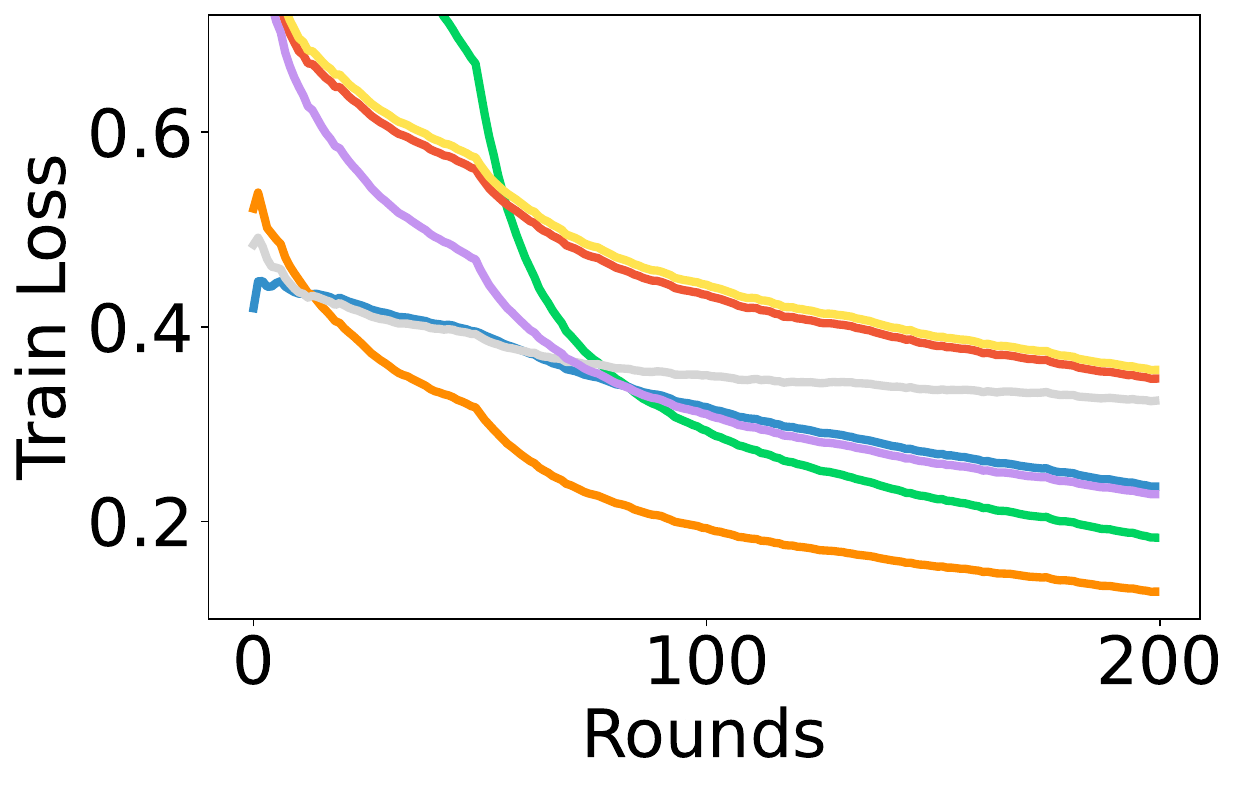}} &
   {\includegraphics[scale=\myscaleplotsubplot]
   {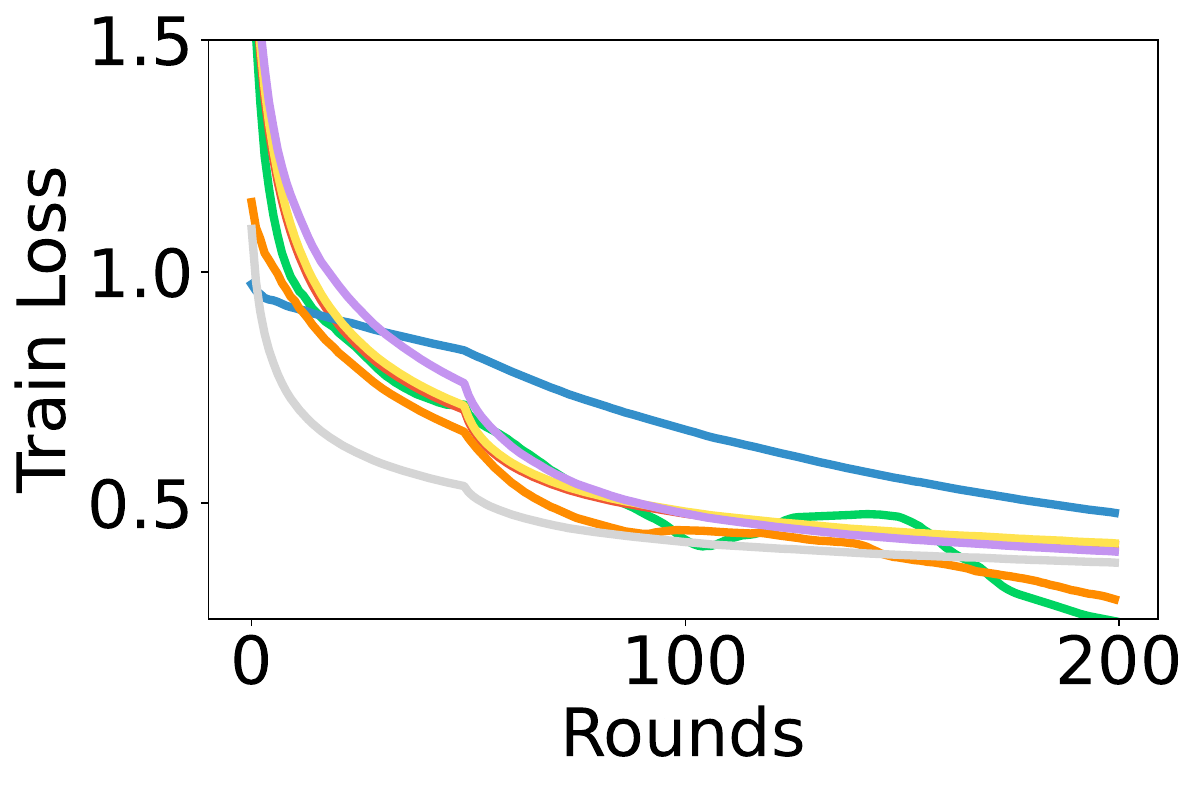}} \\
   High-pers  & Medium-pers & No-pers
   \end{tabular}
   \caption{Test accuracy (top row) and training loss (bottom row) over 200 {communication} rounds of PFLEGO, FedAvg and FedPer on Fashion-MNIST. Each of the three columns corresponds to a certain degree of personalization. Settings: 100 clients, 200 rounds 50 inner steps, r=20.
    }
  \label{fig:main_fashion_mnist_plot}
  \end{center}
\end{figure*}

\begin{figure*}[ht]
    \begin{center}
    \begin{tabular}{ccc}
     \multicolumn{3}{c}{{\includegraphics[scale=\myscaleplottitle]
    {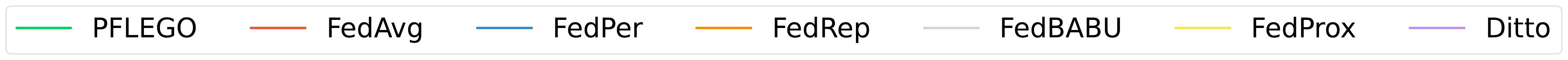}}}
    \end{tabular}\\
    \begin{tabular}{ccc}
    {\includegraphics[scale=\myscaleplotsubplot]
    {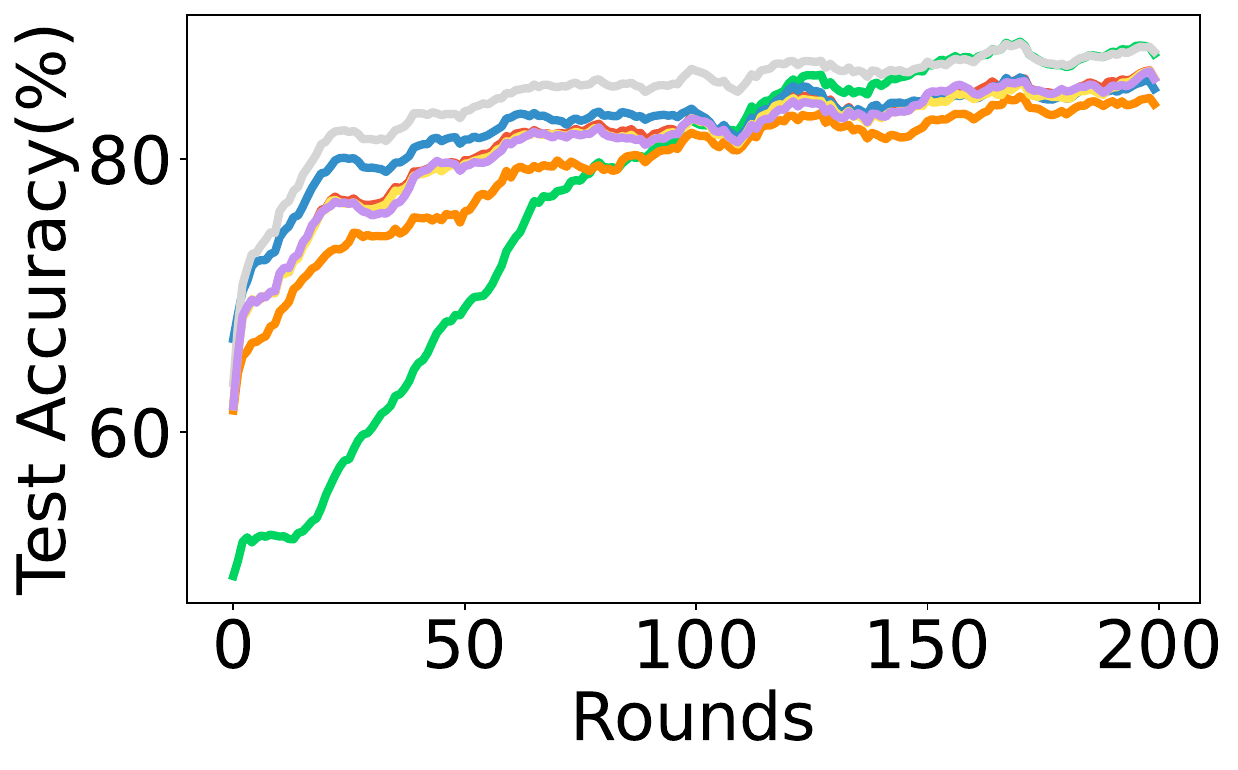}} &
    {\includegraphics[scale=\myscaleplotsubplot]
    {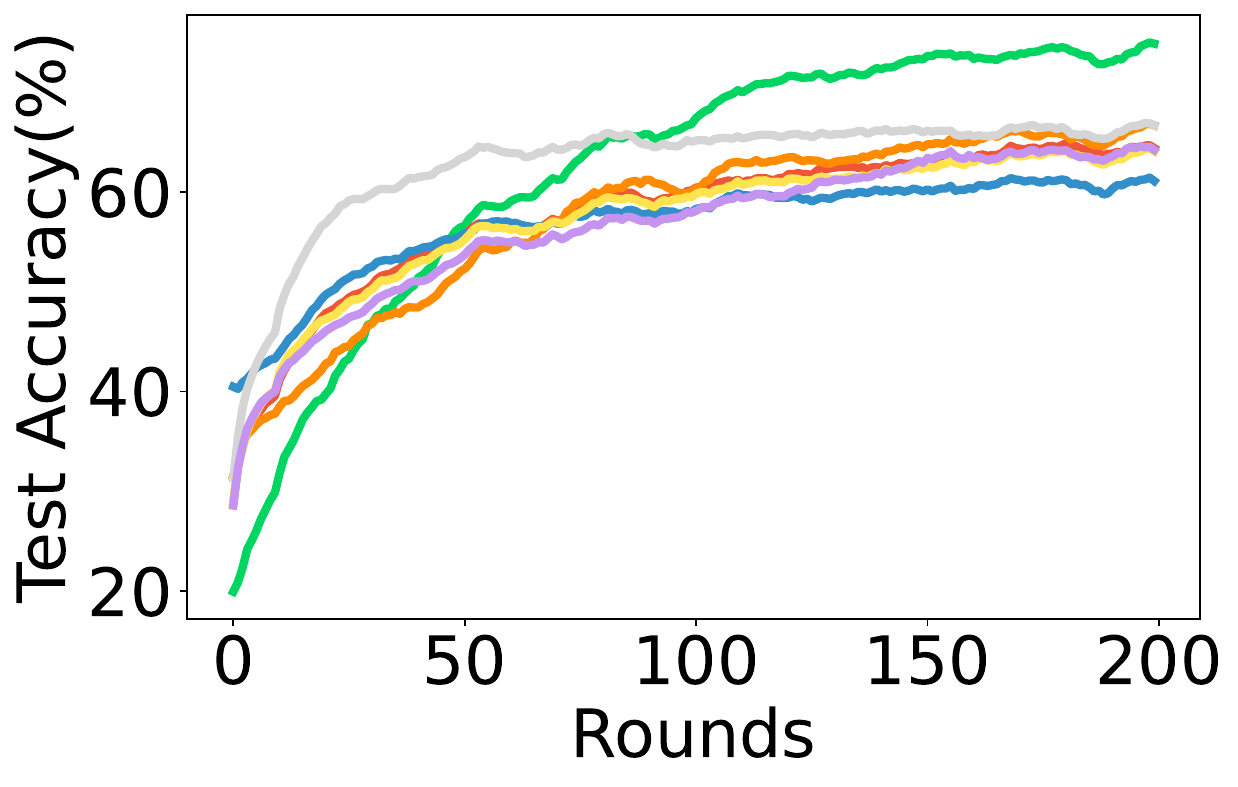}} &
    {\includegraphics[scale=\myscaleplotsubplot]
    {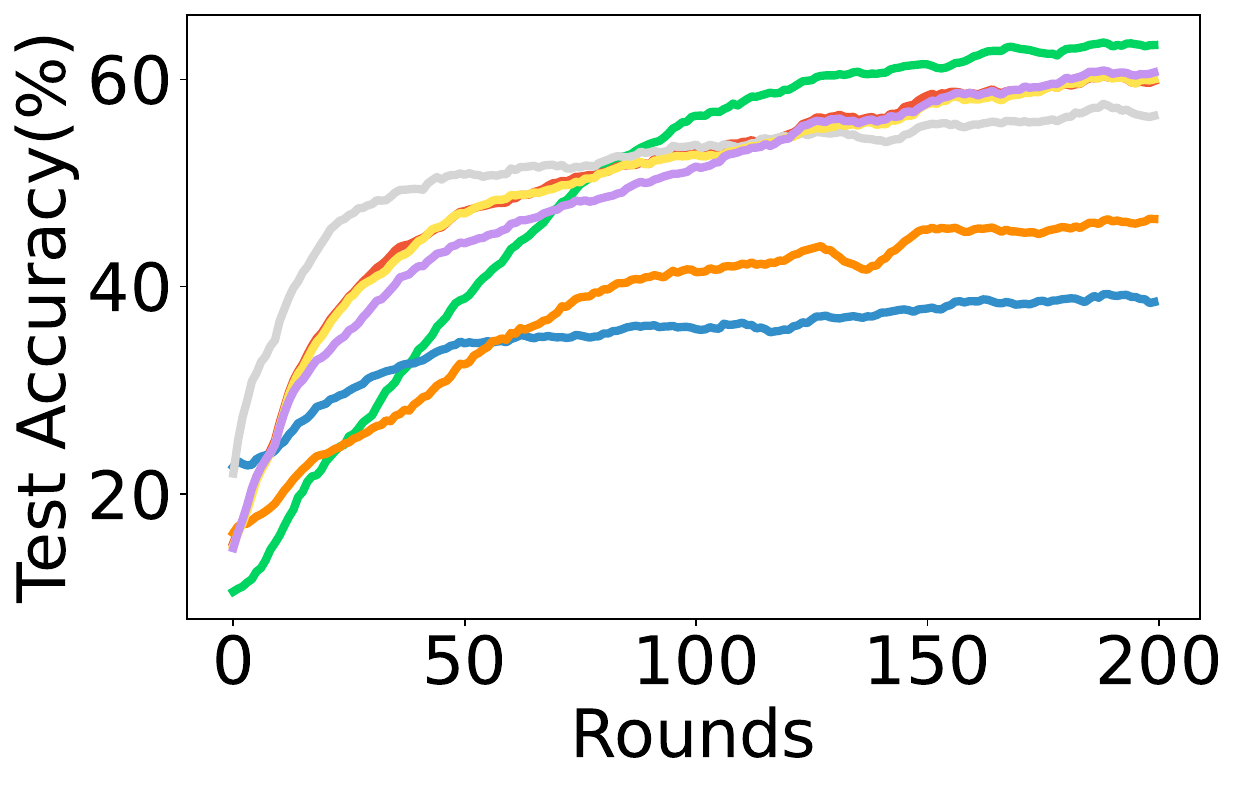}} \\
    {\includegraphics[scale=\myscaleplotsubplot]
    {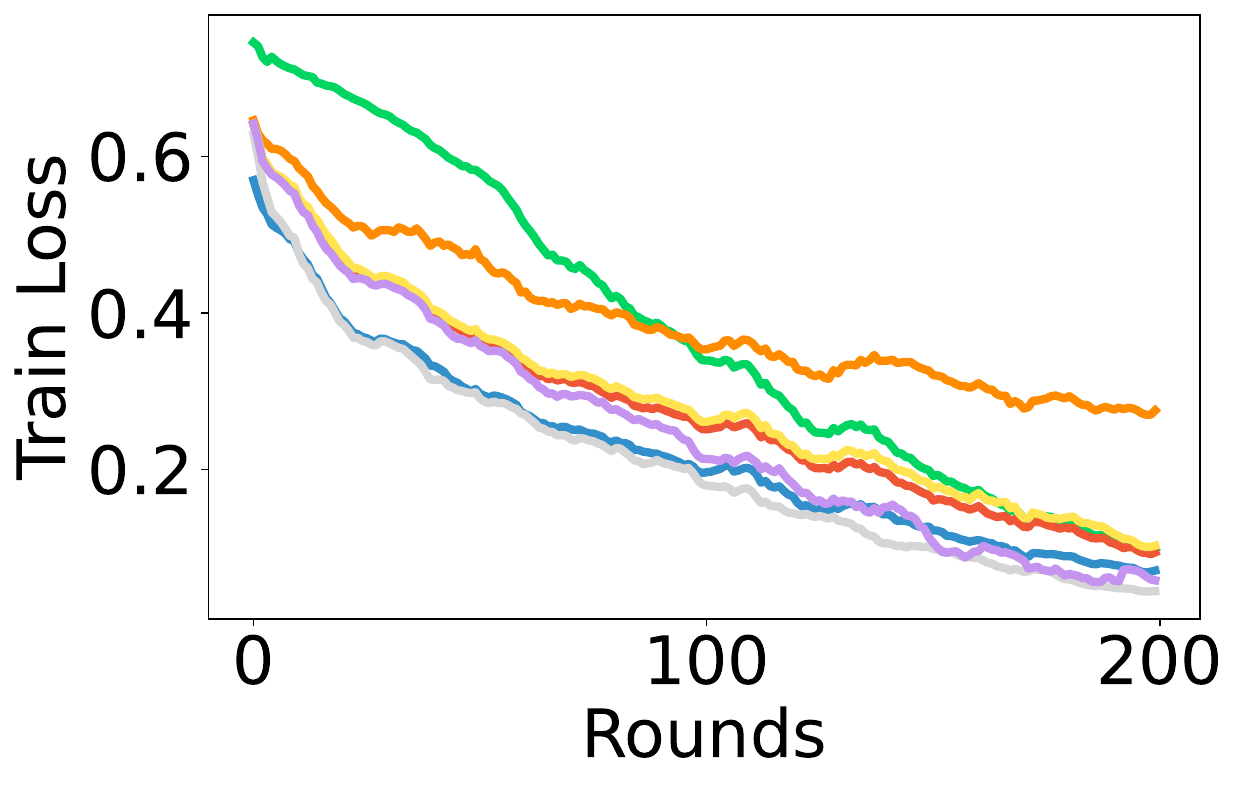}} &
    {\includegraphics[scale=\myscaleplotsubplot]
    {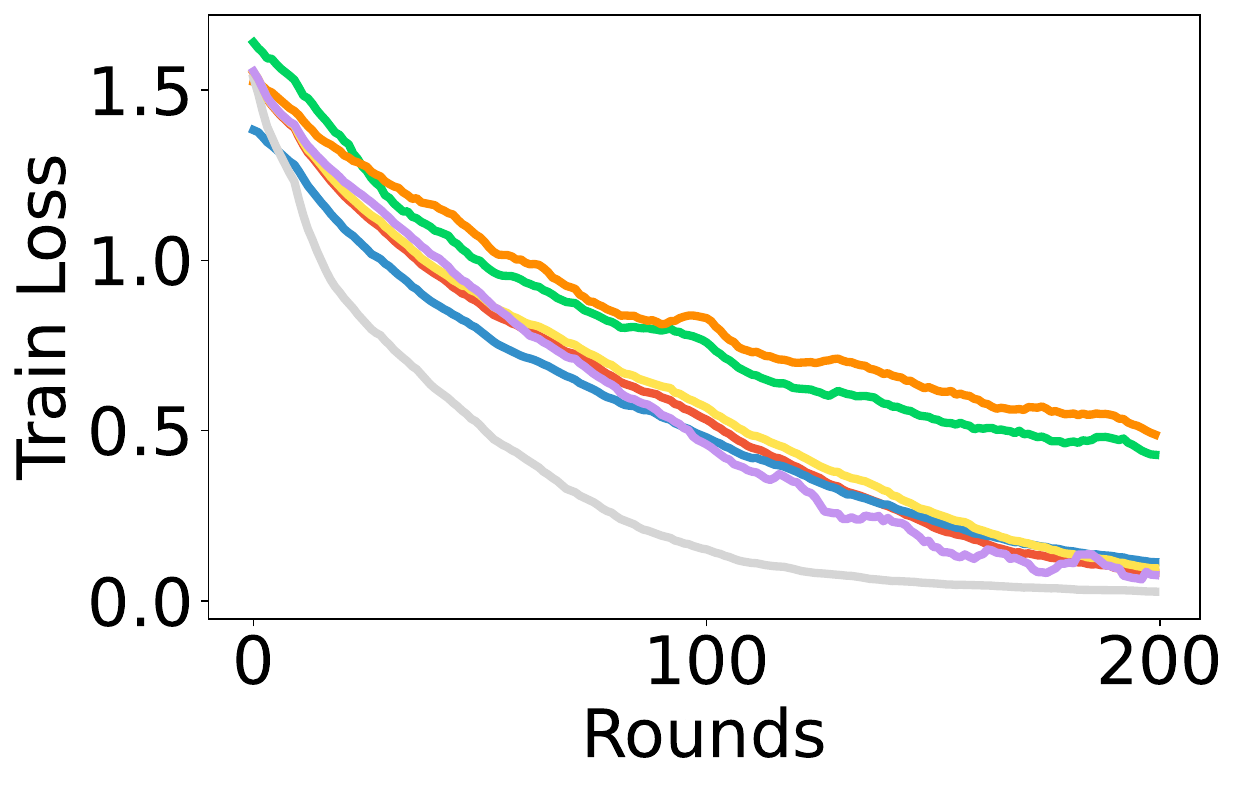}} &
    {\includegraphics[scale=\myscaleplotsubplot]
    {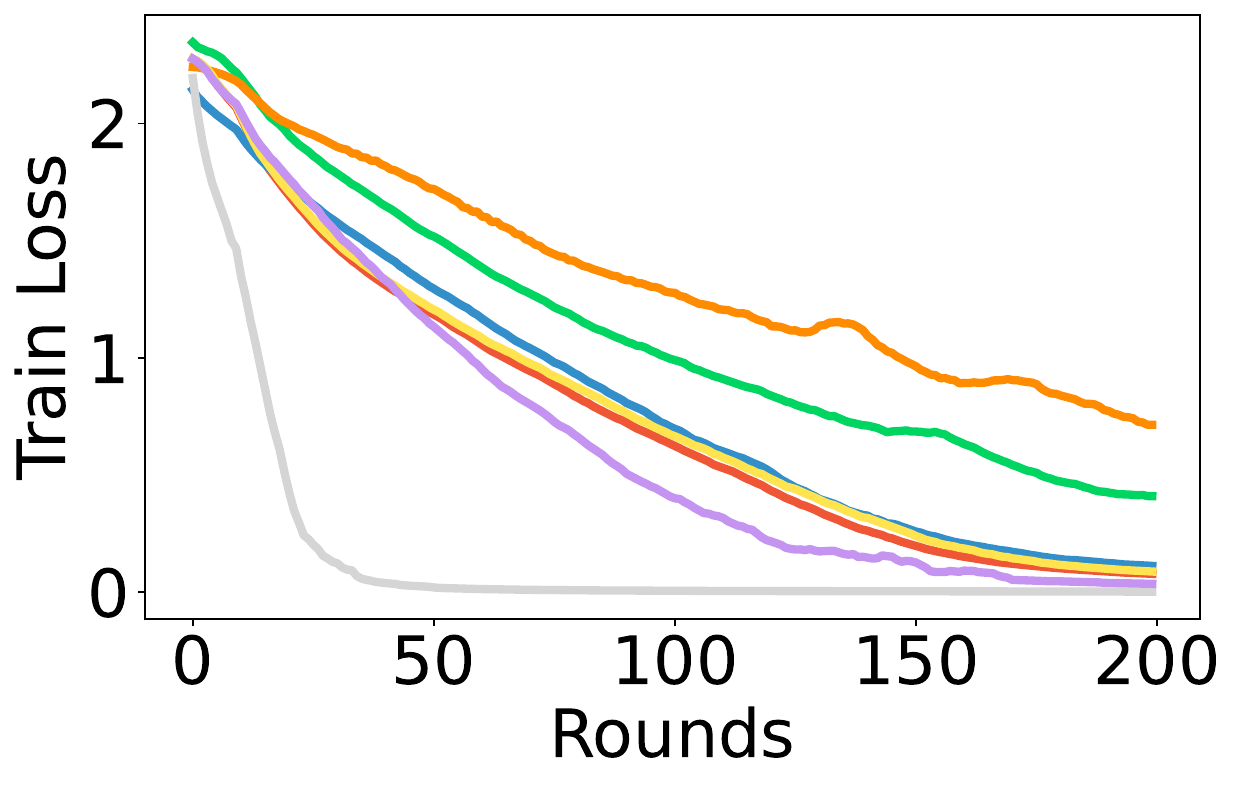}} \\
    High-pers  & Medium-pers & No-pers
    \end{tabular}
   \caption{Test accuracy (top row) and training loss (bottom row) over 200 rounds of PFLEGO, FedAvg and FedPer with CNN model on CIFAR-10 dataset. Each of the three columns corresponds to a certain degree of personalization. Settings: 100 clients, 200 {communication} rounds 50 inner steps, r=20.
    }
  \label{fig:main_cifar10_plot}
    \end{center}
\end{figure*}  

\begin{figure*}[!htb]
    \begin{center}
    \begin{tabular}{ccc}
     \multicolumn{3}{c}{{\includegraphics[scale=\myscaleplottitle]
    {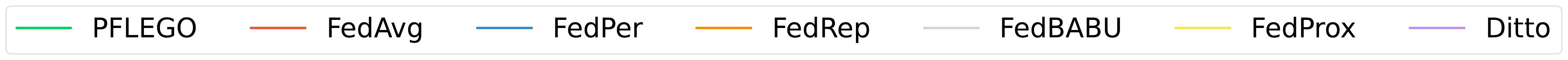}}}
    \end{tabular}\\
    \begin{tabular}{ccc}
    {\includegraphics[scale=\myscaleplotsubplot]
    {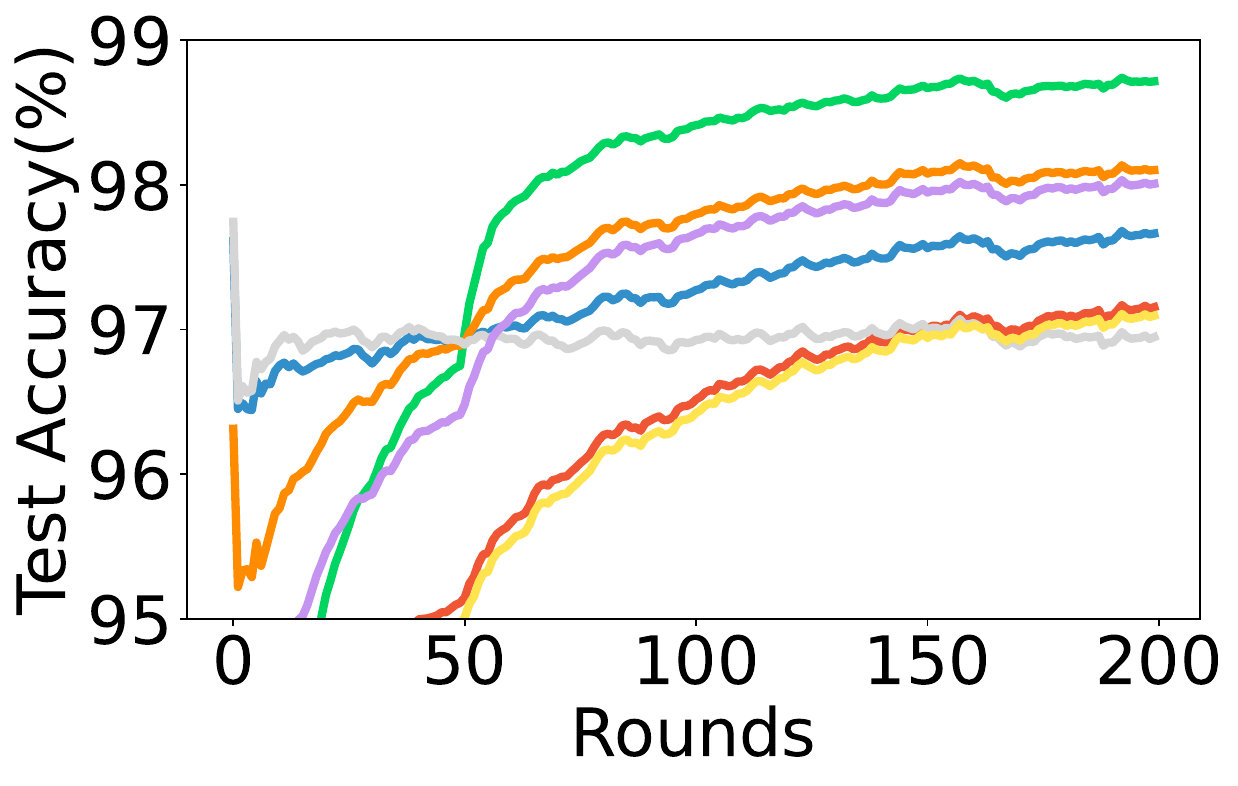}} &
    {\includegraphics[scale=\myscaleplotsubplot]
    {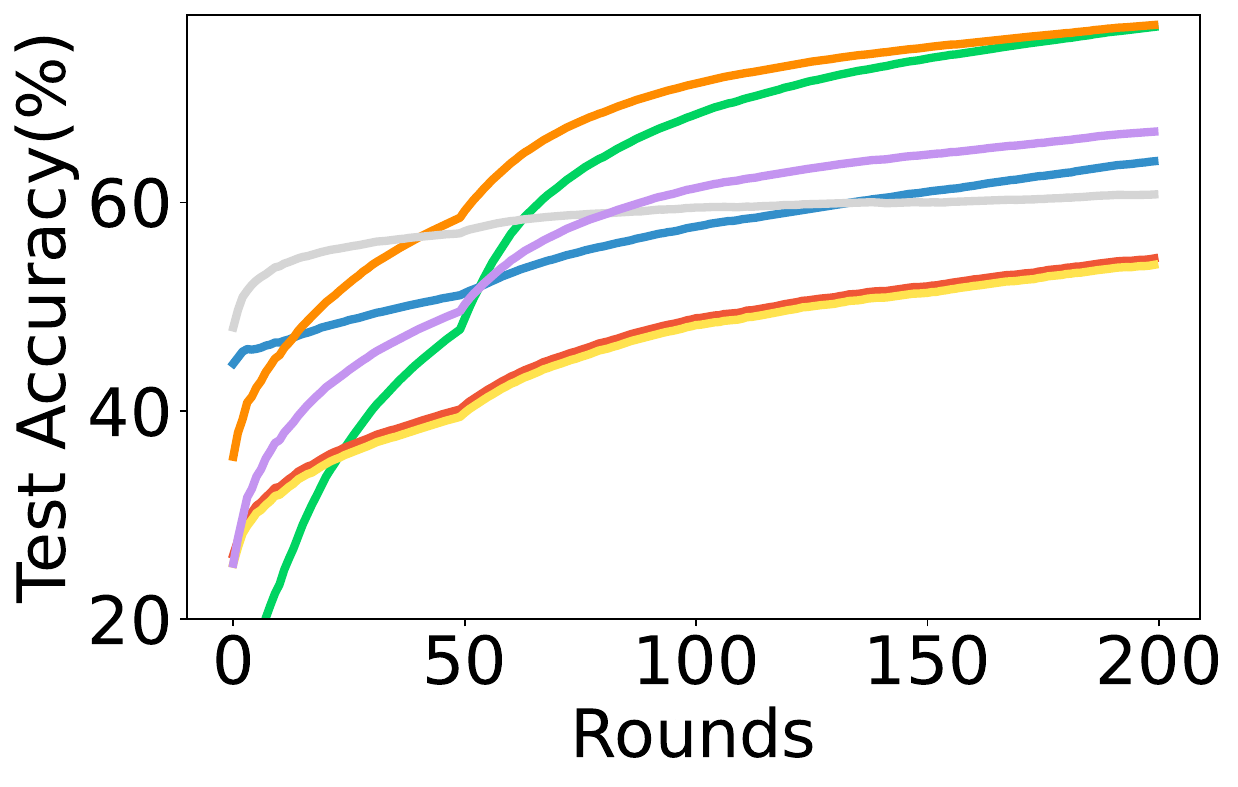}} &
    {\includegraphics[scale=\myscaleplotsubplot]
    {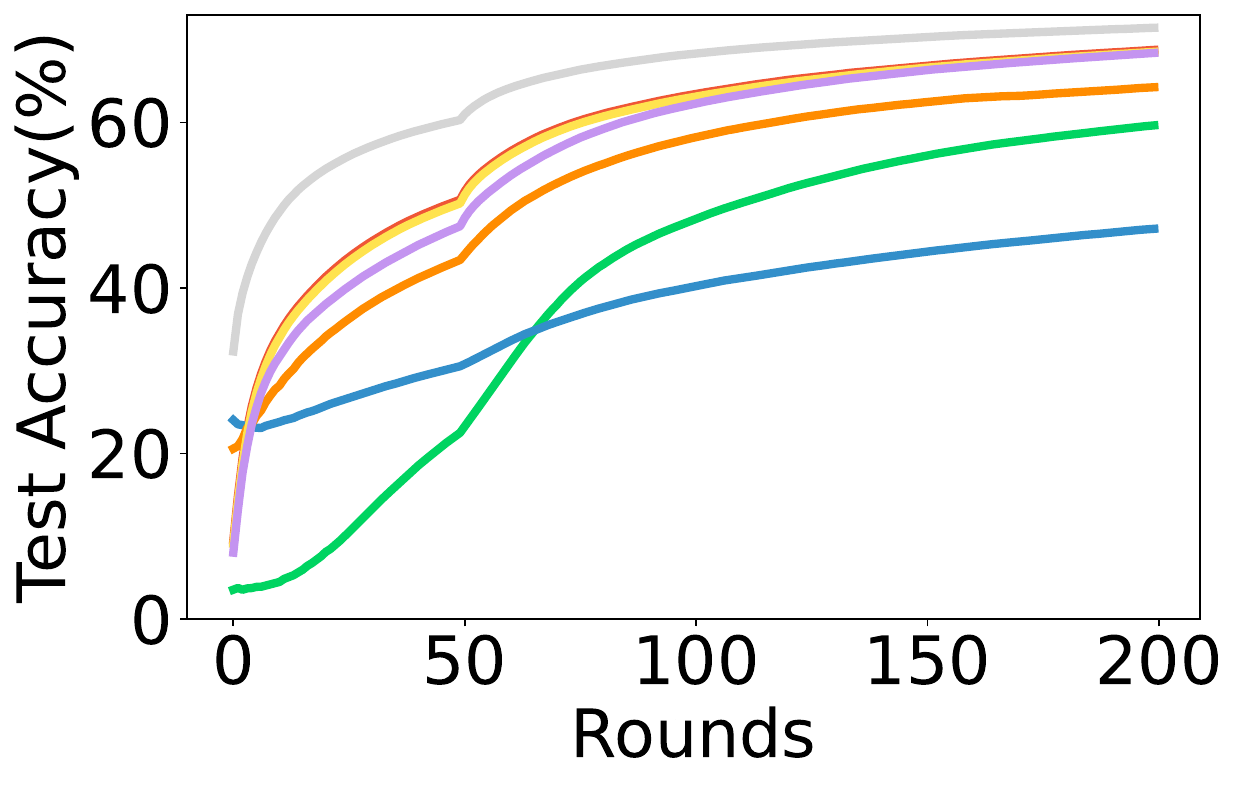}} \\
    {\includegraphics[scale=\myscaleplotsubplot]
    {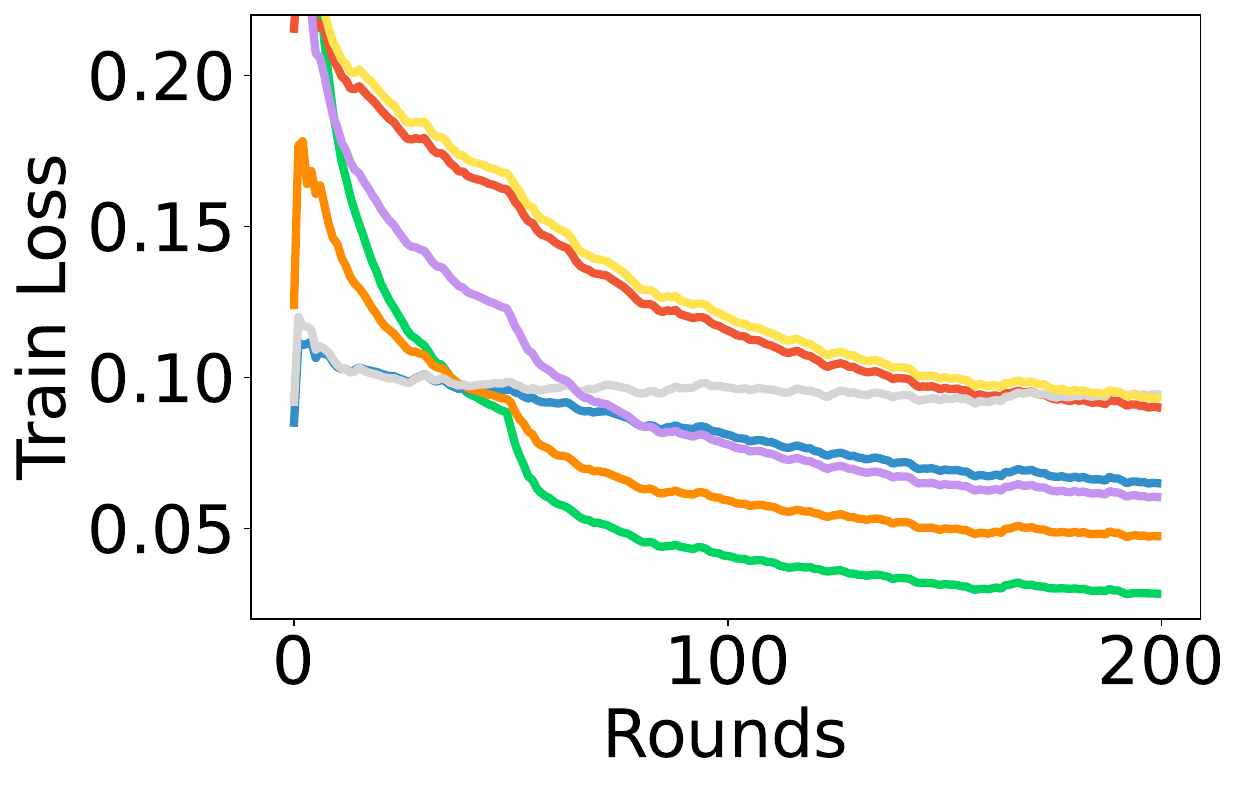}} &
    {\includegraphics[scale=\myscaleplotsubplot]
    {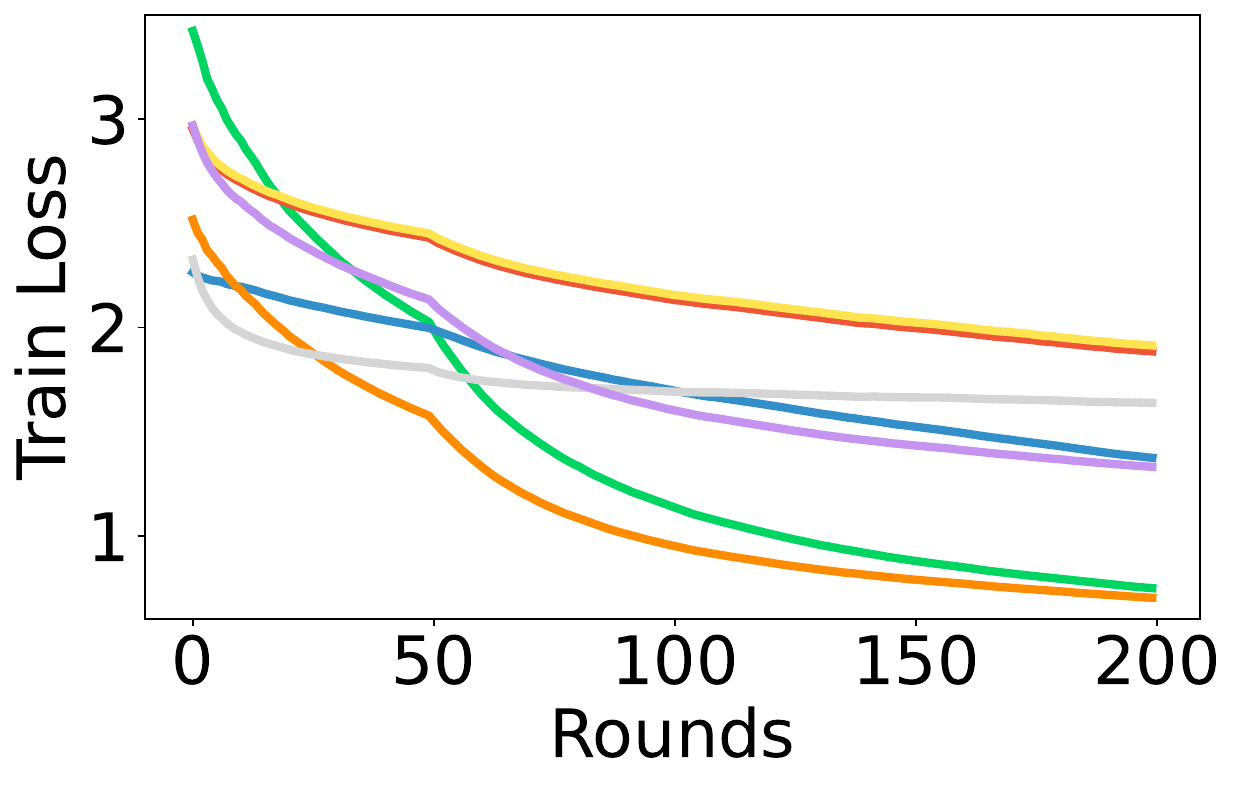}} &
    {\includegraphics[scale=\myscaleplotsubplot]
    {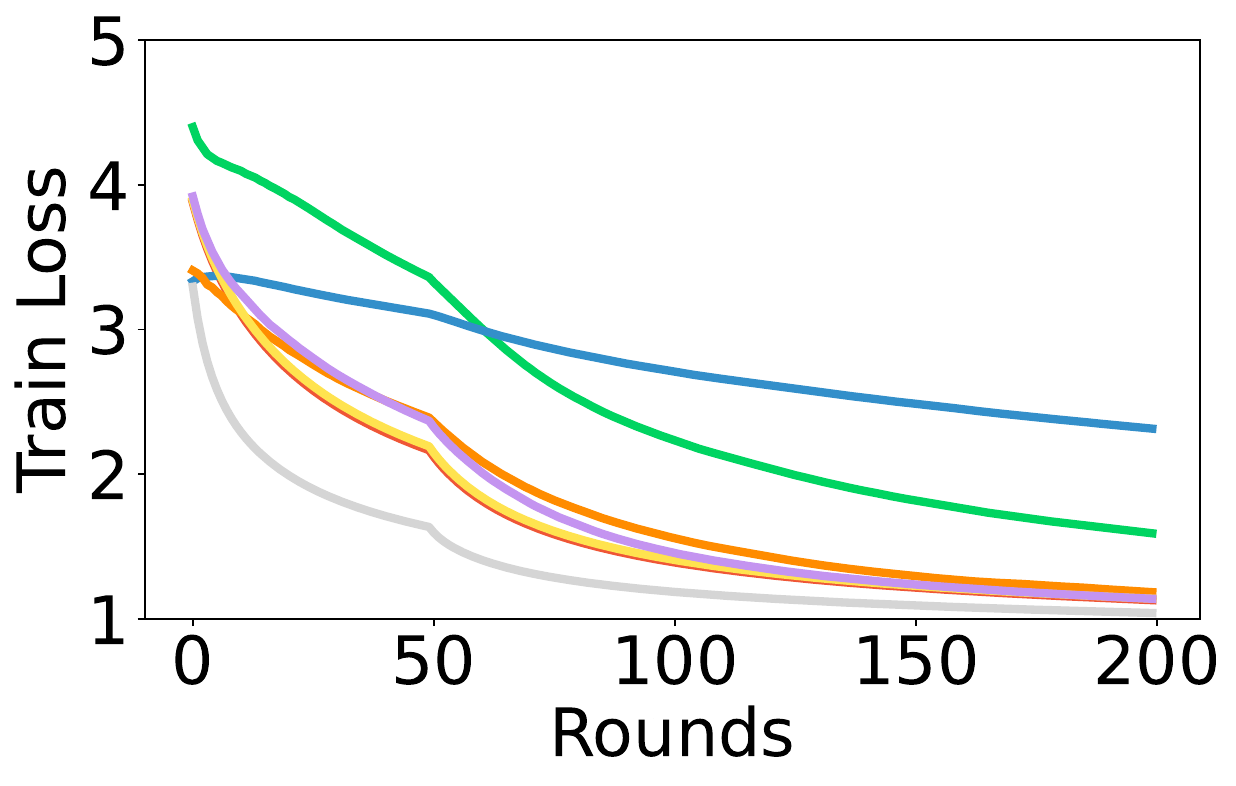}} \\
    High-pers  & Medium-pers & No-pers
    \end{tabular}
   \caption{Test accuracy (top row) and training loss (bottom row) over 200 rounds of PFLEGO, FedAvg and FedPer with MLP model on EMNIST dataset. Each of the three columns corresponds to a certain degree of personalization. Settings: 100 clients, 200 {communication} rounds 50 inner steps, r=20.
    }
  \label{fig:main_emnist_plot}
    \end{center}
\end{figure*}

\begin{figure*}[ht]
\begin{center}
\begin{tabular}{ccc}
 \multicolumn{3}{c}{{\includegraphics[scale=0.26]
{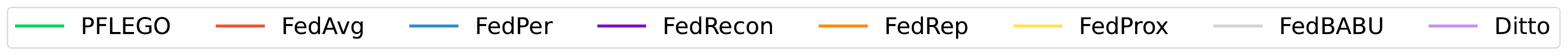}}}
\end{tabular}\\
\begin{tabular}{ccc}
{\includegraphics[scale=0.207]
{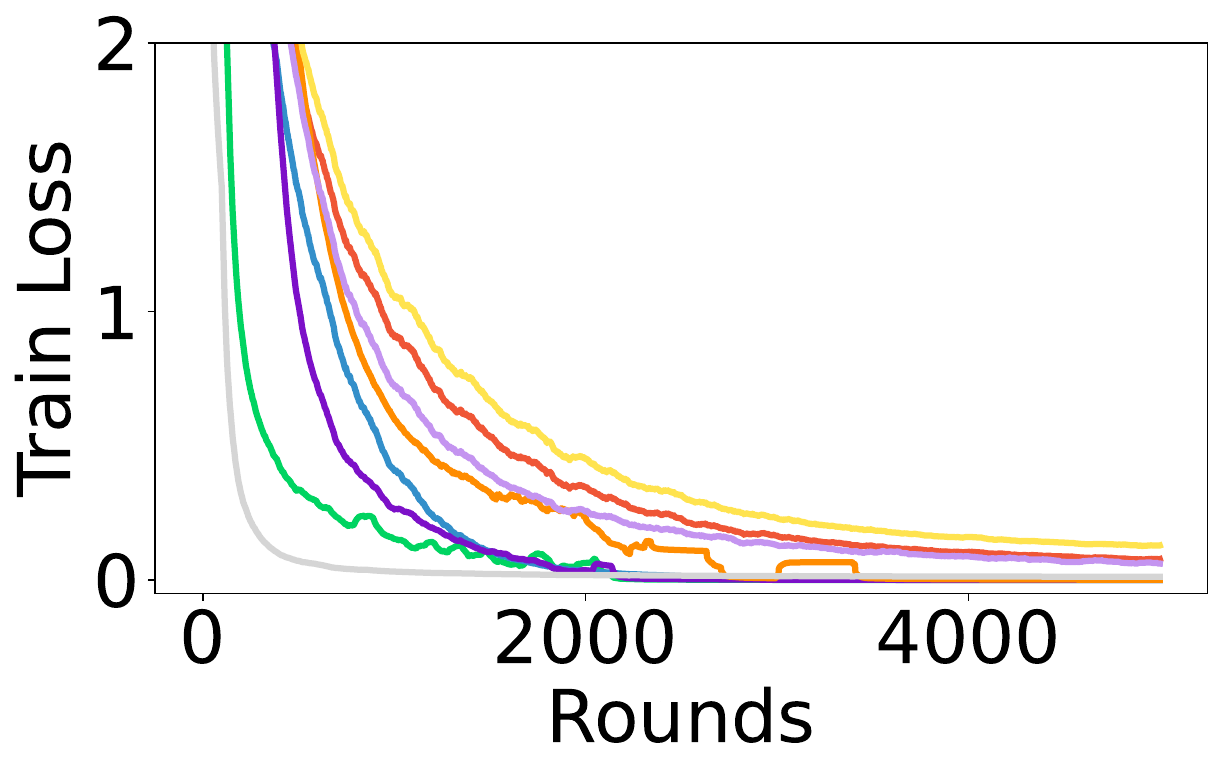}} &
{\includegraphics[scale=0.207]
{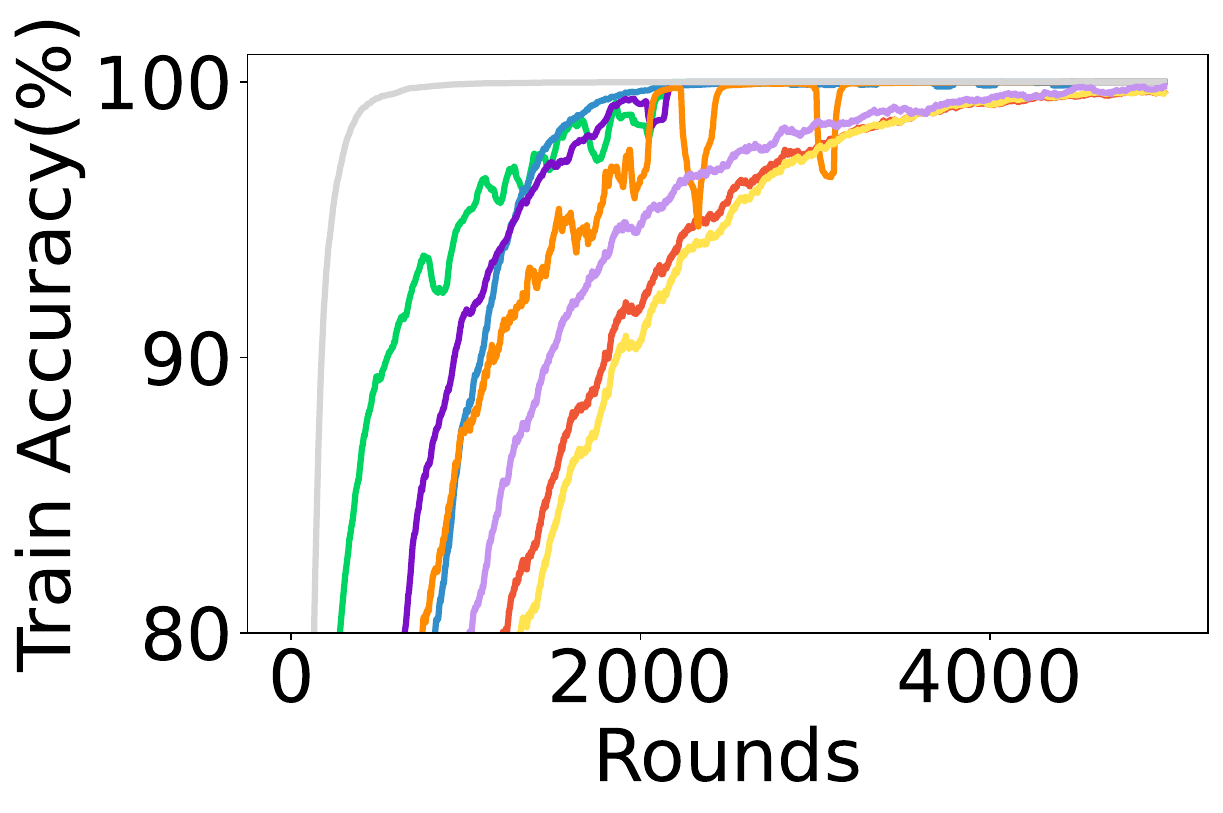}} &
{\includegraphics[scale=0.207]
{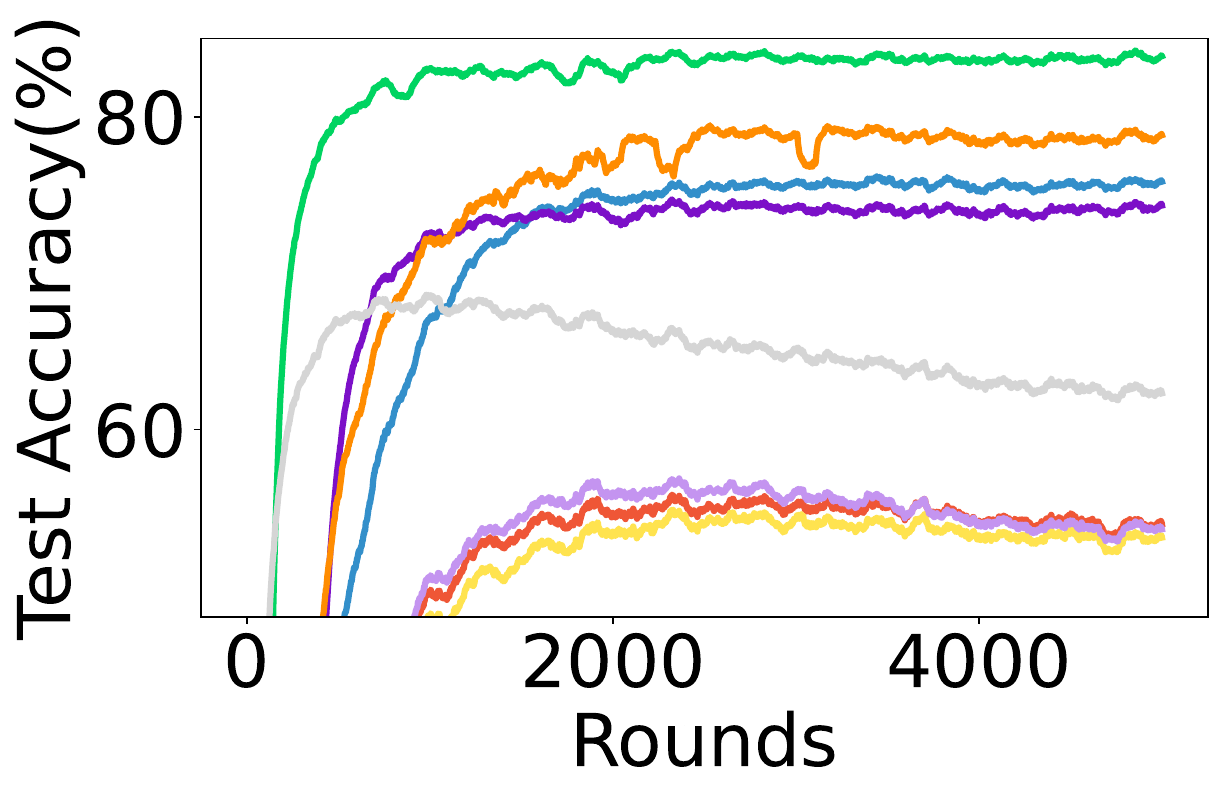}}
\end{tabular}
 \caption{Training loss (left), training accuracy (middle) and test accuracy (right) for the Omniglot dataset. All FL methods were run for $5000$ {communication} rounds, with $50$ inner client steps and $r=20\%$ client participation per round.
 }
  \label{fig:main_omniglot_plot}
  \end{center}
\end{figure*}  

\begin{figure*}[ht]
    \begin{center}
    \begin{tabular}{ccc}
    \multicolumn{3}{c}{{\includegraphics[scale=0.27]
    {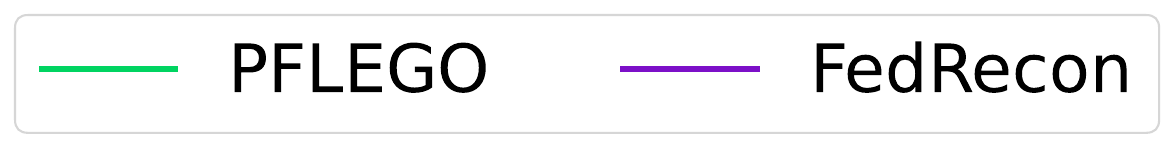}}}
    \end{tabular}\\

    \begin{tabular}{ccc}
{\includegraphics[scale=0.20]
{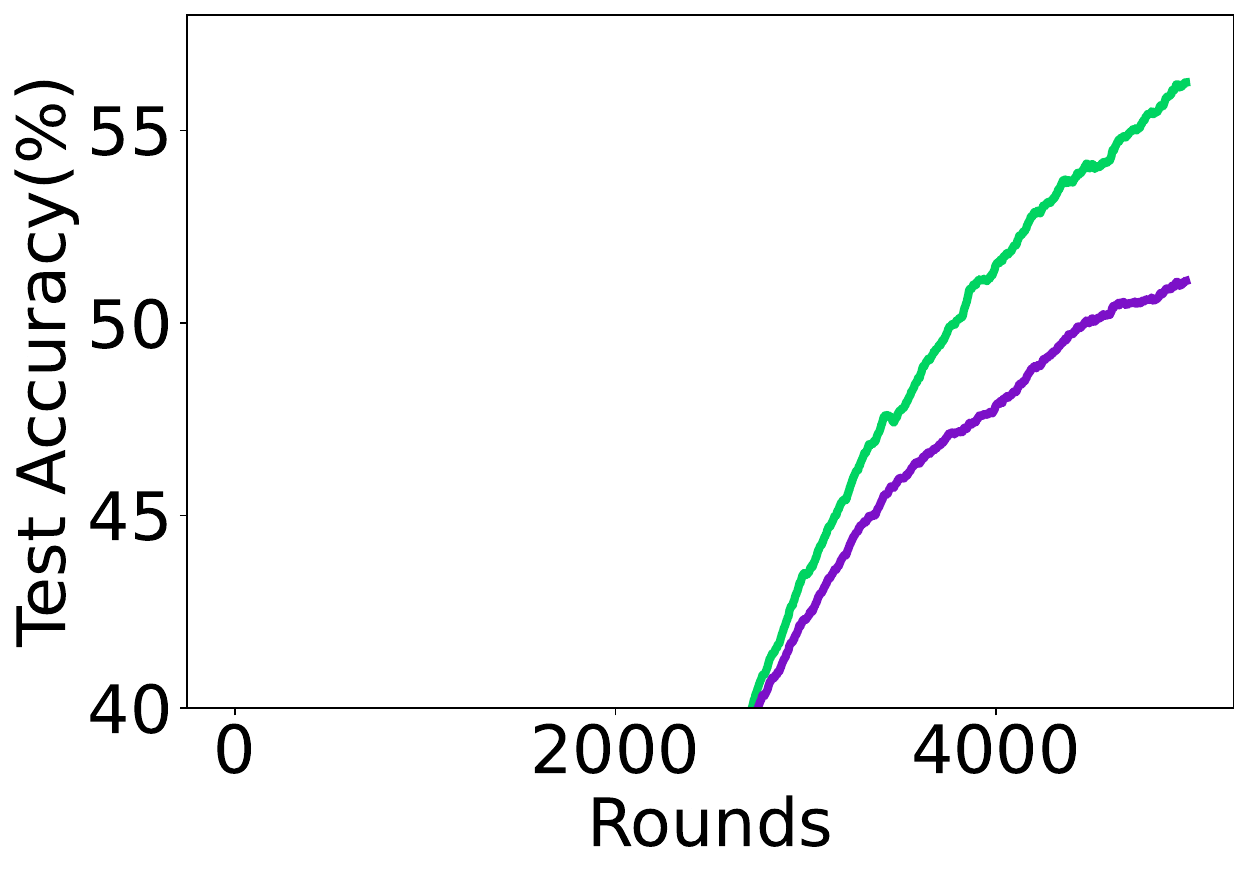}} &
{\includegraphics[scale=0.20]
{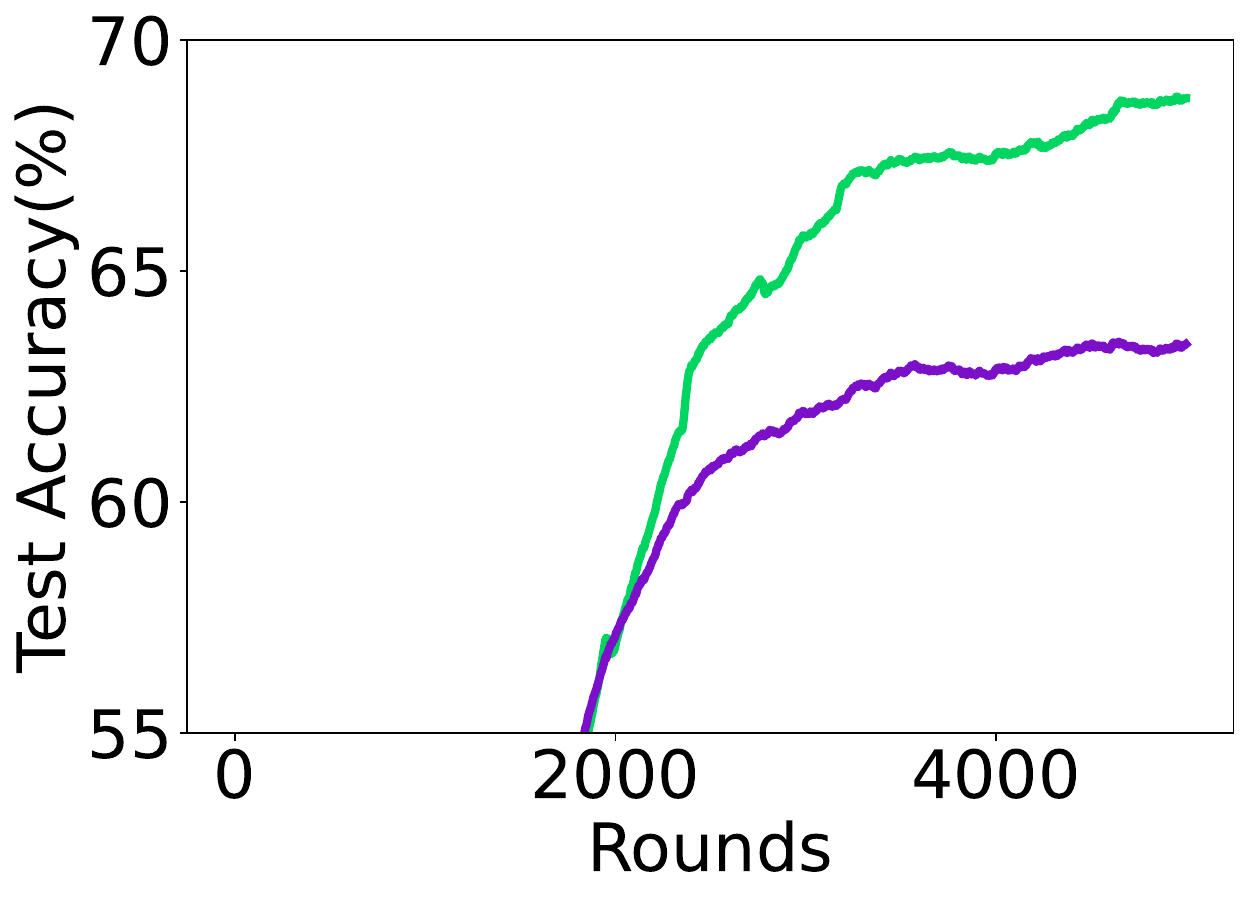}} &
{\includegraphics[scale=0.20]
{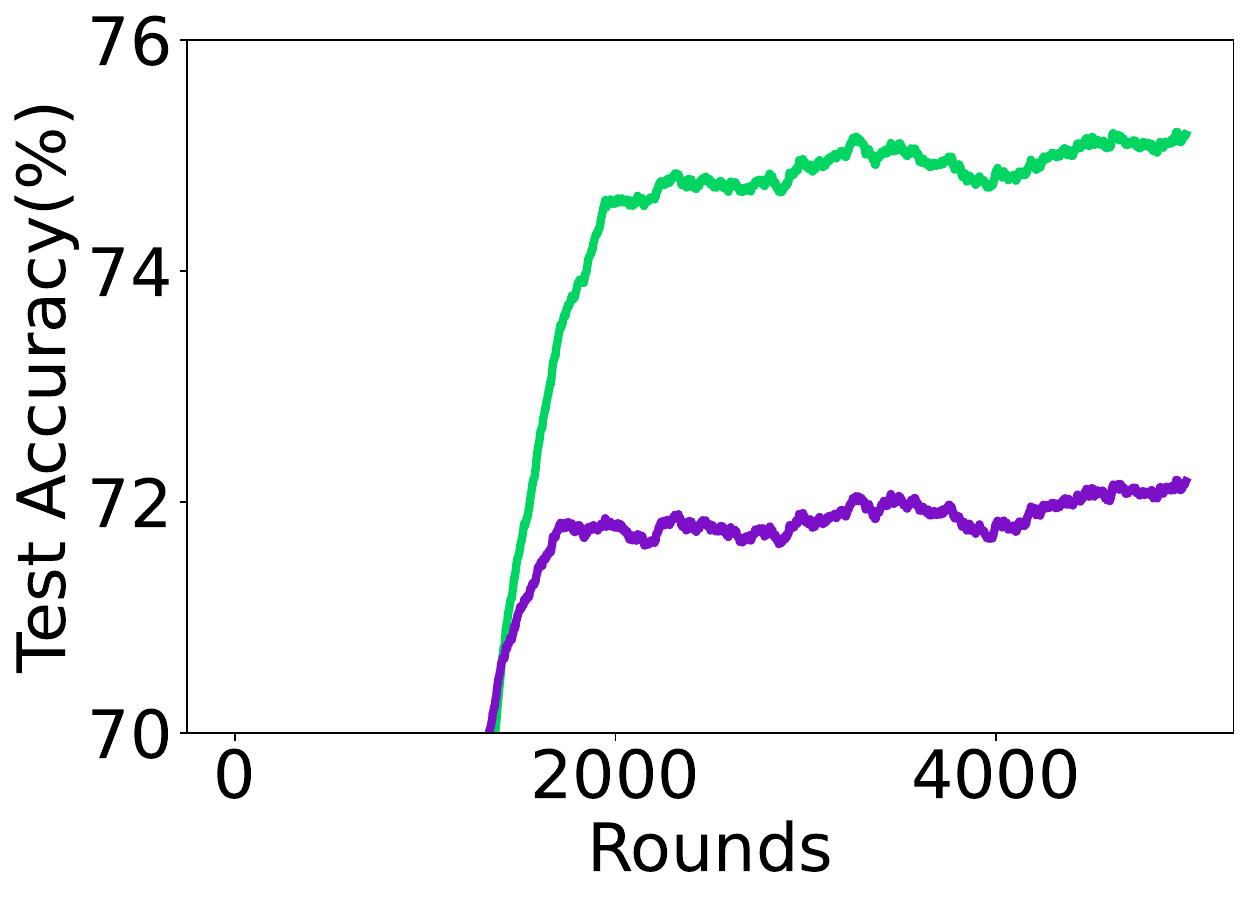}}\\

(a) $\tau=5$ & (b) $\tau=10$ & (c) $\tau=25$\\

{\includegraphics[scale=0.20]
{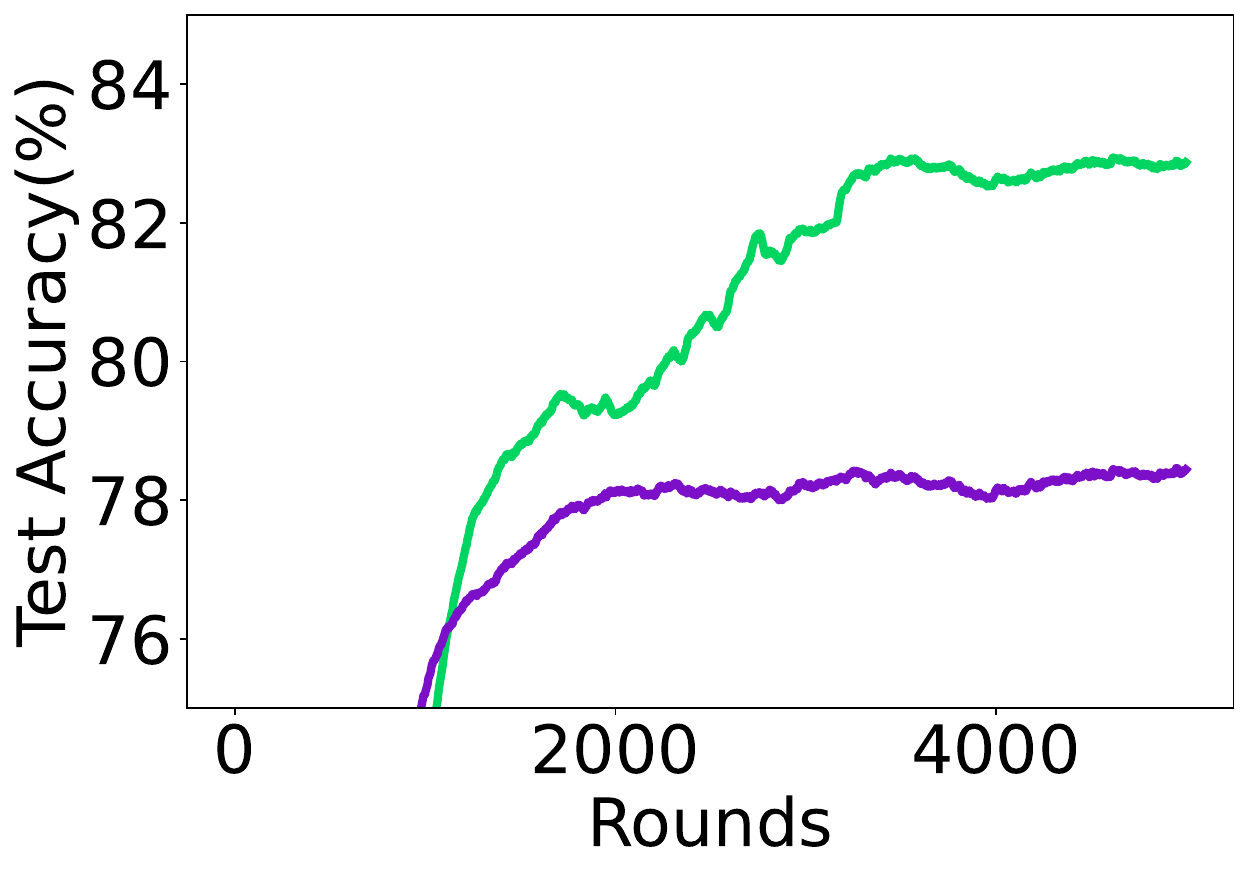}} &

{\includegraphics[scale=0.20]
{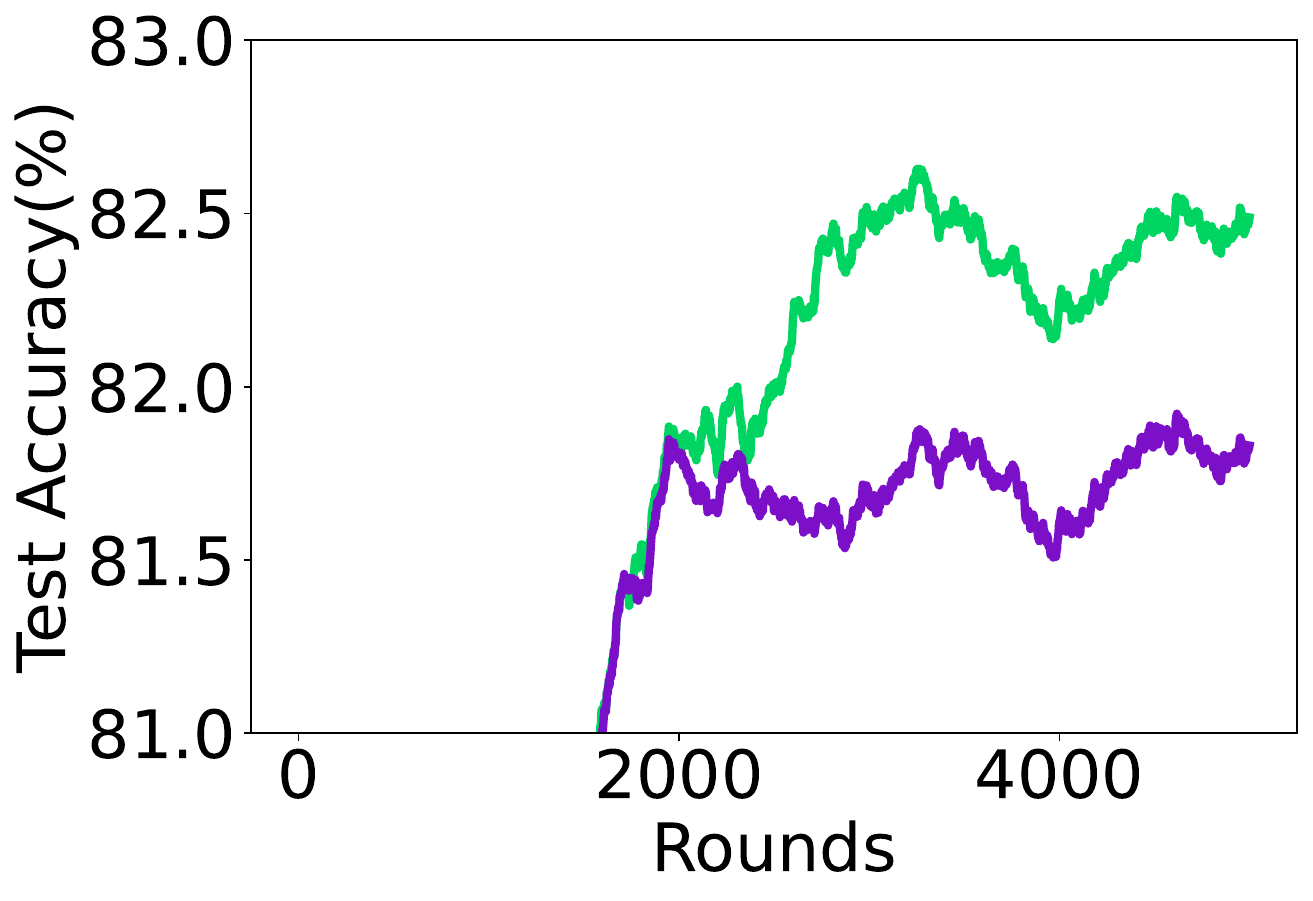}}&
{\includegraphics[scale=0.20]
{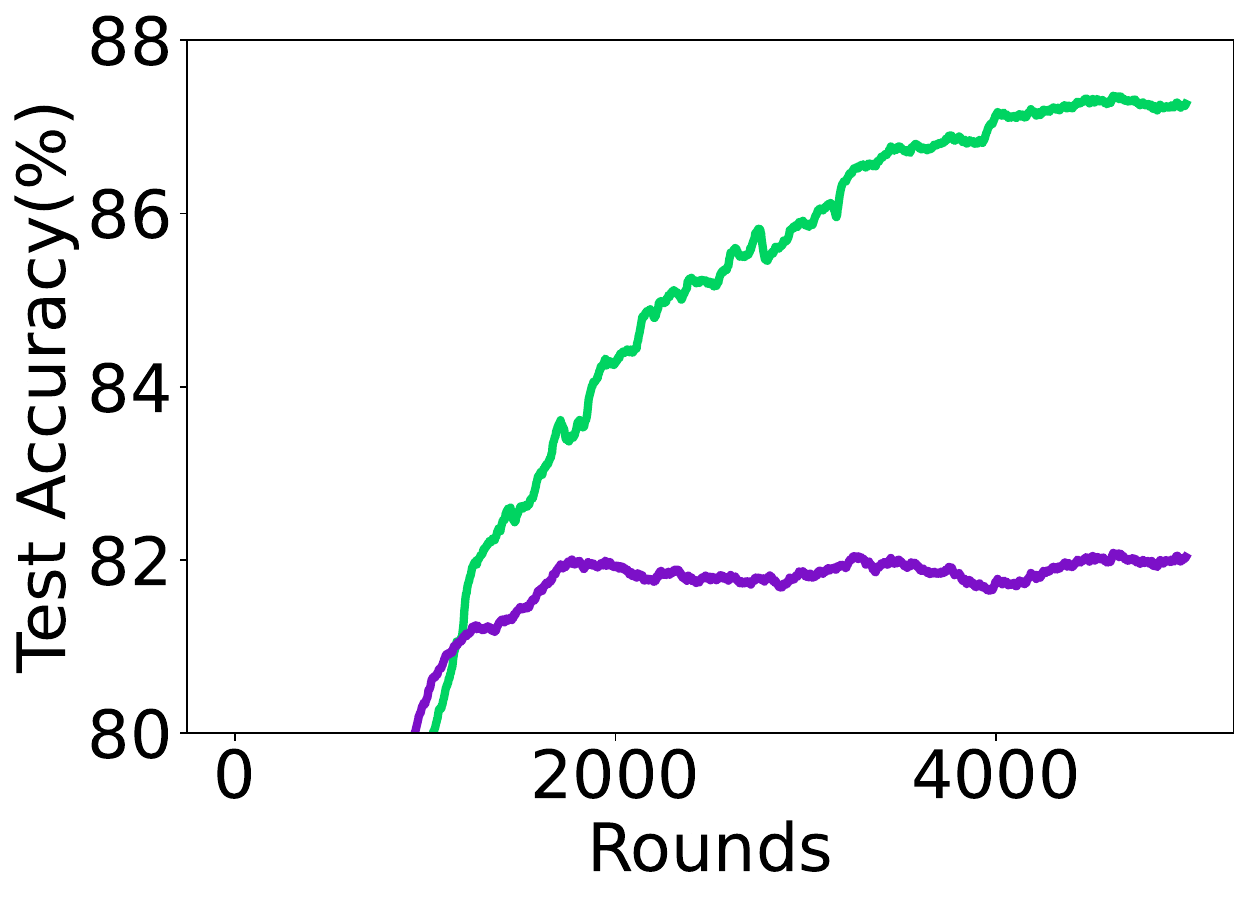}}\\
(d) $\tau=50$ & (e) $\tau=75$ & (f) $\tau=100$ \\

\end{tabular}
    \caption{PFLEGO vs the block coordinate descent FedRecon algorithm. Settings: 50 clients, 5000 {communication} rounds, r $\sim$ [4,28], $\rho = 0.001, \beta = 0.007$ and Inner Steps $\tau$: 5 (a), 10 (b) 25 (c), 50 (d), 75 (e), and 100 (f) for the Omniglot dataset. }
    \label{fig:fedrec_vs_pflego}
    \end{center}
\end{figure*}


We also perform an additional comparison between our method PFLEGO and the block coordinate descent algorithm FedRecon \cite{singhal2021federated} at the Omniglot dataset. At the client side we test various values on the number of the inner steps $\tau$. At the server side, at the beginning of each round the server randomly selects at least 4 participants up to 28 participants. We observe that in each case PFLEGO has higher test accuracy than FedRecon regardless of the number of inner steps $\tau$, {validating that performing multiple local updates with an exact global SGD step can achieve both strong accuracy and efficient communication}. In Figure \ref{fig:main_omniglot_plot}, we observe that FedBabu\cite{fedbabu} minimizes the training loss faster than PFLEGO, but PFLEGO achieves higher accuracy. That is because the client-specific parameters in FedBabu have an orthogonal weight initialization at the server side and are never trained (Figure \ref{fig:fedrec_vs_pflego}). 

\begin{figure*}[!htbp]
    \begin{center}
    \begin{tabular}{cc}
    \multicolumn{2}{c}{{\includegraphics[scale=0.20]
    {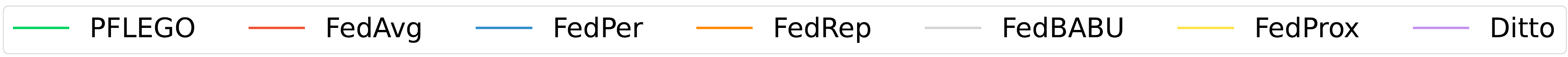}}}
    \end{tabular}\\
    
    \begin{tabular}{cc}
    {\includegraphics[scale=0.24]
    {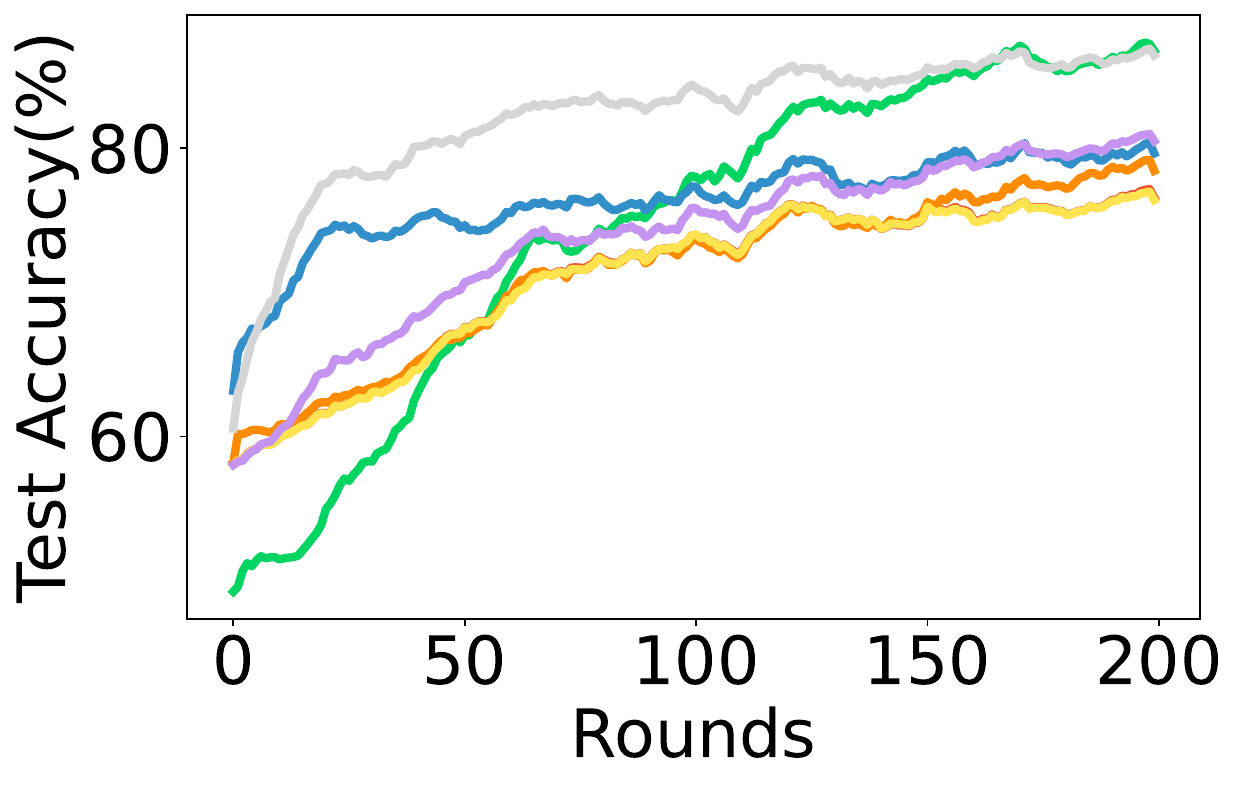}} &
    {\includegraphics[scale=0.24]
    {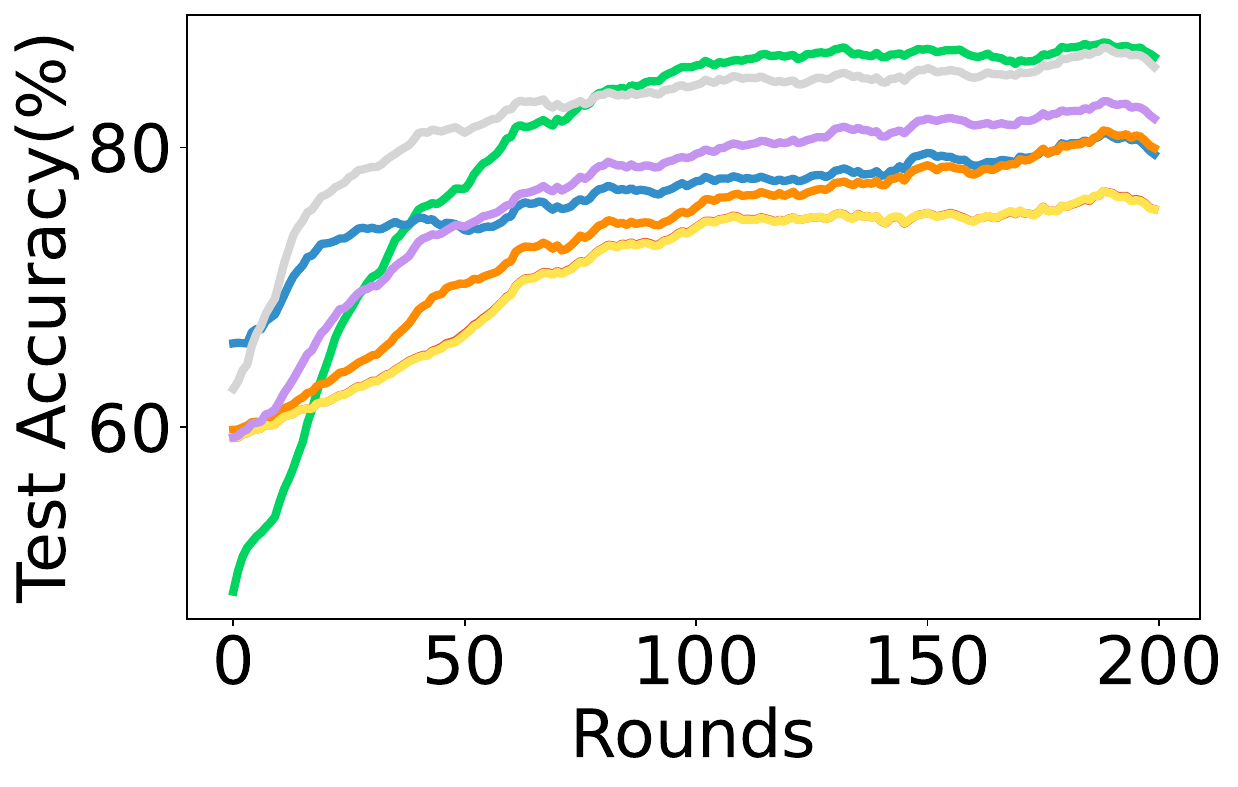}} \\
    $r = 20\%$ & $r = 40\%$ \\ 
    {\includegraphics[scale=0.24]
    {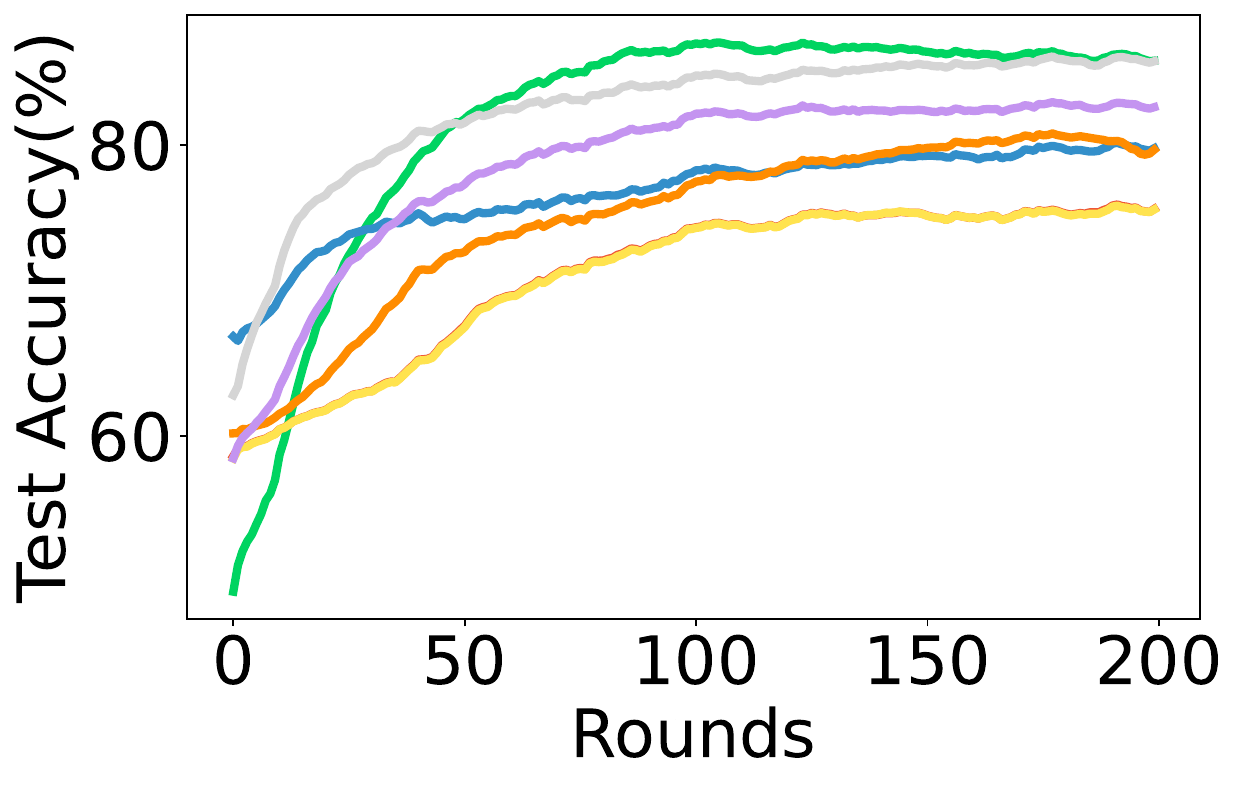}} &
    {\includegraphics[scale=0.24]
    {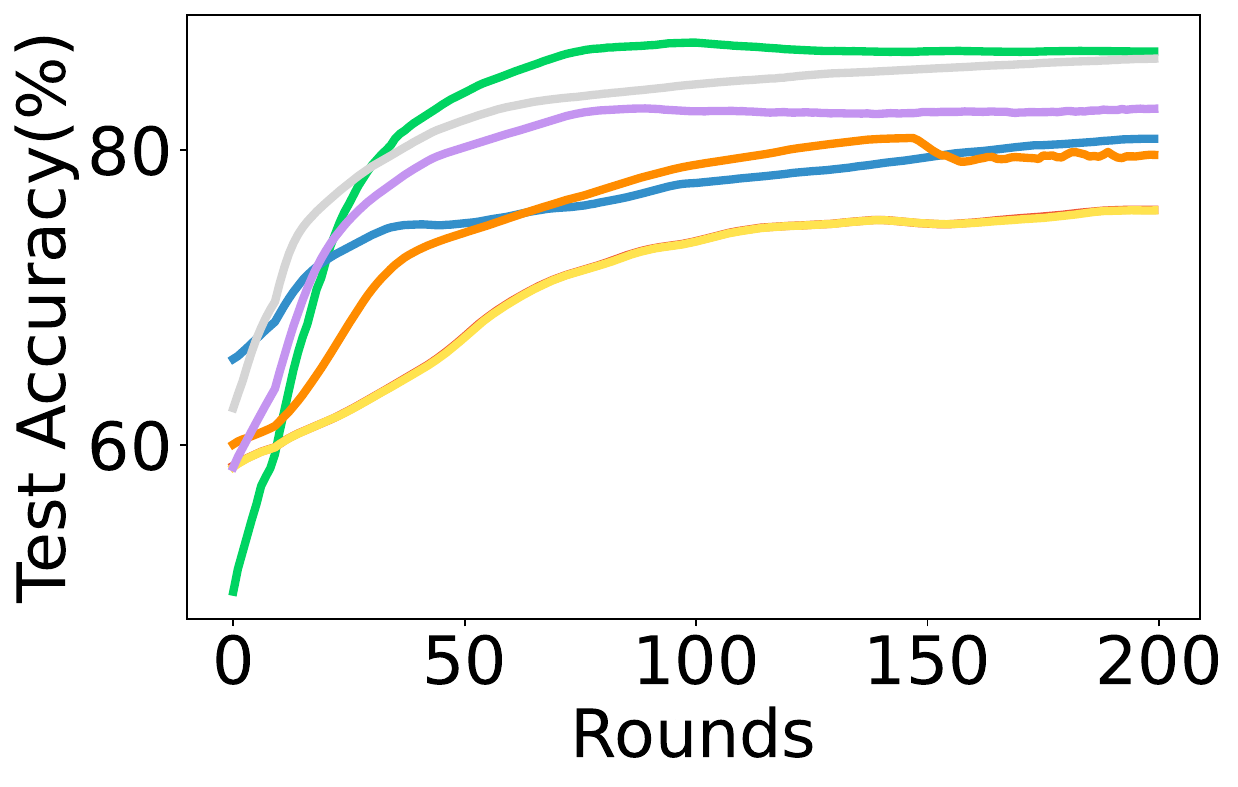}} \\
    $r = 60\%$ & $r = 100\%$ 
    \end{tabular}
   \caption{Test accuracy of the ablation study of client participation using the CIFAR-10 for the High-pers case. PFLEGO is compared against FedAvg and FedPer in terms of the ability to minimize the train loss across rounds. The panel in each column corresponds to a different 
   participation percentage $r$.}
  \label{fig:abl_cifar_plot_test_acc}
    \end{center}
\end{figure*}    

\begin{figure*}[!htbp]
\begin{center}
\begin{tabular}{ccc}
 \multicolumn{3}{c}{{\includegraphics[scale=0.267]
    {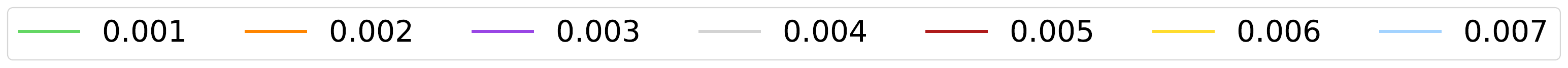}}}
    \end{tabular}\\

    \begin{tabular}{ccc}
{\includegraphics[scale=0.202]
{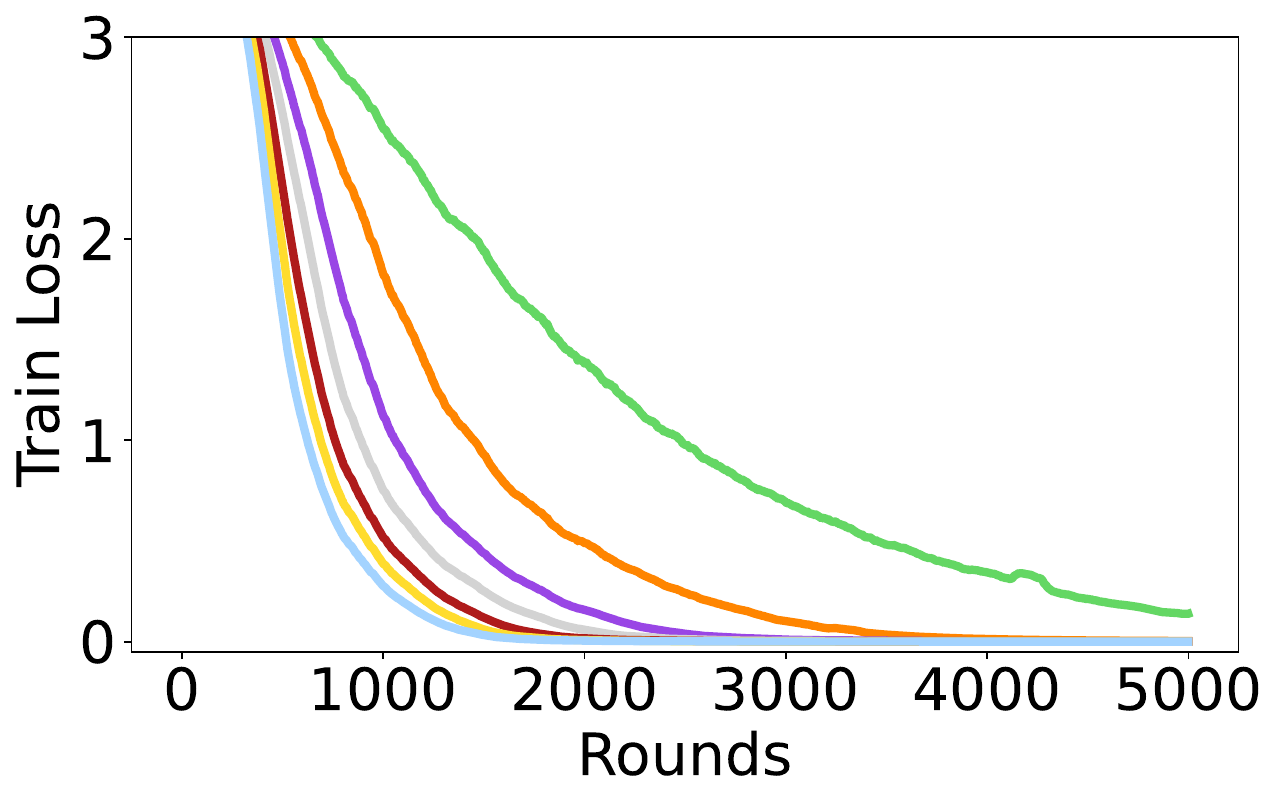}} &
{\includegraphics[scale=0.202]
{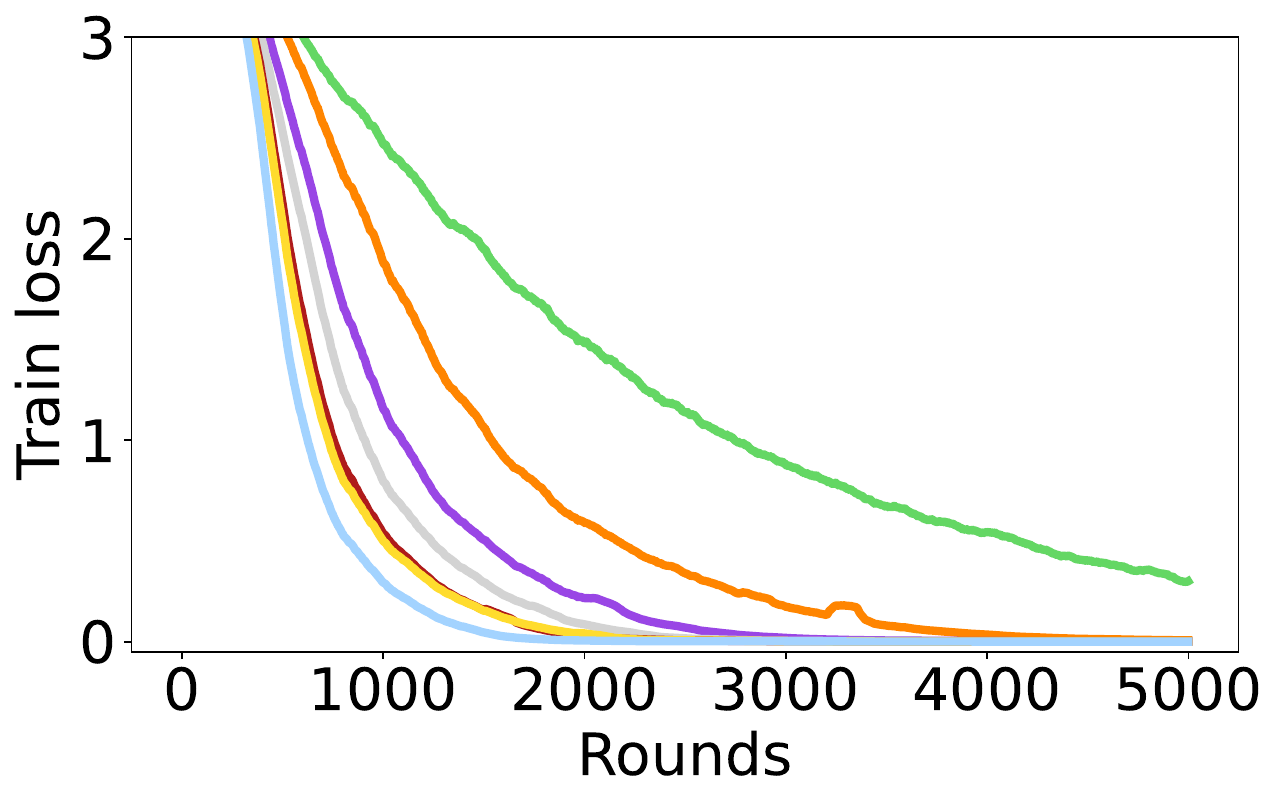}} &
{\includegraphics[scale=0.202]
{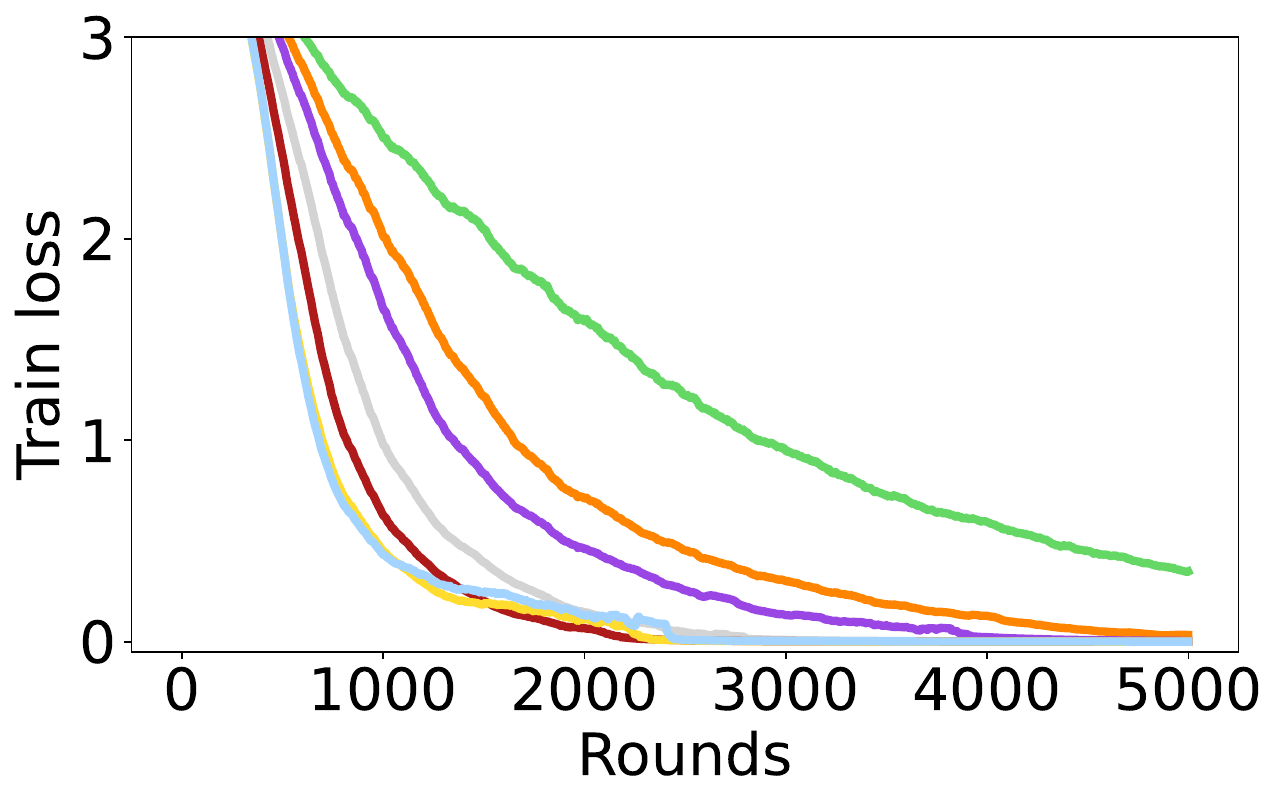}} \\
$\rho = 0.001$ & $\rho=0.002$ & $\rho=0.003$
\end{tabular}
   \caption{Training loss for different value of learning rate $\rho$. For each value of $\rho$ we increase the client learning rate value $\beta$. All experiments are on the Omniglot dataset using PFLEGO with settings: $I=50$ clients, $\tau=50$ inner steps, $T=5000$ {communication} rounds, and $r=20\%$ participating clients per round.
   }
      \label{fig:varying_beta}
      \end{center}
\end{figure*}

\subsection{Ablation Studies}
\label{appendix:effectoflearningrateandparticipationrate}

We carry out experiments 
to investigate the effect of hyperparameters such as the  client participation rate and the 
client and server learning rates. 

\textbf{Effect of participation rate $r$}
We examine the effect of clients participation rate $r$ 
in each optimization round, i.e., the average percentage 
$r$ of clients participating  in each round where a single server update 
is performed over 
$\theta$. We consider $r \in \{20,40,60, 100\}$. 
We use the CIFAR-10 dataset, with 
$T=200$ rounds and 
$\tau=50$ inner client GD steps per round. The client learning rate for all algorithms is $\beta=0.001$ while for PFLEGO the learning rate of the server is $\rho=0.001$.
From Figure \ref{fig:abl_cifar_plot_test_acc} 
we observe that for FedAvg and FedPer, the convergence speed does not vary with $r$, while for PFLEGO the optimization algorithm converges faster as $r$ increases. The more we increase the number of participation, the more gradients will be sent back to the server for aggregation. Then the server can perform an SGD step to the common weights that minimizes the total loss. \eqref{eq:fullloss}.


\label{appendix:effectoflearningrate}
\textbf{Effect of learning rate} In Figure \ref{fig:varying_beta} we conduct an ablation to study the effect of the client learning rate $\beta$ for
PFLEGO, and for various values of the server learning
rate value $\rho$ (i.e., the base rate in Adam). We observe
that we systematically obtain faster convergence whenever $\beta$ is usually
greater than $\rho$. This aligns well with Proposition \ref{prop:2} in Section \ref{sec:proposed_framework}, where for a Lipschitz constant $L > 0$, $0 < \beta < \frac{2}{L}$, and $0 < \rho < \frac{2}{L}\frac{r}{I}$, we ensure convergence, and we observe that:
(i) $\rho$ belongs to a smaller interval than $\beta$
and (ii) $\rho$ are inversely proportional to $L$. 

\section{Conclusions}
\label{conclusion}
We propose PFLEGO, a new Personalized FL algorithm which has theoretical convergence guarantees and, compared to previous FL methods it has lower computation cost.
Regarding convergence, we show rigorously that PFLEGO performs exact SGD-based steps, where the stochasticity arises due to randomness in client participation. Experimentally, we observe that we can achieve faster convergence and lower training loss than state-of-the-art alternatives such as FedAvg, FedPer and others. Importantly, PFLEGO's advantage arises in regimes where a high degree of personalization is needed, since as seen  
in the experiments its performance does not saturate but it keeps improving with the number of local client updates. At the same time these local updates are computationally efficient since they do not require additional neural network evaluations.  
We show that the convergence of PFLEGO
is ensured through a rigorous proof which holds even in non-convex settings, e.g.,
neural networks, where a $\mathcal{O} \left (\frac{1}{\sqrt{T}} \right )$ convergence rate is established with respect to the number of communication rounds $T$.
 
There exists several issues that warrant further investigation. First, in this work, the degree of personalization, i.e., the overlap between the subsets of classes available to each client, was assumed to be known a priori. If this were not known, a learning algorithm would need to be devised, which gradually learns the degree of personalization needed, e.g., through estimating similarity of tasks of different clients, so as to decide whether a personalization algorithm (e.g.,  PFLEGO) would be needed or not. Second, in this work, we aimed at optimizing total training loss over clients. Another direction for investigation would be to elaborate on fairness aspects by ensuring similar loss across clients, through a carefully selected objective function and a subsequent decentralized FL process. {Third, our current analysis and experiments demonstrate robustness to random client participation in each round, but do not address the case where the client population itself evolves over time. Extending PFLEGO to handle dynamic client participation, where new clients may join and others may leave, is an important direction for future work. Fourth, although our method is robust to partial participation and stragglers, we do not explicitly model or simulate unreliable network conditions or client dropouts due to connectivity issues. Incorporating and evaluating PFLEGO under more realistic network reliability scenarios is another valuable avenue for future research. Fifth, while our experiments focus on classification tasks, the PFLEGO framework is general and can be directly applied to regression problems or to data with non-IID features, as long as the loss function is differentiable. Exploring the performance of PFLEGO on regression tasks and more diverse data modalities is an interesting direction for future work.}

\section*{Acknowledgments}

This work was conducted in the context of the Horizon Europe project PRE-ACT (Prediction of Radiotherapy side effects using explainable AI for patient communication and treatment modification). It was supported by the European Commission through the Horizon Europe Program (Grant Agreement number 101057746), by the Swiss State Secretariat for Education, Research and Innovation (SERI) under contract number 22 00058, and by the UK government (Innovate UK application number 10061955). The authors wish to thank Assoc. Prof. Stavros Toumpis from Athens University of Economics and Business for helpful comments and discussions on the presentation of the paper.

\section*{Data availability}

The datasets used to perform the experiments are included
in the paper. The datasets are automatically accessed during the execution of the code through the TensorFlow Datasets library and is available in the corresponding author at \href{https://github.com/sotirisnik/PFLEGO}{https://github.com/sotirisnik/PFLEGO}.\\\\

\bibliographystyle{plain}
\bibliography{main}


\onecolumn
\appendix

\section{Experimental settings}
\subsection{Neural network architectures}
\label{appendix:neural_network_architectures}
In Table \ref{tab:architectures_altogether} we present the architectures for MNIST, Fashion-MNIST, EMNIST, CIFAR-10, and Omniglot. We also provide an equivalent pictorial visualization of the used NN architectures; see Figure \ref{fig:nn_architectures}. 
In Table \ref{datasets_details_table} we summarize the details of the different datasets.

\begin{table}[ht]
\caption{Architecture details for all datasets. $B$ denotes the number of data inputs to the neural network, $A$ denotes the activation function. {We use the ReLU activation function throughout our architectures, as it is the de facto standard in modern deep learning due to its simplicity, empirical effectiveness, and ability to mitigate vanishing gradient issues that can occur with traditional activation functions such as sigmoid or tanh.} }
    \label{tab:architectures_altogether}
        \centering
        \begin{tabular}{lccc}
        \toprule
                   & \multicolumn{3}{c}{MNIST, Fashion-MNIST and EMNIST}                                            \\ \midrule
        Layer      & Kernel details           & Stride               & Output shape                         \\ \midrule
        Input      & None                     & None                    & B$\times$28$\times$28$\times$1       \\
        Flatten    & 784                      & None                    & B$\times$784                         \\ 
        Dense      & 200, A=ReLU                       & None                    & B$\times$200    
        \\ 
        Dense(classifier)      & 10                       & None                    & B$\times$10                          \\ \midrule
        
                   & \multicolumn{3}{c}{CIFAR-10}                                                                     \\ \midrule
        Input      & None                     & None                           & B$\times$32$\times$32$\times$3       \\
        Conv1      & F=64, \hspace{1pt}K=5$\times$5, \hspace{1pt}A=ReLU & 1                              & B$\times$32$\times$32$\times$20      \\
        Max Pool 1 & K=3$\times$3             & 2                              & B$\times$16x16$\times$20             \\
        Conv2      & F=64, \hspace{1pt}K=5$\times$5, \hspace{1pt}A=ReLU & 1                              & B$\times$16x16$\times$20             \\
        Max Pool 2 & K=3$\times$3             & 2                              & B$\times$8$\times$8$\times$20        \\
        Flatten    & 4096                     & None                           & B$\times$1280                        \\
        Dense      & 384, \hspace{1pt}A=ReLU               & None                           & B$\times$800                         \\
        Dense      & 192, \hspace{1pt}A=ReLU               & None                           & B$\times$500                         \\
        Dense(classifier)      & 10                       & None                           & B$\times$10                          \\
                  \midrule
                   & \multicolumn{3}{c}{Omniglot}                                                                     \\ \midrule
        Input      & None                     & None                           & B$\times$28$\times$28$\times$1       \\
        Conv1      & F=64, \hspace{1pt}K=3$\times$3, \hspace{1pt}A=ReLU & 1                              & B$\times$28$\times$28$\times$64      \\
        MaxPool1   & K=2$\times$2             & 2                              & B$\times$14x14$\times$64             \\
        Conv2      & F=64, \hspace{1pt}K=3$\times$3, \hspace{1pt}A=ReLU & 1                              & B$\times$14$\times$14$\times$64      \\
        MaxPool2   & K=2$\times$2             & 2                              & B$\times$7x7$\times$64               \\
        Conv3      & F=64, \hspace{1pt}K=3$\times$3, \hspace{1pt}A=ReLU & 1                              & B$\times$7$\times$7$\times$64        \\
        MaxPool3   & K=2$\times$2             & 2                              & B$\times$3x3$\times$64               \\
        Conv4      & F=64, \hspace{1pt}K=3$\times$3, \hspace{1pt}A=ReLU & 1                              & B$\times$3$\times$3$\times$64,A=ReLU \\
        MaxPool4   & K=2$\times$2             & 2                              & B$\times$1x1$\times$64               \\
        Flatten    & 64                       & None                           & B$\times$64                          \\
        Dense(classifier)      & 55                       & None                           & B$\times$55   \\ \bottomrule
        \end{tabular}
\end{table}

\begin{figure}[h]
\begin{center}
\begin{tabular}{ccc}
{\includegraphics[scale=0.4,angle =0]{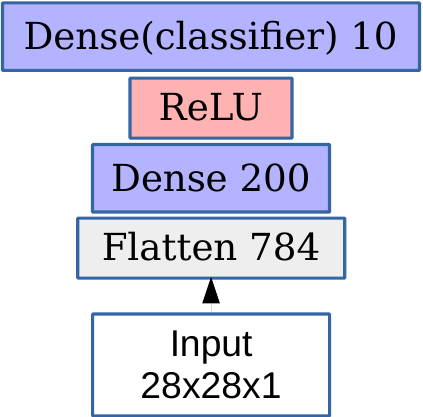}} &

{\includegraphics[width=0.4\textwidth,angle =0]{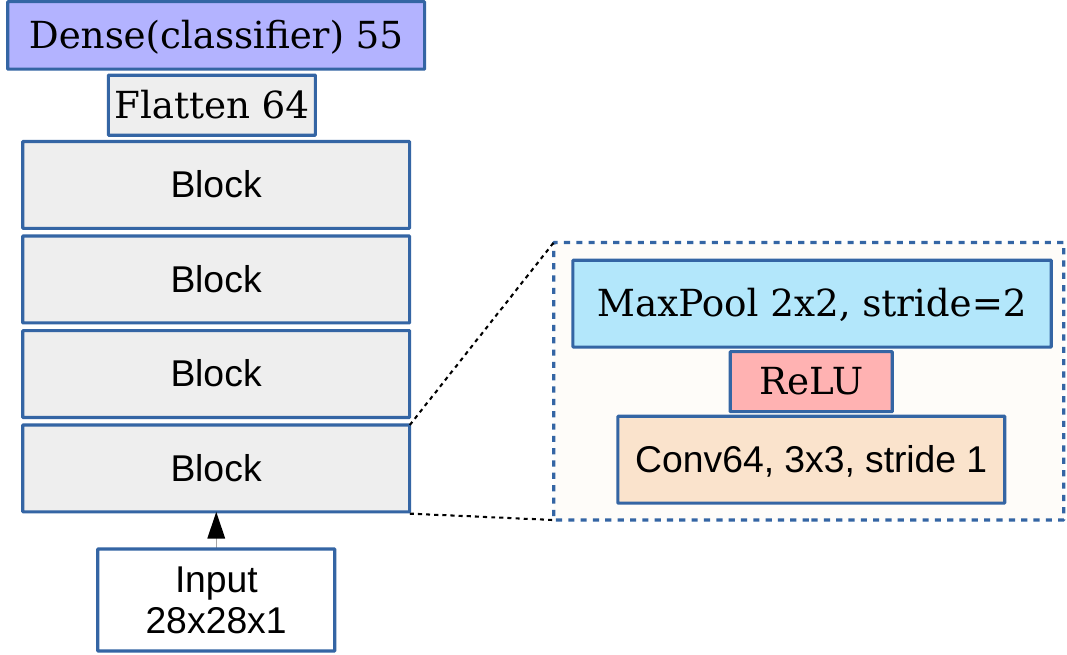}} &
{\includegraphics[width=0.4\textwidth,angle =90]{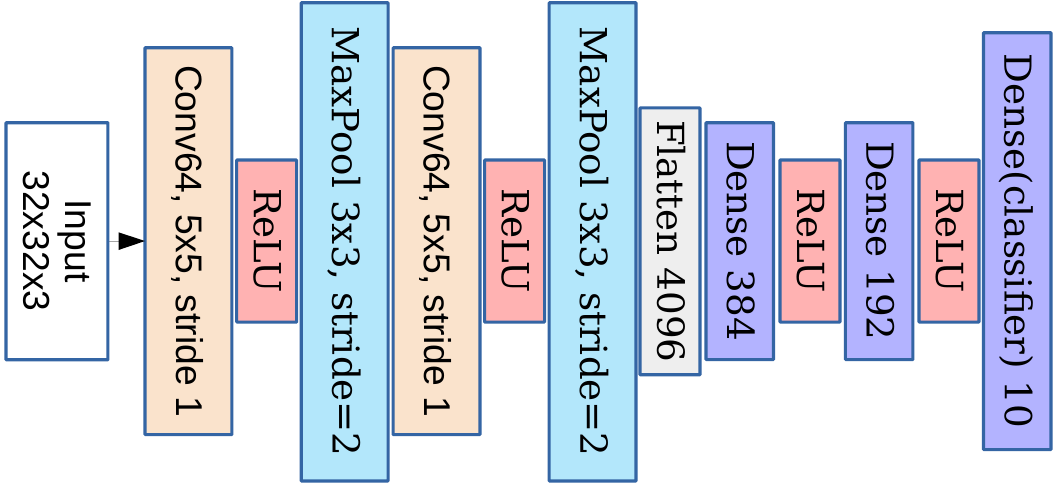}
}
 \\
MNIST & Omniglot & CIFAR-10
\end{tabular}
   \caption{Model architectures. \textbf{Left:} MLP architecture for MNIST, Fashion-MNIST and EMNIST; \textbf{Middle:} 4-convolutional layer architecture for Omniglot; \textbf{Right:} 2-convolutional layer architecture for CIFAR-10. }
  \label{fig:nn_architectures}
  \end{center}
\end{figure}

\begin{table}[!htbp]
\caption{Summary of dataset characteristics}
    \label{datasets_details_table}
    \vskip 0.15in
    \centering
        \begin{tabular}{cccccc}
        \toprule
        Dataset & Input Dimension & Labels & Train set & Test set & Feat. vector units $M$ \\
        \midrule
        MNIST & $28\times28\times1$ & 10 & 60,000 & 10,000 & 200 \\ 
        CIFAR-10 & $32\times32\times3$ & 10 & 50,000 & 10,000 & 192 \\ 
        Fashion-MNIST & $28\times28\times1$ & 10 & 60,000 & 10,000 & 200 \\ 
        EMNIST & $28\times28\times1$ & 62 & 697,932 & 116,323 & 200 \\ 
        Omniglot & $28\times28\times1$ & 1623 & 19,280 & 13,180 & 64 \\ 
        \bottomrule
        \end{tabular}
\vskip -0.1in
\end{table}

\subsection{Dataset splitting across clients}
\label{subsec:datasetsplit}
{In Figure \ref{fig:round_robin}, we see an example assignment of 7 data points to 3 clients using a round-robin (RR) partitioning scheme. In this example, all data points belong to the same class; in our experiments, each class is processed and distributed in the same way. The procedure for distributing data points of a given class (e.g., class 0) is as follows: 
 (a) shuffle all the data points of class 0, (b) filter all clients that were predetermined to contain the labels of class 0, (c) iterate through all filtered clients from step (b) and assign 1 data point to them. Step (c) is repeated in a cyclic manner until all data points are distributed.

In the example, the 7 data points are first randomly permuted; then we traverse the permuted sequence from left to right and assign one data point to each client in a cyclic manner until all examples are exhausted.
In contrast with random partition, RR does not require to shuffle the data samples first, however we do not omit the shuffling step of the data samples at the RR partition during our experiments.
}
\begin{figure}[H]
  \vskip 0.2in
  \begin{center}
  \includegraphics[width=0.8\textwidth]{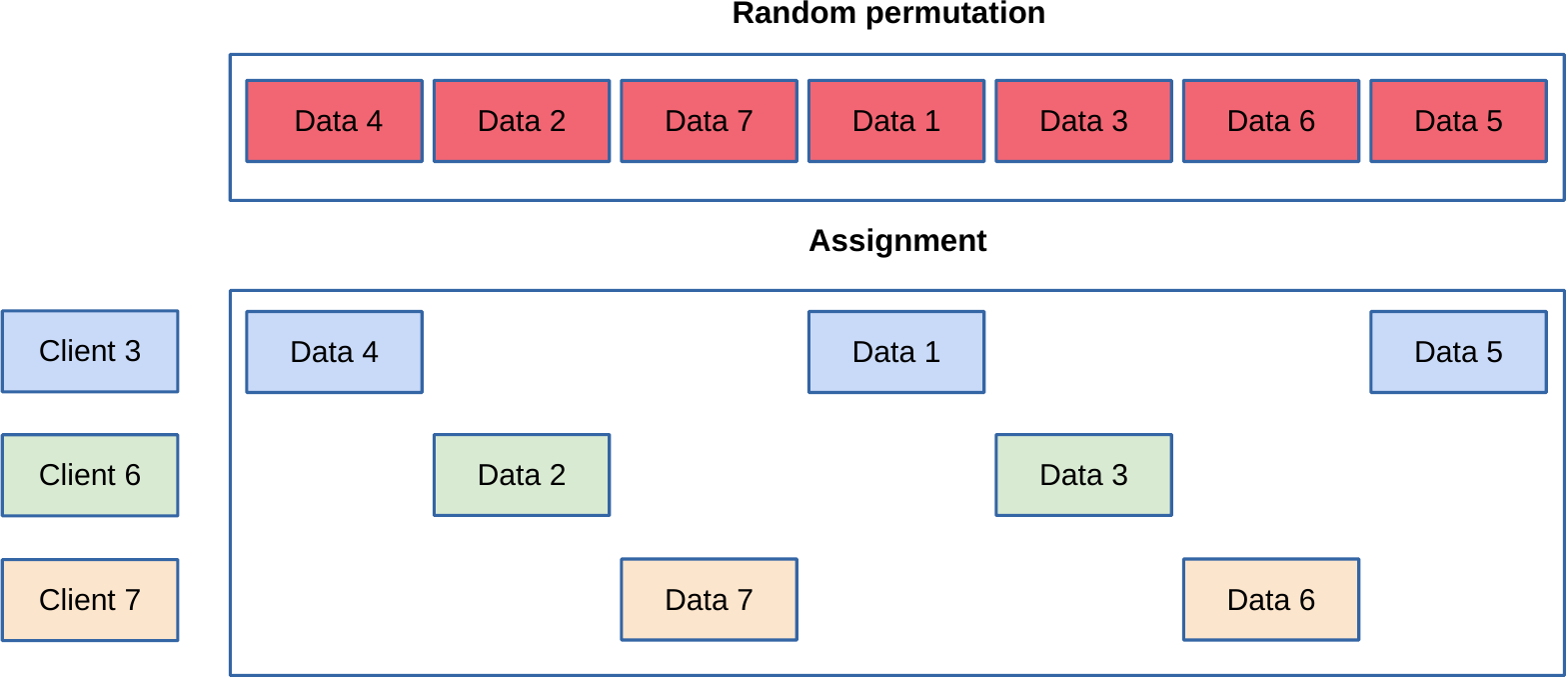}
  \caption{Assignment of 7 data points to 3 clients using a Round-Robin (RR) partitioning scheme.}
  \label{fig:round_robin}
  \end{center}
  \vskip -0.2in
\end{figure}

\subsection{Hyperparameter selection}
\label{subsec:hyperparameterselection}
For MNIST, CIFAR-10, EMNIST, and Fashion-MNIST we perform a hyperparameter search for $\beta$ in the set of values $\{0.001, 0.002, \ldots ,0.007\}$, and for $\rho$ is in the set $\{0.001, 0.002, 0.003\}$. For Omniglot we further examine two additional values for $\beta$, $0.008$ and $0.009$. For PFLEGO we further utilize our findings from Proposition 
    \ref{prop:2}: first we choose the learning rate of the server $\rho$, and then we choose the learning of the client $\beta \in \left (0,\frac{2}{L} \right )$. In Tables \ref{tab:best_hyper_results},\ref{omniglot_single_table_hyperparameters} we
    give the values of the best hyperparameters for various degrees of personalization. In all experiments, the server samples $20\%$ of the clients in each round. Note that for all experiments of FedBabu \cite{fedbabu}, we use momentum equal to 0.999, and each client performs 5 gradient updates at the common weights, after the client performs its gradient updates on the personalized weights. For FedProx \cite{li2020federated} we observe that $\mu=0.25$ obtained the best results. For Ditto \cite{ditto} we observe that a higher value of $\mu$ is more appropriate in "No-Pers", and respectively a lower value is more appropriate in the high personalization case.

\begin{table*}[ht]
\caption{Best hyperparameters for MNIST, CIFAR-10, EMNIST, and Fashion-MNIST on various personalized settings.}
\label{tab:best_hyper_results}
\vskip 0.15in
    \begin{center}
    \begin{small}
  \resizebox{\textwidth}{!}{%
        \begin{tabular}{llllllll}
        \toprule
         &  & \multicolumn{ 3}{c}{MNIST} & \multicolumn{ 3}{c}{CIFAR-10} \\ 
         \midrule
        Method / Deg of Pers.  &  & \multicolumn{1}{c}{High-Pers} & \multicolumn{1}{c}{Medium-Pers} & \multicolumn{1}{c}{No-Pers} & \multicolumn{1}{c}{High-Pers} & \multicolumn{1}{c}{Medium-Pers} & \multicolumn{1}{c}{No-Pers} \\ 
        
        FedPer &  & \multicolumn{1}{c}{$\beta=0.007$} & \multicolumn{1}{c}{$\beta=0.007$} & \multicolumn{1}{c}{$\beta=0.007$} & \multicolumn{1}{c}{$\beta=0.007$} & \multicolumn{1}{c}{$\beta=0.005$} & \multicolumn{1}{c}{$\beta=0.005$} \\ 
        FedAvg &  & \multicolumn{1}{c}{$\beta=0.007$} & \multicolumn{1}{c}{$\beta=0.007$} & \multicolumn{1}{c}{$\beta=0.007$} & \multicolumn{1}{c}{$\beta=0.007$} & \multicolumn{1}{c}{$\beta=0.007$} & \multicolumn{1}{c}{$\beta=0.006$} \\

        FedProx &  & \multicolumn{1}{c}{$\beta=0.007$, \hspace{1pt} $\mu=0.25$} & \multicolumn{1}{c}{$\beta=0.007$, \hspace{1pt} $\mu=0.25$} & \multicolumn{1}{c}{$\beta=0.007$, \hspace{1pt} $\mu=0.25$} & \multicolumn{1}{c}{$\beta=0.007$, \hspace{1pt} $\mu=0.25$} & \multicolumn{1}{c}{$\beta=0.007$, \hspace{1pt} $\mu=0.25$} & \multicolumn{1}{c}{$\beta=0.006$, \hspace{1pt} $\mu=0.25$} \\

    Ditto &  & \multicolumn{1}{c}{$\beta=0.007$, \hspace{1pt} $\mu=0.25$} & \multicolumn{1}{c}{$\beta=0.007$, \hspace{1pt} $\mu=0.5$} & \multicolumn{1}{c}{$\beta=0.007$, \hspace{1pt} $\mu=1.0$} & \multicolumn{1}{c}{$\beta=0.007$, \hspace{1pt} $\mu=0.25$} & \multicolumn{1}{c}{$\beta=0.007$, \hspace{1pt} $\mu=0.5$} & \multicolumn{1}{c}{$\beta=0.007$, \hspace{1pt} $\mu=1.0$} \\

    FedRep &  & \multicolumn{1}{c}{$\beta=0.007$} & \multicolumn{1}{c}{$\beta=0.007$} & \multicolumn{1}{c}{$\beta=0.007$} & \multicolumn{1}{c}{$\beta=0.005$} & \multicolumn{1}{c}{$\beta=0.007$} & \multicolumn{1}{c}{$\beta=0.007$} \\

        FedBabu &  & \multicolumn{1}{c}{$\beta=0.007$} & \multicolumn{1}{c}{$\beta=0.007$} & \multicolumn{1}{c}{$\beta=0.007$} & \multicolumn{1}{c}{$\beta=0.007$}  & \multicolumn{1}{c}{$\beta=0.007$} & \multicolumn{1}{c}{$\beta=0.007$} \\
        
        PFLEGO &  & \multicolumn{1}{c}{$\beta=0.189$, \hspace{1pt} $\rho=0.003$} & \multicolumn{1}{c}{$\beta=0.189$, \hspace{1pt} $\rho=0.003$} & $\beta=0.189$, \hspace{1pt} $\rho=0.003$ & \multicolumn{1}{c}{$\beta=0.002$, \hspace{1pt} $\rho=0.001$} & \multicolumn{1}{c}{$\beta=0.0189$, \hspace{1pt} $\rho=0.003$} & \multicolumn{1}{c}{$\beta=0.03$, \hspace{1pt} $\rho=0.003$} \\ 
        \midrule
         &  & \multicolumn{ 3}{c}{EMNIST} & \multicolumn{ 3}{c}{Fashion-MNIST} \\ \midrule
        Method / Deg of Pers.  &  & \multicolumn{1}{c}{High-Pers} & \multicolumn{1}{c}{Medium-Pers} & \multicolumn{1}{c}{No-Pers} & \multicolumn{1}{c}{High-Pers} & \multicolumn{1}{c}{Medium-Pers} & \multicolumn{1}{c}{No-Pers} \\ 
        FedPer &  & \multicolumn{1}{c}{$\beta=0.007$} & \multicolumn{1}{c}{$\beta=0.007$} & \multicolumn{1}{c}{$\beta=0.007$} & \multicolumn{1}{c}{$\beta=0.007$} & \multicolumn{1}{c}{$\beta=0.007$} & \multicolumn{1}{c}{$\beta=0.007$} \\ 
        FedAvg &  & \multicolumn{1}{c}{$\beta=0.007$}  & \multicolumn{1}{c}{$\beta=0.007$} & \multicolumn{1}{c}{$\beta=0.007$} & \multicolumn{1}{c}{$\beta=0.007$} & \multicolumn{1}{c}{$\beta=0.007$} & \multicolumn{1}{c}{$\beta=0.007$} \\ 

        FedProx &  & \multicolumn{1}{c}{$\beta=0.007$, \hspace{1pt} $\mu=0.25$} & \multicolumn{1}{c}{$\beta=0.007$, \hspace{1pt} $\mu=0.25$} & \multicolumn{1}{c}{$\beta=0.007$, \hspace{1pt} $\mu=0.25$} & \multicolumn{1}{c}{$\beta=0.007$, \hspace{1pt} $\mu=0.25$} & \multicolumn{1}{c}{$\beta=0.007$, \hspace{1pt} $\mu=0.25$} & \multicolumn{1}{c}{$\beta=0.007$, \hspace{1pt} $\mu=0.25$} \\

        Ditto &  & \multicolumn{1}{c}{$\beta=0.007$, \hspace{1pt} $\mu=0.25$} & \multicolumn{1}{c}{$\beta=0.007$, \hspace{1pt} $\mu=0.25$} & \multicolumn{1}{c}{$\beta=0.007$, \hspace{1pt} $\mu=1.0$} & \multicolumn{1}{c}{$\beta=0.005$, \hspace{1pt} $\mu=0.25$} & \multicolumn{1}{c}{$\beta=0.007$, \hspace{1pt} $\mu=0.25$} & \multicolumn{1}{c}{$\beta=0.007$, \hspace{1pt} $\mu=1.0$} \\

        FedRep &  & \multicolumn{1}{c}{$\beta=0.007$} & \multicolumn{1}{c}{$\beta=0.007$} & \multicolumn{1}{c}{$\beta=0.007$} & \multicolumn{1}{c}{$\beta=0.003$} & \multicolumn{1}{c}{$\beta=0.007$} & \multicolumn{1}{c}{$\beta=0.007$} \\

        FedBabu &  & \multicolumn{1}{c}{$\beta=0.007$} & \multicolumn{1}{c}{$\beta=0.007$} & \multicolumn{1}{c}{$\beta=0.007$} & \multicolumn{1}{c}{$\beta=0.007$}  & \multicolumn{1}{c}{$\beta=0.007$} & \multicolumn{1}{c}{$\beta=0.007$} \\
        
        PFLEGO &  & \multicolumn{1}{c}{$\beta=0.03$, \hspace{1pt} $\rho=0.003$} & \multicolumn{1}{c}{$\beta=0.03$, \hspace{1pt} $\rho=0.003$} & \multicolumn{1}{c}{$\beta=0.03$, \hspace{1pt} $\rho=0.003$} & \multicolumn{1}{c}{$\beta=0.006$, \hspace{1pt} $\rho=0.002$} & \multicolumn{1}{c}{$\beta=0.006$, \hspace{1pt} $\rho=0.002$} & \multicolumn{1}{c}{$\beta=0.007$, \hspace{1pt} $\rho=0.003$} \\ 
        \bottomrule
        \end{tabular}
         }
    \end{small}
    \end{center}
\vskip -0.1in
\end{table*}

\begin{table*}[!htbp]
\caption{Best hyperparameters for Omniglot dataset on high degree of personalization.}
\label{omniglot_single_table_hyperparameters}
\begin{center}
\begin{small}
\resizebox{\textwidth}{!}{%
\begin{tabular}{cccccccccc}
\toprule
\multicolumn{8}{c}{Omniglot}
\\
\midrule
FedRep &
FedBabu &
FedProx & Ditto & FedPer & FedAvg & FedRecon & PFLEGO\\ 
\multicolumn{1}{c}{$\beta=0.009$} & \multicolumn{1}{c}{$\beta=0.009$} & \multicolumn{1}{c}{$\beta=0.009$, \hspace{1pt} $\mu=0.25$} & \multicolumn{1}{c}{$\beta=0.007$, \hspace{1pt} $\mu=0.5$} & \multicolumn{1}{c}{$\beta=0.009$} & \multicolumn{1}{c}{$\beta=0.009$} & \multicolumn{1}{c}{$\beta=0.009$, \hspace{1pt} $\rho=0.003$} & \multicolumn{1}{c}{$\beta=0.1341$, \hspace{1pt} $\rho=0.003$} \\
\bottomrule
\end{tabular}
}
\end{small}
\end{center}
\end{table*}




\section{Bounds for hyperparameters for algorithm convergence}
\label{appendix:hyperparameterssegments}
In this section we discuss the optimization process of the clients, and in particular we establish convergence of the common weights $\theta_{t}$ to a stationary point. We derive intervals in which the learning rate $\beta$ of the client and the learning rate of the server $\rho$ must fall in order to guarantee convergence and ensure that the loss function will decrease at each optimization update. In order to establish the convergence of the common weights $\theta_{t}$ to a stationary point, we use Proposition \ref{prop:1}, and the assumptions of $L$-smoothness in Proposition \ref{prop:2}.


 First, at each round $t$, let $\theta_{t-1}$ be the common weights at the server, the server broadcast $\theta_{t-1}$ to a subset of clients $\mathcal{I}_t \subset \{1,\ldots,I\}$, which is selected uniformly at random, and client $i \in \mathcal{I}_t$ sets its  global parameters to $\theta_{t-1}$.  Client $i$ performs a number of $j=1\dots\tau$ gradient updates on its personalized weights $W_{i}$.  Next for simplicity we omit the index of the round on the client-specific parameters, and we use the notation $W_{i}^{j} \triangleq W_{i,t-1} $ for its updates. Client $i$ uses a gradient-based optimizer to optimize its loss function, e.g.,

$$W_{i}^{j} \leftarrow W_{i}^{j-1} - \beta \nabla_{W_{i}^{j-1}} \ell_{i}( W_{i}^{j-1}, \theta_{t-1}),$$ but keeps $\theta_{t-1}$ fixed. However, at its last iteration client $i$ sets $\beta=\rho$, computes the joint gradient $\left ({\nabla_{\theta_{t-1}} \ell_i(\mathcal{D}_i; W_i,\theta_{t-1}) }, {\nabla_{W_i} \ell_i(\mathcal{D}_i; W_i,\theta_{t-1}) } \right)$, and returns the first component ${\nabla_{\theta_{t-1}} \ell_i(\mathcal{D}_i; W_i,\theta_{t-1})}$ to the server.

At the client-side, for the first $\tau-1$ gradient steps the optimization is deterministic (the client uses its full dataset $\mathcal{D}_{i}$), so the whole procedure is standard GD. We require that the learning rate $\beta$ is small enough so that each step decreases the loss, i.e., $\ell_{i}( W_{i}^{j},\cdot) \leq  \ell_{i}( W_{i}^{j-1},\cdot )$. To formalize this, we assume that the loss is $L$-smooth, so that at the $j$-th gradient step the following holds \cite{strang2019linear}:

\begin{align}
    \ell_{i}( W_{i}^{j},\cdot) \leq \ell_{i}( W_{i}^{j-1},\cdot ) + \langle \nabla \ell_{i}( W_{i}^{j-1},\cdot), W_{i}^{j} - W_{i}^{j-1} \rangle + \frac{L}{2} \|W_{i,j} - W_{i}^{j-1}\|^{2}.
    \label{ineq:lipschitzw}
\end{align}

We examine 2 cases. The first 
case is when $j=1\dots\tau-1$  and the client applies GD steps with its learning rate $\beta$. Then, (\ref{ineq:lipschitzw}) becomes

\begin{align}
\ell_{i}( W_{i}^{j},\cdot)
\leq
\ell_{i}( W_{i}^{j-1},\cdot )
- \beta  \|\nabla \ell_{i}( W_{i}^{j-1},\cdot)\|^{2}  + \frac{L}{2} \beta^{2} \|\nabla \ell_{i}( W_{i}^{j-1},\cdot)\|^{2}
\nonumber \\
=
\ell_{i}( W_{i}^{j-1},\cdot )
- \beta \|\nabla \ell_{i}( W_{i}^{j-1},\cdot)\|^{2}  \left ( 1 - \frac{L}{2} \beta \right)
\nonumber \\
=\ell_{i}( W_{i}^{j-1},\cdot )+ \beta \|\nabla \ell_{i}( W_{i}^{j-1},\cdot)\|^{2} \left ( \frac{L}{2} \beta - 1 \right ).
\nonumber
\end{align}
Therefore for $\beta > 0$, the term
$\left ( \frac{L}{2} \beta  - 1 \right )$ must be negative, and thus the learning rate must satisfy 
$ 0 < \beta < \frac{2}{L}$. 
The second case is for the final iteration 
$j = \tau$. In this case the client $i$ also uses its full dataset $\mathcal{D}_{i}$, however this step corresponds to the SGD step coordinated by the server. This step sets the learning rate to $\rho \frac{I}{r}$, and thus equivalently to the first case, 
 $\left ( \frac{L}{2} \rho \frac{I}{r} - 1 \right ) < 0$, which means that $\rho$ must satisfy $0 < \rho < \frac{2}{L}  \frac{r}{I}$. 

Thus we can choose $0 < \beta < \frac{2}{L}$, and $0 < \rho < \frac{2}{L}\frac{r}{I}$. Observe that the selection of $\rho$ falls within a narrower interval compared to $\beta$, and also the interval of $\rho$ depends on the number of participants $r$ per round, and the total number of clients $I$.

\label{appendix:proof}

Likewise, the new common weights $\theta_{t}$ are computed at the server by aggregating the returns from the clients

\begin{equation}
    \theta_{t} \leftarrow \theta_{t-1} - \rho_{t-1} \frac{I}{r} \sum_{i \in \mathcal{I}_t} \alpha_i \nabla_{\theta_{t-1}} 
 \ell_i(W_i,\theta_{t-1}).
 \label{eq:updatethetawithrho}
\end{equation}
{Note that in (\ref{eq:serverupdate}) we used a constant learning rate for the server, however in (\ref{eq:updatethetawithrho}) we denote the learning rate of the server at round $t$ with $\rho_{t-1}$.}
For each optimization round $t \in [1,T]$ by the Lipschitz smoothness \cite{strang2019linear} we have

\begin{equation}
\label{ineq:1}\mathbb{E} [ \mathcal{L}( \theta_{t})] \leq \mathbb{E} [ \mathcal{L}( \theta_{t-1} ) ] + {\mathbb{E} [ \langle \nabla \mathcal{L}( \theta_{t-1}), \underbrace{\theta_{t} - \theta_{t-1}}_{stochastic} \rangle ]} + \frac{L}{2} {E[ \|\underbrace{\theta_{t} - \theta_{t-1}}_{stochastic}\|^{2} ]}.
\end{equation}
Note that if we use the server update given in (\ref{eq:serverupdate}), then the last term of (\ref{ineq:1}) is written as:

\begin{align}
{\mathbb{E} [ \|\theta_{t} - \theta_{t-1}\|^{2} ]} 
& = \mathbb{E} [ \|\theta_{t-1} - \rho_{t-1} \frac{I}{r} \sum_{i=1}^{I} {\bf 1}_{i \in \mathcal{I}_t} \alpha_i \bg_i - \theta_{t-1}\|^{2} ] \nonumber \\ 
& = \mathbb{E} [ \|- \rho_{t-1} \frac{I}{r}\sum_{i=1}^{I} {\bf 1}_{i \in \mathcal{I}_t} \alpha_i \bg_i\|^{2} ] \\
& = \rho_{t-1}^{2} \frac{I^{2}}{r^{2}} \mathbb{E} [ \|\sum_{i=1}^{I} {\bf 1}_{i \in \mathcal{I}_t} \alpha_i \bg_i\|^{2} ].
\nonumber
\end{align}
The $l_{2}$ term $\|\sum_{i=1}^{I} {\bf 1}_{i \in \mathcal{I}_t} \alpha_i \bg_i\|^{2}$ is a multiplication between two gradient vectors.

\begin{align}
\|\sum_{i=1}^{I} {\bf 1}_{i \in \mathcal{I}_t} \alpha_i \bg_i\|^{2}
& = \left (\sum_{i=1}^{I} {\bf 1}_{i \in \mathcal{I}_t} \alpha_i \bg_i \right )^{T} \times \left (\sum_{j=1}^{I} {\bf 1}_{j \in \mathcal{I}_t} \alpha_j \bg_j \right )
\\
& = \left (\sum_{i=1}^{I} {\bf 1}_{i \in \mathcal{I}_t} \alpha_i \bg_i^{T} \right )\times \left (\sum_{j=1}^{I} {\bf 1}_{j \in \mathcal{I}_t} \alpha_j \bg_j \right )
\\ 
& = \sum_{i=1}^{I}\sum_{j=1}^{I} {\bf 1}_{i \in \mathcal{I}_t} {\bf 1}_{j \in \mathcal{I}_t} \alpha_i \alpha_j \bg_i^{T} \bg_j.
\end{align}
We split the sum in terms of $i=j$, and $i\neq j$:

\begin{equation}
\label{split_i_and_j}
\|\sum_{i=1}^{I} {\bf 1}_{i \in \mathcal{I}_t} \alpha_i \bg_i\|^{2}
=
\sum_{i=1}^{I} {\bf 1}_{i \in \mathcal{I}_t} \alpha_i^2 g_i^2 
+ \sum_{i \neq j}^{I}
{\bf 1}_{i \in \mathcal{I}_t} {\bf 1}_{j \in \mathcal{I}_t} \alpha_i \alpha_j \bg_i \bg_j. 
\end{equation}
Substituting (\ref{split_i_and_j}) back to the last term of (\ref{ineq:1}) we obtain:

\begin{align}
{\mathbb{E} [ \|\theta_{t} - \theta_{t-1}\|^{2} ]} 
& = \rho_{t-1}^{2} \frac{I^{2}}{r^{2}} \mathbb{E} \left [ \sum_{i=1}^{I} {\bf 1}_{i \in \mathcal{I}_t} \alpha_i^2 g_i^2 
+ \sum_{i \neq j}^{I}
{\bf 1}_{i \in \mathcal{I}_t} {\bf 1}_{j \in \mathcal{I}_t} \alpha_i \alpha_j \bg_i \bg_j \right ]
\nonumber \\
& =   \rho_{t-1}^{2} \frac{I^{2}}{r^{2}} \left (
\sum_{i=1}^{I} \mathbb{E} [{\bf 1}_{i \in \mathcal{I}_t} \alpha_i^2 g_i^2]  
+ \sum_{i \neq j}^{I}
\mathbb{E} [{\bf 1}_{i \in \mathcal{I}_t} {\bf 1}_{j \in \mathcal{I}_t} \alpha_i \alpha_j \bg_i \bg_j ] \right ) \nonumber \\
& \leq \rho_{t-1}^{2} \frac{I^{2}}{r^{2}} \left ( 
G^2 \frac{r}{I} \sum_{i=1}^{I} \alpha_i^2  
+ G^2 \frac{r(r-1)}{I^2} \sum_{i \neq j}^{I} \alpha_i \alpha_j \right ). \label{ineq:2}
\end{align}
Note that the term $\frac{r(r-1)}{I^2}$ corresponds to the case (\textit{a}) as described in Section \ref{sec:client-selection}. However it is important to observe that $\frac{r(r-1)}{I^2} \leq \frac{r^2}{I^2}$, where the quantity $\frac{r^2}{I^2}$ corresponds to the case (\textit{b}) as described in Section \ref{sec:client-selection}. We continue the analysis with the later term, since it does not impact the final result of our analysis.  Substituting $\frac{r^2}{I^2}$ back to (\ref{ineq:2}) we obtain:

\begin{align}
{\mathbb{E} [ \|\theta_{t} - \theta_{t-1}\|^{2} ]} 
& \leq \rho_{t-1}^{2} \frac{I^{2}}{r^{2}} \left ( 
G^2 \frac{r}{I} \sum_{i=1}^{I} \alpha_i^2  
+ G^2 \frac{r^2}{I^2} \sum_{i \neq j}^{I} \alpha_i \alpha_j \right )
\\
& = \rho_{t-1}^{2} G^2 \left ( 
\frac{I}{r} \sum_{i=1}^{I} \alpha_i^2  
+ \sum_{i \neq j}^{I} \alpha_i \alpha_j \right )
\\
& = \rho_{t-1}^{2} G^2
\frac{I\sum_{i=1}^{I} \alpha_i^2
+ r \sum_{i \neq j}^{I} \alpha_i \alpha_j
}{r}.
\nonumber 
\end{align}
Further for the second term of the (\ref{ineq:1}), we use the server update (\ref{eq:serverupdate}) and we note that

\begin{align}
    {\mathbb{E} [ \langle \nabla \mathcal{L}( \theta_{t-1}), \theta_{t} - \theta_{t-1} \rangle ]}
   & = \mathbb{E} [ \langle \nabla \mathcal{L}( \theta_{t-1}), - \rho_{t-1} \frac{I}{r}\sum_{i=1}^{I} {\bf 1}(i \in \mathcal{I}_t) \alpha_i \bg_i \rangle ]
    \nonumber
    \\
   & = \nabla \mathcal{L}(\theta_{t-1})^{T}\mathbb{E} [ - \rho_{t-1} \frac{I}{r}\sum_{i=1}^{I} {\bf 1}(i \in \mathcal{I}_t) \alpha_i \bg_i ]
   \\
   & = - \rho_{t-1} \nabla \mathcal{L}(\theta_{t-1})^{T} \nabla \mathcal{L}(\theta_{t-1}) \sum_{i=1}^{I} \alpha_i
    \nonumber
    \\
   & \triangleq - \rho_{t-1} \mathbb{E} [ \|\nabla \mathcal{L}(\theta_{t-1}\|^{2} ]. \nonumber
\end{align}
Similarly to \cite{haoyu}, we sum (\ref{ineq:1}) over $t \in \{1, \dots , T \}$.

\begin{align}
    \sum_{t=1}^{T}\mathbb{E} [ \mathcal{L}( \theta_{t})] & \leq \sum_{t=1}^{T}\mathbb{E} [ \mathcal{L}( \theta_{t-1} ) ] - \sum_{t=1}^{T}{\mathbb{E} [ \langle \nabla \mathcal{L}( \theta_{t-1}), \theta_{t} - \theta_{t-1} \rangle ]} + \frac{L}{2} \sum_{t=1}^{T} {E[ \|\theta_{t} - \theta_{t-1}\|^{2} ]}
    \nonumber
    \\
    & = \sum_{t=1}^{T}\mathbb{E} [ \mathcal{L}( \theta_{t-1} ) ] - \sum_{t=1}^{T}\rho_{t-1} \mathbb{E} [ \| \nabla \mathcal{L}( \theta_{t-1})\|^{2} ]  + \frac{L}{2} G^2
    \frac{I
    + r
    }{r} \sum_{t=1}^{T} \rho_{t-1}^{2}.
    \label{ineq:sumover}
    \end{align}
    We follow the same convention as in \cite{haoyu,ghadimi,lian2017,alistarh} where the average expected squared gradient norm is used to characterize the convergence rate. We rearrange the terms in (\ref{ineq:sumover}) and we obtain
    
    \begin{align}
        \sum_{t=1}^{T}\rho_{t-1} \mathbb{E} [ \| \nabla \mathcal{L}( \theta_{t-1})\|^{2} ] & \leq \sum_{t=1}^{T}(\mathbb{E} [ \mathcal{L}( \theta_{t-1} ) ] - \mathbb{E} [ \mathcal{L}( \theta_{t})])
    + \frac{L}{2} G^2
    \frac{I
    + r
    }{r} \sum_{t=1}^{T} \rho_{t-1}^{2}.
    \label{ineq:rearrange}
\end{align}
We assume that $\rho_{t}$, where $t \in [0,T]$ is a decreasing sequence. Then we use the following lower bound for $\sum_{t=1}^{T}\rho_{t-1} \mathbb{E} [ \| \nabla \mathcal{L}( \theta_{t-1})\|^{2} ]$, 

$$   \rho_{T-1} 
    \sum_{t=1}^{T}\mathbb{E} [ \| \nabla \mathcal{L}( \theta_{t-1})\|^{2} ]
    \leq
    \sum_{t=1}^{T}\rho_{t-1} \mathbb{E} [ \| \nabla \mathcal{L}( \theta_{t-1})\|^{2} ]. $$
In the same manner we derive an upper bound for $\sum_{t=1}^{T} \rho_{t-1}^{2}$, e.g.,

$$\sum_{t=1}^{T} \rho_{t-1}^{2}  \leq \sum_{t=1}^{T} \rho_{0}^{2}.$$
The telescopic sum is bounded by
$$\sum_{t=1}^{T}(\mathbb{E} [ \mathcal{L}( \theta_{t-1} ) ] - \mathbb{E} [ \mathcal{L}( \theta_{t})]) = \mathcal{L}( \theta_{0} ) - \mathcal{L}( \theta_{T}) \leq \mathcal{L}( \theta_{0} ) - \mathcal{L}( \theta^{*}),$$
where $\theta^{*}$ are the optimal parameters of $\mathcal{L}$. Thus, putting it altogether in (\ref{ineq:rearrange}),
\begin{align}   
    \rho_{T-1} 
    \sum_{t=1}^{T}\mathbb{E} [ \| \nabla \mathcal{L}( \theta_{t-1})\|^{2} ]
   & \leq
    \sum_{t=1}^{T}\rho_{t-1} \mathbb{E} [ \| \nabla \mathcal{L}( \theta_{t-1})\|^{2} ] \leq \left ( \mathcal{L}( \theta_{0} ) - \mathcal{L}( \theta^{*}) \right )
    + \frac{L}{2} G^2
    \frac{I
    + r
    }{r} \sum_{t=1}^{T} \rho_{t-1}^{2}
    \nonumber
    \\
   & \leq
     \left ( \mathcal{L}( \theta_{0} ) - \mathcal{L}( \theta^{*}) \right )
    + \frac{L}{2} G^2
    \frac{I
    + r
    }{r} \sum_{t=1}^{T} \rho_{0}^{2}
    \nonumber
    \\
    & =
    \left ( \mathcal{L}( \theta_{0} ) - \mathcal{L}( \theta^{*}) \right )
    + \frac{L}{2} G^2
    \frac{I
    + r
    }{r} T \rho_{0}^{2}.
    \nonumber
    \end{align}   
From this we obtain,
\begin{equation}
    \rho_{T-1} 
    \sum_{t=1}^{T}\mathbb{E} [ \| \nabla \mathcal{L}( \theta_{t-1})\|^{2} ]
    \leq
    \left ( \mathcal{L}( \theta_{0} ) - \mathcal{L}( \theta^{*}) \right )
    + \frac{L}{2} G^2
    \frac{I
    + r }{r} T \rho_{0}^{2}.
    \nonumber
\end{equation}
We divide both sides by $T$,
\begin{equation}
    \frac{1}{T}
    \sum_{t=1}^{T}\mathbb{E} [ \| \nabla \mathcal{L}( \theta_{t-1})\|^{2} ]
    \leq
    \frac{1}{T\rho_{T-1}}
    \left ( \mathcal{L}( \theta_{0} ) - \mathcal{L}( \theta^{*}) \right )
    + \frac{L}{2} G^2
    \frac{I
    + r
    }{r}\frac{\rho_{0}^{2}}{\rho_{T-1}}.
    \label{ineq:dividebyt}
\end{equation}
At this point, we prove the convergence rates obtained in the Corollary \ref{corol:1}. In the first case, if we set $\rho_{0} = \frac{1}{L}\frac{\sqrt{r}}{\sqrt{T}\sqrt{I}}$, and set the learning rate $\rho_{t}$ to a constant value, e.g., $\rho_{t}=\rho_{0}$, then, (\ref{ineq:dividebyt}) becomes
\begin{align}
    \frac{1}{T}
    \sum_{t=1}^{T}\mathbb{E} [ \| \nabla \mathcal{L}( \theta_{t-1})\|^{2} ]
    & \leq
    \frac{1}{T\rho_{0}}
    \left ( \mathcal{L}( \theta_{0} ) - \mathcal{L}( \theta^{*}) \right )
    + \frac{L}{2} G^2
    \frac{I+ r
    }{r} \rho_{0}
    \nonumber
    \\
   &  =
    \frac{L}{\sqrt{T}}\frac{I}{\sqrt{Ir}}
    \left ( \mathcal{L}( \theta_{0} ) - \mathcal{L}( \theta^{*}) \right )
    + \frac{G^2}{\sqrt{T}}\frac{I+r}{2\sqrt{Ir}} =\mathcal{O} \left (\sqrt{\frac{I}{T}} \right ).
    \nonumber
\end{align}
Further, if we set $r = I$, then,
\begin{align}
    \frac{1}{T}
    \sum_{t=1}^{T}\mathbb{E} [ \| \nabla \mathcal{L}( \theta_{t-1})\|^{2} ]
    \leq
    \frac{L}{\sqrt{T}}
    \left ( \mathcal{L}( \theta_{0} ) - \mathcal{L}( \theta^{*}) \right )
    + \frac{G^2}{\sqrt{T}}   =\mathcal{O} \left (\frac{1}{\sqrt{T}} \right ),
    \nonumber
\end{align}
which concludes the Corollary \ref{corol:1}.
The convergence rate $\mathcal{O} \left (\frac{1}{\sqrt{T}} \right )$ indicates that as the number of rounds $T$ increases, then PFLEGO progressively improves its performance and gets closer to the optimal solution. With this we obtain the desired result of Proposition \ref{prop:2}.
However, for $\rho_{0}=\frac{\sqrt{r}}{L\sqrt{T}\sqrt{I}}$, we must verify that the third assumptions of Propositon \ref{prop:2} holds, i.e., the  inequality $\rho_{0} < \frac{2}{L} \frac{r}{I}$. Specifically, we are interested in finding the minimum number of participants $r$, that are required to participate per round to guarantee that the loss function will decrease.

$$\rho_{0} = \frac{\sqrt{r}}{L\sqrt{T}\sqrt{I}} < \frac{2}{L} \frac{r}{I}.$$
$L$ is greater than 0, therefore 
\begin{align}
\frac{\sqrt{r}}{L\sqrt{T}\sqrt{I}}  < \frac{2}{L} \frac{r}{I}
\Leftrightarrow
\frac{\sqrt{r}}{\sqrt{T}\sqrt{I}} & < 2\frac{r}{I}.
\nonumber
\end{align}
We square both sides

\begin{align}
    \left (\frac{\sqrt{r}}{\sqrt{T}\sqrt{I}} \right )^{2} < \left (2\frac{r}{I} \right )^{2}
    \Leftrightarrow
    \frac{r}{TI} < 4\frac{r^{2}}{I^{2}}
    \Leftrightarrow
    r < 4\frac{r^{2}T}{I}
    \Leftrightarrow
    0 & < r \left ( r\frac{4T}{I} - 1 \right ).
        \nonumber
\end{align}
Thus, $r \ge \left \lfloor \frac{I}{4T} \right \rfloor + 1$.






\end{document}